\documentclass{article}

\PassOptionsToPackage{numbers, compress}{natbib}

 \usepackage[main, final]{neurips_2026}

\usepackage[utf8]{inputenc} 
\usepackage[T1]{fontenc}    
\usepackage{hyperref}       
\usepackage{url}            
\usepackage{booktabs}       
\usepackage{amsfonts}       
\usepackage{nicefrac}       
\usepackage{microtype}      
\usepackage{xcolor}         
\usepackage{amsmath}
\usepackage{amsthm}
\usepackage{amssymb}
\usepackage{graphicx}
\usepackage{subfigure}
\usepackage{wrapfig}
\usepackage{multirow}
\usepackage{algorithm}
\usepackage{algorithmic}
\usepackage{caption}

\newtheorem{theorem}{Theorem}[section]
\newtheorem{proposition}[theorem]{Proposition}

\title{Diffusion Subgoal Planning for Long-Horizon Offline Goal-Conditioned
Reinforcement Learning}

\author{%
  Hengrui Zhang\textsuperscript{\rm 1}~~Yuhu Cheng\textsuperscript{\rm 1}~~C. L. Philip Chen\textsuperscript{\rm 2}~~Xuesong Wang\textsuperscript{\rm 1}\thanks{Corresponding author.} 
  \vspace{0.2cm}
  \\
  $^1$School of Information and Control Engineering, China University of Mining and Technology 
  \\ 
  $^2$School of Computer Science and Engineering, South China University of Technology \vspace{0.1cm}
  \\
  \texttt{\{hengruizhang, chengyuhu, wangxuesong\}@cumt.edu.cn} 
  \\
  \texttt{\{philip.chen\}@ieee.org}
  }

\begin{document}

\maketitle

\begin{abstract}
Offline goal-conditioned reinforcement learning (GCRL) learns goal-directed policies from reward-free data, but in long-horizon tasks, goal-conditioned value functions often provide unstable guidance due to sparse rewards and discounting. Hierarchical methods partially mitigate this issue via subgoal decomposition; however, high-level decision-making still relies on noise-sensitive value estimates, leading to unstable behavior in complex environments. We address this limitation by proposing \textbf{D}iffusion \textbf{S}ubgoal \textbf{P}lanning (\textbf{DSP}), a diffusion-based framework for high-level subgoal generation. DSP casts high-level planning as guided generative inference over goal-conditioned subgoals and learns both conditional and unconditional flows, enabling classifier-free guidance to introduce a goal-directed bias at inference time. By removing explicit value-based guidance from high-level planning, DSP generates reachable and goal-directed subgoals through a generative model while retaining hierarchical execution. Experiments on offline GCRL benchmarks demonstrate that DSP outperforms prior methods on a range of navigation and manipulation tasks, with particularly strong performance in maze environments that require multi-step subgoal planning.
\end{abstract}

\begin{center}
\vspace{-0.2cm}
{
\textbf{Project page:} \href{https://henry0132.github.io/DSP/}{\texttt{https://henry0132.github.io/DSP}} \\
\vspace{0.2cm}
\textbf{Repository:} 
\href{https://github.com/Henry0132/DSP}
{\texttt{https://github.com/Henry0132/DSP}}
}
\end{center}

\section{Introduction}
\label{sec:1}

Goal-conditioned reinforcement learning (GCRL) formulates decision-making problems in terms of desired outcomes, reducing the need to manually design reward functions in high-dimensional spaces and enabling task generalization \cite{GCSL, GoalTutorial, GCPO}. Offline GCRL \cite{OGBench, SMORE} further extends this setting to learning solely from reward-free trajectories, resembling self-supervised learning from past experience. However, without explicit rewards, desired goals are often temporally distant, and value function estimation can suffer from limited data coverage and accumulated errors \cite{HIQL, Pi-HIQL, ProQ}, resulting in noisy and unreliable learning signals in long-horizon tasks.

Hierarchical Implicit Q-Learning (HIQL) \cite{HIQL} addresses long-horizon challenges by introducing a two-level policy, where a high-level policy proposes intermediate subgoals and a low-level policy executes actions conditioned on them. While this hierarchy improves stability over flat policies, large-scale offline GCRL evaluations show that in tasks involving complex navigation, spatial search, or high-dimensional manipulation, HIQL can struggle to obtain reliable guidance from noisy value functions, leading to degraded performance \cite{OGBench, OTA, TTGS}. Subsequent analyses attribute these failures in part to imprecise value guidance and the resulting errors in high-level planning \cite{SHARSA, OTA}. Since hierarchical performance depends critically on subgoal quality, several methods seek to improve value estimation itself, such as Pi-HIQL \cite{Pi-HIQL}, which enforces geometric structure via Eikonal regularization, and OTA \cite{OTA}, which promotes hierarchical consistency through option-aware mechanisms.

However, these approaches share a common assumption that high-level planning should be driven by value function estimates. When subgoals are selected through noisy values, long-horizon planning can remain unstable. This naturally raises a central question:

\begin{quote}
    \emph{Can hierarchical planning be retained without value-based high-level guidance?}
\end{quote}

Diffusion models \cite{DDPM, Score_SDE} provide a promising alternative to value-based planning by generating reachable intermediate states that reflect the structure of the offline data distribution \cite{Diffuser, DD}. Moreover, goal-directed preferences can be incorporated into the sampling process via guidance mechanisms such as classifier-free guidance (CFG) \cite{CFG}, enabling controllable subgoal generation without explicit high-level value estimation.

Motivated by these properties, we propose \textbf{D}iffusion \textbf{S}ubgoal \textbf{P}lanning (\textbf{DSP}), a diffusion-based framework for high-level subgoal generation in offline GCRL. DSP formulates high-level planning as guided generative inference over goal-conditioned subgoals, enabling high-level decision-making without explicit value-based guidance while preserving the benefits of hierarchical execution. Concretely, DSP learns both conditional and unconditional velocity fields, and uses classifier-free guidance at inference time to bias subgoal generation toward the desired goal. This guidance mechanism further admits an implicit advantage-weighted interpretation at the subgoal level. Experiments on offline GCRL benchmarks demonstrate that DSP outperforms prior methods on a range of navigation and manipulation tasks, with particularly strong gains in long-horizon maze environments that require multi-step subgoal planning.

\section{Related Works}
\label{sec:2}

\textbf{Offline GCRL.} Offline goal-conditioned reinforcement learning aims to learn a universal policy that reaches arbitrary target states from arbitrary initial states using a fixed dataset \citep{UVFA, GCSL, GoFresh, OGBench}. Most approaches treat goals as future states and adopt self-supervised training schemes, including hindsight relabeling \citep{HER, CHER}, state occupancy matching \citep{SMORE, CRL}, and related goal relabeling methods \citep{GoFAR, Goplan}. Recent work further incorporates hierarchical structures to decompose complex tasks into multi-step plans via recursive subgoal inference \citep{HIQL, Pi-HIQL, OTA, CGCIVL, HGC-ORL}. Despite differences in formulation, these methods share a common assumption that high-level subgoals are optimized using learning signals derived from goal-conditioned value functions, making noisy value estimates a central bottleneck for reliable long-horizon planning in offline settings.

\textbf{Hierarchical RL.} Hierarchical reinforcement learning introduces decision-making at multiple temporal scales to address long-horizon planning and exploration \citep{LSP-HRL, RIS, Lisa}. Early approaches rely on graph-based subgoal planning \citep{HRAC, HIGL}, while later methods generate intermediate waypoints using temporal distance \citep{SAGAS}, latent abstractions \citep{CoGHP, Lisa}, or value-based criteria \citep{OTA, Guider}. Although these approaches aim to simplify planning by defining reachable high-level subgoals, graph-based methods are often computationally expensive and difficult to scale \citep{World_Graph, DHRL, MBP+Aggregation}, while waypoint prediction methods that depend on value functions or distance estimates are highly sensitive to estimation accuracy, particularly under sparse data coverage \citep{CRL}. As a result, existing hierarchical RL methods can remain sensitive to value-function noise, particularly in offline settings \citep{OGBench}.

\textbf{Diffusion Policies for RL.} Diffusion models have been adopted in reinforcement learning as expressive generative models \citep{DDPM, DDIM, NCSN, Score_SDE, FM, RF}, with applications to action generation \citep{DiffusionQL, QGPO, SRPO, IDQL}, trajectory modeling \citep{DD, SkillDiffuser}, and goal-conditioned or planning-oriented generation \citep{Merlin, PG, MODULI}. Recent methods further use guidance or flow matching to improve action-level policy learning and controllability \citep{CFGRL, FQL}. Other works explore hierarchical diffusion-based planning or subgoal generation \citep{HDMI, HD, SIHD}, where diffusion models are typically used as trajectory or milestone planners \citep{SubgoalDiffuser, DTAMP}, often combined with MPC, return guidance, or multi-step planning procedures. In contrast, DSP uses guided generation as a high-level subgoal decision mechanism in hierarchical offline GCRL, replacing explicit value-based subgoal selection while retaining a separately trained low-level executor.

\section{Preliminaries}
\label{sec:3}

\textbf{Problem Setting.} Offline GCRL is formulated as a finite-horizon discounted Markov decision process $\mathcal{M} = (\mathcal{S}, \mathcal{A}, \mathcal{G}, \mathcal{P}, r, \mathcal{P}_g, d_0, \gamma)$, where $\mathcal{S}$, $\mathcal{A}$, and $\mathcal{G}$ denote the state, action, and goal spaces, respectively. The transition dynamics are given by $\mathcal{P}(s' \mid s, a)$, and the reward function is $r(s, g)$. The initial state distribution $d_0$ and goal distribution $\mathcal{P}_g$ specify how episodes are initialized. The discount factor is $\gamma \in [0, 1)$. Following standard practice, we assume $\mathcal{G} = \mathcal{S}$. At the beginning of each episode, a goal $g \sim \mathcal{P}_g$ is sampled, and the agent aims to reach $g$ by maximizing the expected discounted return $J(\pi)=\mathbb{E}_{\tau \sim \pi(\cdot \mid g)}\left[\sum_{t=0}^{T} \gamma^t r(s_t, g)\right]$, where $\tau = (s_0, a_0, s_1, a_1, \ldots, s_T)$ is a trajectory induced by a goal-conditioned policy $\pi(a \mid s, g)$. The corresponding goal-conditioned value function is $V^\pi(s, g)=\mathbb{E}_{\tau \sim \pi(\cdot \mid g)}\left[\sum_{t=0}^{T} \gamma^t r(s_t, g)\;\middle|\;s_0 = s\right].$ In the offline setting, the policy is learned solely from a fixed dataset $\mathcal{D}$ of trajectories, without further interaction with the environment.

\textbf{Hierarchical Implicit Q-Learning.} In GCRL, accurately estimating value functions for distant goals is particularly challenging in long-horizon tasks. HIQL \cite{HIQL} addresses this issue by introducing a hierarchical policy structure built on implicit Q-learning \cite{IQL}, which directly learns value estimates from offline data. Specifically, HIQL learns a goal-conditioned value function by minimizing the following expectile regression objective:
\begin{align}
\label{equ:hiql_value}
    \mathcal{L}_V
    =
    \mathbb{E}_{(s_t,s_{t+1}) \sim \mathcal{D},\, g \sim p^{\mathcal{D}}}
    \left[
        L_2^\varepsilon
        \left(
            r(s_t,g)
            + \gamma \bar{V}(s_{t+1},g)
            - V(s_t,g)
        \right)
    \right],
\end{align}
where $p^{\mathcal{D}}$ denotes the goal sampling distribution, $\bar V$ is a target value network, and $L_2^\varepsilon$ is the expectile loss 
\begin{align}
\label{equ:expectile}
    L_2^\varepsilon(u)
    =
    \left|\varepsilon - \mathbf{1}(u < 0)\right| u^2,
\end{align}
with expectile parameter $\varepsilon$. Following prior work \cite{WGCSL, GoFAR, GOAT}, the reward is defined as
$r(s_t, g) = - \mathbf{1}(s_t \neq g)$. HIQL decomposes the policy into two levels. The high-level policy $\pi^h(s_{t+k} \mid s_t, g)$ predicts a subgoal $k$ steps ahead, while the low-level policy $\pi^l(a_t \mid s_t, s_{t+k})$ executes actions to reach the proposed subgoal. Both policies are trained using advantage-weighted regression (AWR):
\begin{align}
\label{equ:hiql_high}
    \mathcal{J}(\pi^h)
    =
    \mathbb{E}_{(s_t,s_{t+k}) \sim \mathcal{D},\, g \sim p^{\mathcal{D}}}
    \left[
        \exp\left(\beta^h A^h(s_t,s_{t+k},g)\right)
        \log \pi^h(s_{t+k} \mid s_t,g)
    \right],
\end{align}
\begin{align}
\label{equ:hiql_low}
    \mathcal{J}(\pi^l)
    =
    \mathbb{E}_{(s_t,a_t,s_{t+1},s_{t+k}) \sim \mathcal{D}}
    \left[
        \exp\left(\beta^l A^l(s_t,s_{t+1},s_{t+k})\right)
        \log \pi^l(a_t \mid s_t,s_{t+k})
    \right],
\end{align}
where $\beta^h$ and $\beta^l$ are inverse temperature parameters, and the advantages are defined as $A^h(s_t, s_{t+k}, g) = V(s_{t+k}, g) - V(s_t, g)$ and $A^l(s_t, s_{t+1}, s_{t+k}) = V(s_{t+1}, s_{t+k}) - V(s_t, s_{t+k})$.

\section{High-Level Planning with Reduced Horizon-Dependent Value Noise}
\label{sec:4}

This section analyzes a limitation of HIQL in long-horizon settings and motivates our approach. We first use a one-dimensional illustrative example to show how hierarchical structures improve the signal-to-noise ratio (§\ref{sec:4.1}). We then provide a theoretical analysis showing that DSP removes one horizon-dependent error path caused by noisy high-level value estimates while preserving the benefits of hierarchy (§\ref{sec:4.2}). Finally, we empirically validate these insights on representative offline GCRL benchmarks (§\ref{sec:4.3}).

\subsection{Why Value-Function Noise Limits Long-Horizon Subgoal Planning}
\label{sec:4.1}

To analyze how value-function noise constrains long-horizon subgoal planning, we revisit the one-dimensional toy example introduced in \citep{HIQL}. As shown in Figure~\ref{fig:toy_env}, the environment consists of a one-dimensional state space, where the agent can move left or right at each time step toward a goal located at the rightmost position. The reward is sparse: $r(s_t=g)=0$ and $-1$ otherwise. 
\begin{wrapfigure}{r}{0.4\textwidth}
    \vspace{-0.8em}
    \centering
    \includegraphics[width=0.99\linewidth]{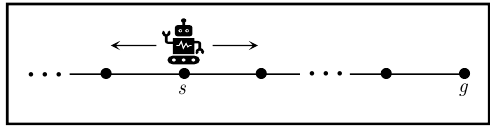}
    \caption{1-D toy environment.}
    \label{fig:toy_env}
    \vspace{-1.0em}
\end{wrapfigure}
Assuming unit discount, the optimal goal-conditioned value function is $V^*(s,g)=-|s-g|$. We model the learned value function as a noisy approximation of the optimal one, $\hat V(s,g) = V^*(s,g) + \sigma z_{s,g} V^*(s,g)$, where $z_{s,g}\sim \mathcal{N}(0,I)$ and $\sigma$ controls the noise scale. This formulation captures a common empirical challenge in offline GCRL: value estimation errors tend to grow with the distance to the goal.

Under this setting, flat policies incur a higher probability of incorrect actions than hierarchical policies, as formalized in Proposition~\ref{prop:4.1}.

\begin{proposition}[Hierarchical Policy Can Reduce Policy Error \citep{HIQL}]
\label{prop:4.1}
In the toy environment, the probability of the flat policy $\pi$ selecting an incorrect action is
\begin{align}
    \mathcal{E}(\pi)
    =
    \Phi\left(
        -\frac{\sqrt{2}}{\sigma\sqrt{T^2+1}}
    \right),
\end{align}
and the probability of the hierarchical policy $\pi^l \circ \pi^h$ selecting an incorrect action is bounded by
\begin{align}
    \mathcal{E}(\pi^l \circ \pi^h)
    \le
    \Phi\left(
        -\frac{\sqrt{2}}{\sigma\sqrt{(T/k)^2+1}}
    \right)
    +
    \Phi\left(
        -\frac{\sqrt{2}}{\sigma\sqrt{k^2+1}}
    \right),
\end{align}
where $\Phi$ denotes the cumulative distribution function of the standard normal distribution,
$\Phi(x)=\mathbb{P}[z\le x]=\frac{1}{\sqrt{2\pi}}\int_{-\infty}^{x} e^{-t^2/2}dt$.
\end{proposition}

Proposition~\ref{prop:4.1} shows that hierarchy improves robustness by reducing the effective decision horizon from $T$ to $T/k$. However, high-level planning in HIQL still depends on value estimates $V(s,g)$, causing estimation errors to propagate through subgoal selection with a $T/k$ scaling. Consequently, hierarchy attenuates but does not eliminate horizon-dependent value noise, which can become a bottleneck in long-horizon tasks.

\subsection{Diffusion Subgoal Planning Reduces Horizon-Dependent Value Noise}
\label{sec:4.2}

We analyze the error structure of DSP in the same toy environment to
contrast its behavior with HIQL. DSP adopts a two-level hierarchical
policy: the high-level policy $\pi_{\mathrm{DSP}}^h$ generates a
subgoal, and the low-level policy $\pi_{\mathrm{DSP}}^\ell$ produces an
action conditioned on that subgoal. The overall policy is therefore
composed as
$\pi_{\mathrm{DSP}}^\ell \circ \pi_{\mathrm{DSP}}^h$.

Appendix~\ref{apdx:proposition2_proof} first derives the exact one-step
decision error for an arbitrary generated subgoal distribution, without
assuming Gaussianity or unimodality. To obtain a transparent closed-form
comparison with HIQL, Proposition~\ref{prop:4.2} considers a
matched-distance setting in which both methods select subgoals at the
same distance $k$. This controls for low-level execution difficulty and
isolates the difference between generative and value-based high-level
direction selection.

\begin{proposition}[Reduced High-Level Value-Noise Amplification]
\label{prop:4.2}
Consider the one-dimensional environment illustrated in
Figure~\ref{fig:toy_env}, with current state $s$, final goal $g=s+T$,
and $1 \le k < T$. Suppose that both DSP and HIQL select one of two
subgoals: $s+k$ in the correct direction or $s-k$ in the incorrect
direction. Let $\epsilon$ denote the probability that DSP selects
$s-k$, and define
\begin{align}
    p_\ell(k)
    &:=
    \Phi\left(
        -\frac{\sqrt{2}}
        {\sigma\sqrt{k^2+1}}
    \right), \quad
    p_h(T,k) :=
    \Phi\left(
        -\frac{\sqrt{2}}
        {\sigma\sqrt{(T/k)^2+1}}
    \right).
    \label{eq:matched_error_terms}
\end{align}
Here, $p_\ell(k)$ is the probability that the low-level policy moves
away from its selected subgoal at distance $k$, while $p_h(T,k)$ is
the probability that HIQL selects the wrong-direction subgoal.

Under the independent Gaussian value-noise model of
Proposition~\ref{prop:4.1}, and assuming that DSP's generated direction
is independent of the low-level value noise, the exact one-step decision
errors are
\begin{align}
    \mathcal{E}\left(
        \pi_{\mathrm{DSP}}^\ell
        \circ
        \pi_{\mathrm{DSP}}^h
    \right)
    &=
    p_\ell(k)
    +
    \left(1-2p_\ell(k)\right)\epsilon,
    \label{eq:dsp_exact_error}\\
    \mathcal{E}\left(
        \pi^\ell
        \circ
        \pi^h
    \right)
    &=
    p_\ell(k)
    +
    \left(1-2p_\ell(k)\right)p_h(T,k).
    \label{eq:hiql_exact_error}
\end{align}
Consequently,
\begin{align}
    \mathcal{E}\left(
        \pi_{\mathrm{DSP}}^\ell
        \circ
        \pi_{\mathrm{DSP}}^h
    \right)
    <
    \mathcal{E}\left(
        \pi^\ell
        \circ
        \pi^h
    \right)
    \quad\Longleftrightarrow\quad
    \epsilon < p_h(T,k).
    \label{eq:dsp_hiql_condition}
\end{align}

For the additional closed-form special case in which
$\sigma_{\mathrm{data}}$ is independent of $T$ and the wrong-direction
probability is parameterized as
\begin{align}
    \epsilon
    =
    \Phi\left(
        -\frac{k}{\sigma_{\mathrm{data}}}
    \right),
    \label{eq:gaussian_direction_error}
\end{align}
condition~\eqref{eq:dsp_hiql_condition} is equivalent to
\begin{align}
    T >
    k\sqrt{
        \left(
            \frac{\sqrt{2}\sigma_{\mathrm{data}}}
            {\sigma k}
        \right)^2
        -1
    }
    \label{eq:dsp_horizon_condition}
\end{align}
when
$\frac{\sqrt{2}\sigma_{\mathrm{data}}}{\sigma k}>1$.
When
$\frac{\sqrt{2}\sigma_{\mathrm{data}}}{\sigma k}\le1$,
condition~\eqref{eq:dsp_hiql_condition} holds for every $T>0$.
\end{proposition}

\textbf{Remark.}
The matched-distance setting assigns DSP and HIQL the same local
subgoal distance and therefore the same low-level error $p_\ell(k)$.
Their ordering is determined solely by whether DSP's wrong-direction
probability $\epsilon$ is smaller than HIQL's high-level value-comparison
error $p_h(T,k)$. As $T/k$ increases, $p_h(T,k)\rightarrow\frac{1}{2}$,
and Equation~\eqref{eq:hiql_exact_error} consequently approaches
$\frac{1}{2}$, reflecting the degradation of long-horizon value-based
direction selection under the noisy-value model.

The Gaussian expression in
Equation~\eqref{eq:gaussian_direction_error} is only a closed-form
special case for parameterizing $\epsilon$. The general derivation in
Appendix~\ref{apdx:proposition2_proof} allows the generated subgoal
distribution to be non-Gaussian, multimodal, and dependent on $T$.
Accordingly, Proposition~\ref{prop:4.2} does not imply
horizon-independent DSP error: the generated subgoal distribution may
deteriorate as the task horizon increases or offline data coverage
decreases. The result instead isolates the removal of the explicit
$T/k$-scaled high-level value-noise path present in HIQL-style subgoal
selection.

The proof of Proposition~\ref{prop:4.2} is provided in
Appendix~\ref{apdx:proposition2_proof}, and empirical validation is
presented in Section~\ref{sec:4.3}.

\subsection{Empirical Evidence: DSP Generates More Reliable Subgoals}
\label{sec:4.3}

\begin{figure*}[t]
    \centering
    \subfigure[HIQL$^\mathrm{w/o}$]{
        \includegraphics[width=0.48\linewidth]{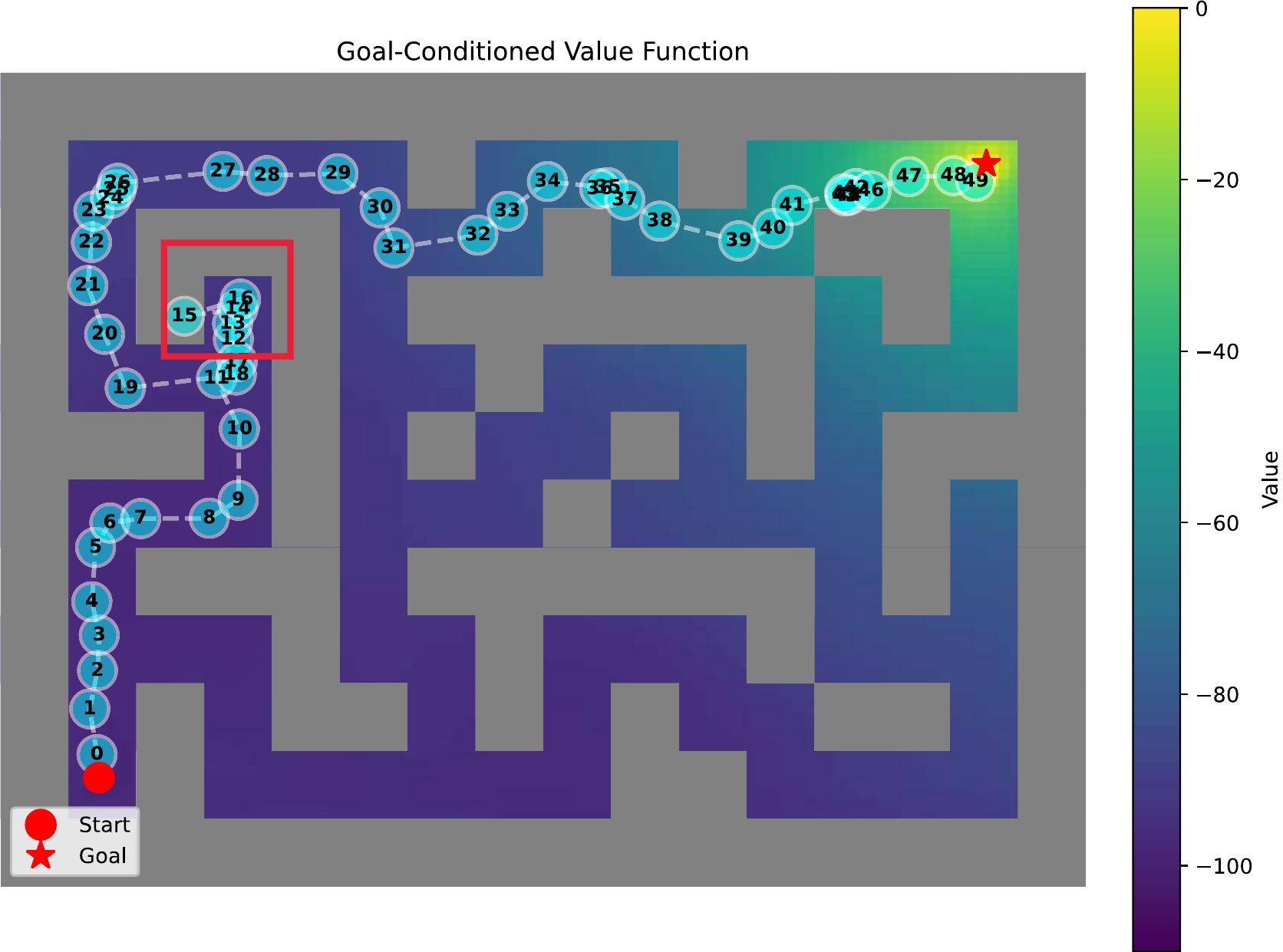}
        \label{fig:hiql_sub}
    }
    \subfigure[DSP]{
        \includegraphics[width=0.48\linewidth]{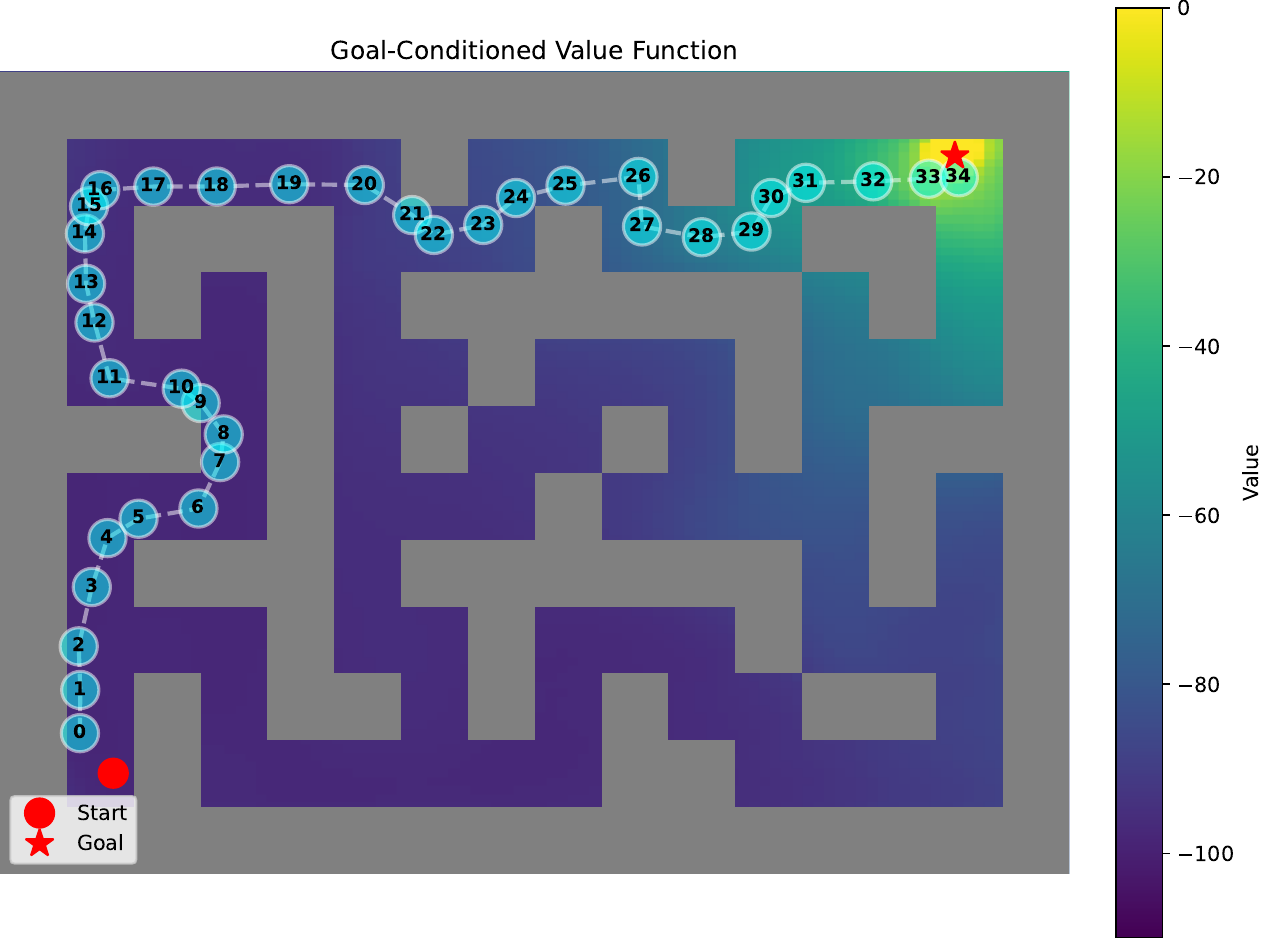}
        \label{fig:dsp_sub}
    }
    \caption{Comparative visualization of high-level planning in the \texttt{antmaze-giant-navigate-v0} environment. The state space is projected onto the 2D $(x,y)$ plane, where the background heatmap illustrates the landscape of the goal-conditioned value function; brighter regions correspond to higher values. HIQL$^\mathrm{w/o}$ selects subgoals through value-based guidance, making its subgoal trajectories sensitive to value-function noise. In contrast, DSP generates smoother and topologically connected paths without explicit value-based high-level guidance.}
    \label{fig:empirical_evidence}
\end{figure*}

We qualitatively compare the high-level subgoals generated by HIQL$^\mathrm{w/o}$\footnote{HIQL$^\mathrm{w/o}$ removes the subgoal representation bottleneck in HIQL and operates directly on state-space subgoals. As shown in Table~\ref{tab:main_result}, HIQL and HIQL$^\mathrm{w/o}$ achieve comparable performance on \texttt{antmaze-giant-navigate-v0}, enabling a fair qualitative comparison.} and DSP on \texttt{antmaze-giant-navigate-v0}, a representative high-dimensional offline GCRL task requiring long-distance navigation through a complex maze under sparse rewards. This environment provides a challenging test for the stability of high-level planning over long horizons.

Figure~\ref{fig:empirical_evidence} overlays the subgoal trajectories produced by HIQL$^\mathrm{w/o}$ and DSP on the same value landscape. As shown in Figure~\ref{fig:hiql_sub}, HIQL$^\mathrm{w/o}$ often proposes subgoals that become trapped in local optima or intersect obstacles, resulting in unstable plans. These behaviors are consistent with our analysis in Sections~\ref{sec:4.1} and~\ref{sec:4.2}, where high-level planning based on noisy value estimates exhibits horizon-dependent error accumulation.

In contrast, DSP produces smoother subgoal sequences that respect the maze topology and maintain progress toward the goal. By generating subgoals through a diffusion-based process rather than value comparison, DSP reduces sensitivity to local errors in the value landscape. As a result, the agent more reliably follows feasible corridors in this long-horizon navigation task. Additional visualizations on mazes of varying sizes are provided in Appendix~\ref{apdx:visualization}.

\section{Diffusion Subgoal Planning}
\label{sec:5}

This section introduces Diffusion Subgoal Planning (DSP). DSP aims to retain the execution efficiency and local guidance benefits of hierarchical structures while removing explicit value-based guidance from high-level subgoal generation. Instead, subgoals are generated through a controllable diffusion-based process grounded in the data distribution. Section~\ref{sec:5.1} formulates subgoal planning as probabilistic inference, Section~\ref{sec:5.2} describes the learning of goal-conditioned velocity fields, and Section~\ref{sec:5.3} presents inference-time control via classifier-free guidance and its connection to implicit advantage-weighted subgoal selection. Section~\ref{sec:5.4} summarizes the complete algorithm.

\textbf{Notation:} Unless otherwise specified, superscripts index diffusion time steps $i \in [0,1]$, while subscripts index trajectory time steps $t \in \{0,\ldots,T-1\}$.

\subsection{Subgoal Generation as Probabilistic Inference}
\label{sec:5.1}

Hierarchical goal-conditioned reinforcement learning can be formulated as a probabilistic inference problem. At the trajectory level, we introduce a binary optimality variable $\mathcal{O}$ as an analytical device indicating whether a trajectory reaches the final goal with high return. Under this formulation, high-level planning corresponds to sampling subgoals from the posterior
\begin{align}
    w \sim p(w \mid s,\mathcal{O} = 1).
\end{align}
This posterior characterizes subgoals that are likely to support goal-reaching behavior from the current state $s$.

Existing approaches typically approximate this posterior by first learning a goal-conditioned value function $V(s,g)$ and then deriving the high-level policy via value-based optimization \cite{HIQL, RIS, OTA}. As analyzed in Section~\ref{sec:4.1}, this approximation couples subgoal selection to value estimation errors, leading to horizon-dependent noise amplification and brittle long-horizon planning.

We instead adopt a generative planning perspective. By Bayes' rule, the posterior can be decomposed as
\begin{align}
    \log p(w \mid s, \mathcal{O}= 1) \propto \log p(w \mid s) + \log p(\mathcal{O} = 1 \mid w,s).
\end{align}
Here, $\log p(w \mid s)$ defines a data-induced prior over reachable and feasible subgoals, while $\log p(\mathcal{O}=1 \mid w,s)$ captures a goal-directed preference. In DSP, this preference is operationalized through the goal condition $g$: the goal-conditioned branch captures how the subgoal distribution changes when the desired goal is specified. Rather than explicitly regressing this preference term through value functions, DSP learns unconditional and goal-conditioned generative velocity fields. Their difference implicitly encodes goal preference at inference time, enabling subgoal sampling that is biased toward goal-relevant waypoints without value-based high-level planning and its associated horizon-dependent noise.

\subsection{Learning the Goal-Directed Velocity Field}
\label{sec:5.2}

We model subgoal generation as a deterministic flow that transforms a simple prior noise variable $w^0 \sim \mathcal{N}(0,I)$ into a goal-conditioned subgoal distribution $p(w \mid s,g)$ via a parameterized velocity field. Rather than predicting noise or scores, we adopt a flow matching objective that directly learns the velocity field defining the generative ordinary differential equation, enabling stable and efficient training \cite{CDC, DFM}.

Given a training subgoal $w^1$, constructed as a $k$-step future state from the offline dataset, i.e., $w^1=s_{t+k}$, we define a linear interpolation path
\begin{align}
\label{equ:flow_interpolation}
    w^i = (1-i)w^0 + iw^1, \quad i \in [0,1],
\end{align}
and supervise the model to match the corresponding velocity along this path, $\frac{d w^i}{d i}=w^1-w^0$. This objective encourages the learned flow to transport samples from the prior toward the empirical subgoal distribution.

To jointly learn the unconditional distribution $p(w \mid s)$ and the goal-conditioned distribution $p(w \mid s,g)$ within a single model, we employ CFG through conditional dropout during training. The resulting objective is
\begin{align}
\label{equ:dsp_high_train}
\mathcal{L}_{\mathrm{DSP}}(\theta)
=
\mathbb{E}_{\substack{
i \sim \mathcal{U}(0,1),\; w^0 \sim \mathcal{N}(0,I),\\
(s_t,w^1) \sim \mathcal{D},\; g \sim p^{\mathcal{D}}
}}
\left[
\left\|
v_\theta(w^i,i,s_t,\tilde g) - (w^1-w^0)
\right\|^2
\right],
\end{align}
where the conditioning variable $\tilde g$ equals $g$ with probability $1-p_{\mathrm{unc}}$ and is set to $\emptyset$ otherwise. This design enables a single network to represent both $v_\theta(\cdot \mid s,g)$ and $v_\theta(\cdot \mid s)$, laying the foundation for controllable guided sampling at inference time.

\subsection{Guided Sampling as Implicit Policy Extraction}
\label{sec:5.3}

Although flow matching learns the empirical goal-conditioned subgoal
distribution during training, DSP further exploits the relationship
between the unconditional and conditional distributions at inference
time. Specifically, classifier-free guidance performs a density-ratio
reweighting of the behavior-induced subgoal distribution.

Let $v_\theta(w^i,i,s,\emptyset)$ and
$v_\theta(w^i,i,s,g)$ denote the unconditional and goal-conditioned
velocity fields. DSP constructs
\begin{align}
    \hat v_\theta(w^i,i)
    =
    v_\theta(w^i,i,s,\emptyset)
    +
    \omega
    \left(
        v_\theta(w^i,i,s,g)
        -
        v_\theta(w^i,i,s,\emptyset)
    \right),
\end{align}
where $\omega$ controls the strength of goal-directed guidance. Under the score approximation interpretation of CFG, the conditional-unconditional difference corresponds to $\nabla_w \log \frac{p(w|s,g)}{p(w|s)}$.

The following proposition establishes the connection between this
density ratio and the RL advantage of the induced high-level behavior
policy.

\begin{proposition}
[CFG as Implicit Relative-Advantage-Weighted Subgoal Extraction]
\label{prop:5.1}
Let $\pi_\beta(w|s)=p(w|s)$ denote the behavior-induced subgoal distribution. Assume that the goal-conditioned subgoal distribution is obtained by future-state relabeling under $\pi_\beta$. Then,
\begin{align}
    \frac{
        p(w|s,g)
    }{
        p(w|s)
    }
    =
    \frac{
        Q_\beta(s,w,g)
    }{
        V_\beta(s,g)
    },
\end{align}

where $Q_\beta(s,w,g)$ is the goal-conditioned value of selecting
subgoal $w$ under $\pi_\beta$, and $V_\beta(s,g) = \mathbb{E}_{w\sim\pi_\beta(\cdot|s)}[Q_\beta(s,w,g)].$

Equivalently,
\begin{align}
    \log
    \frac{
        p(w|s,g)
    }{
        p(w|s)
    }
    =
    \log
    \left(
        1+
        \frac{
            A_\beta(s,w,g)
        }{
            V_\beta(s,g)
        }
    \right),
\end{align}
where $A_\beta(s,w,g) = Q_\beta(s,w,g)-V_\beta(s,g)$ is the behavior-policy RL advantage.

Therefore, CFG sampling induces the tilted distribution
\begin{align}
    p_\omega(w\mid s,g)
    &\propto
    \pi_\beta(w\mid s)
    \left(
        \frac{
            Q_\beta(s,w,g)
        }{
            V_\beta(s,g)
        }
    \right)^\omega
    \propto
    \pi_\beta(w\mid s)
    Q_\beta(s,w,g)^\omega.
\end{align}
Equivalently,
\begin{align}
    p_\omega(w\mid s,g)
    \propto
    \pi_\beta(w\mid s)
    \exp
    \left(
        \omega
        \widetilde A_\beta(s,w,g)
    \right),
\end{align}
where $\widetilde A_\beta(s,w,g) =\log \frac{Q_\beta(s,w,g)}{V_\beta(s,g)}$.
This corresponds to implicit relative-advantage-weighted subgoal
extraction.
\end{proposition}

Proposition~\ref{prop:5.1} clarifies that the implicit quantity amplified
by CFG is the log-relative advantage
$\widetilde A_\beta(s,w,g)=\log \frac{Q_\beta(s,w,g)}{V_\beta(s,g)}$,
which is a strictly increasing transformation of the additive RL
advantage. Therefore, CFG preserves the ordering of subgoals induced by
$A_\beta(s,w,g)$ while avoiding explicit value estimation during
high-level planning. The detailed proof of Proposition~\ref{prop:5.1} is provided in Appendix~\ref{apdx:proposition3_proof}.

\subsection{Practical Algorithm}
\label{sec:5.4}
We combine the diffusion-based high-level planner with the value function and low-level policy learning scheme of HIQL, using the following objectives:
\begin{align}
\label{equ:dsp_value}
\mathcal{L}_V(\phi)
=
\mathbb{E}_{(s_t,s_{t+1})\sim\mathcal{D},\, g\sim p^{\mathcal{D}}}
\left[
L_2^\varepsilon
\left(
r(s_t,g)
+\gamma V_{\bar{\phi}}(s_{t+1},g)
- V_\phi(s_t,g)
\right)
\right].
\end{align}
\begin{align}
\label{equ:dsp_low}
\mathcal{J}_{\pi^\ell}(\vartheta)
=
\mathbb{E}_{(s_t,a_t,s_{t+1},w)\sim\mathcal{D}}
\left[
\exp\left(
\alpha\left(
V_\phi(s_{t+1},w)-V_\phi(s_t,w)
\right)
\right)
\log \pi^\ell_\vartheta(a_t \mid s_t,w)
\right].
\end{align}

Algorithms~\ref{alg:dsp_train} and~\ref{alg:dsp_infer} summarize the training and inference procedures of DSP. The high-level planner generates subgoals via guided diffusion over velocity fields, while the low-level policy executes actions using value-based advantage-weighted updates.

\begin{figure}[h]
\vspace*{-2ex}
\centering
\begin{minipage}[h]{0.42\textwidth}
\begin{algorithm}[H]
\caption{DSP Training}
\label{alg:dsp_train}
\small
\begin{algorithmic}[1]
    \STATE \textbf{Input:} Dataset $\mathcal{D}$; value $V_\phi,V_{\bar\phi}$; flow $v_\theta$; policy $\pi^\ell_\vartheta$.
    \WHILE{not converged}
        \STATE Sample $(s_t,a_t,s_{t+1},w,g)\sim\mathcal{D}$
        \STATE Sample $w^0\sim\mathcal{N}(0,I)$, $i\sim\mathcal{U}(0,1)$
        \STATE Set $w^1\leftarrow w$, $w^i\leftarrow(1-i)w^0+iw^1$
        \STATE Update $V_\phi$ using Eq.~(\ref{equ:dsp_value})
        \STATE $\bar\phi\leftarrow(1-\kappa)\bar\phi+\kappa\phi$
        \STATE Update $v_\theta$ using Eq.~(\ref{equ:dsp_high_train})
        \STATE Update $\pi^\ell_\vartheta$ using Eq.~(\ref{equ:dsp_low})
    \ENDWHILE
    \STATE \textbf{Output:} Trained $V_\phi$, $v_\theta$, and $\pi^\ell_\vartheta$
\end{algorithmic}
\end{algorithm}
\end{minipage}
\hfill
\begin{minipage}[h]{0.57\textwidth}
\begin{algorithm}[H]
\caption{DSP Inference}
\label{alg:dsp_infer}
\small
\begin{algorithmic}[1]
    \STATE \textbf{Input:} Goal $g$; state $s_0$; flow $v_\theta$; policy $\pi^\ell_\vartheta$; guidance $\omega$; steps $N$.
    \STATE \textbf{Define:} $\Delta v_\theta(w,i,s,g)=v_\theta(w,i,s,g)-v_\theta(w,i,s,\emptyset)$
    \STATE Initialize $s_t\leftarrow s_0$
    \WHILE{not done}
        \STATE Sample $w\sim\mathcal{N}(0,I)$
        \FOR{$i\in\{0,\frac{1}{N},\ldots,\frac{N-1}{N}\}$}
            \STATE $w\leftarrow w+\frac{1}{N}\!\left[v_\theta(w,i,s_t,\emptyset)+\omega\Delta v_\theta(w,i,s_t,g)\right]$
        \ENDFOR
        \STATE $a_t\sim\pi^\ell_\vartheta(\cdot\mid s_t,w)$
        \STATE Execute $a_t$ and observe $s_{t+1}$, $s_t\leftarrow s_{t+1}$
    \ENDWHILE
\end{algorithmic}
\end{algorithm}
\end{minipage}
\vspace*{-1ex}
\end{figure}

\section{Experiments}
\label{sec:6}
We evaluate DSP on OGBench, a challenging offline GCRL benchmark designed to test long-horizon reasoning and multi-goal composition. Section~\ref{sec:6.1} describes the experimental setup. Sections~\ref{sec:6.2} and~\ref{sec:6.3} report results on locomotion and manipulation tasks, respectively. Implementation details and hyperparameters are provided in Appendices~\ref{apdx:impl_details} and~\ref{apdx:hyperparameters}. Hyperparameter sensitivity, high-level planner ablations, inference-time analysis, and stitching experiments are reported in Appendices~\ref{apdx:hyper_sensitivity},~\ref{app:controlled_high_level_ablation},~\ref{apdx:infer_time}, and~\ref{apdx:stitch}, respectively.

\subsection{Experimental Setup}
\label{sec:6.1}

We evaluate DSP on a subset of environments and datasets from OGBench, a benchmark designed for offline GCRL with multiple goal-conditioned evaluation protocols. OGBench includes datasets of varying characteristics to assess long-horizon reasoning, trajectory stitching, and multi-goal composition. We focus on two categories of tasks.

\textbf{Locomotion.} Locomotion tasks require controlling a robot to navigate mazes and reach target locations. We consider \texttt{point}, \texttt{ant}, and \texttt{humanoid} embodiments across \texttt{medium}, \texttt{large}, and \texttt{giant} maze layouts. We further include the challenging \texttt{antsoccer} task, where an ant robot must push a ball to a target location, evaluated under both \texttt{arena} and \texttt{medium} settings.

\textbf{Manipulation.} Manipulation tasks involve a 6-DoF robotic arm performing object-centric interactions, including cube grasping, button pressing, window opening, and drawer manipulation. These tasks evaluate the agent’s ability to compose sequential behaviors under complex dynamics.

For comparison, we consider both flat and hierarchical baselines: \textbf{GCBC} \cite{GCSL}, \textbf{CFGRL} \cite{CFGRL}, \textbf{GCIVL} \cite{OGBench}, \textbf{OTA} \cite{OTA}, \textbf{Pi-HIQL} \cite{Pi-HIQL}, and \textbf{HIQL} \cite{HIQL}. We follow the hyperparameter settings reported in the original works and rerun their released code to report mean and standard deviation. Additional details on OGBench and the baseline methods are provided in Appendices~\ref{apdx:env_and_dat} and~\ref{apdx:baselines}.

\subsection{Locomotion Results}
\label{sec:6.2}

\begin{table*}[t]
\centering
\caption{Complete comparison between DSP and the offline GCRL baselines. The table reports the average binary success rate (\%) across five test-time goals for each task, averaged over 5 seeds. Standard deviations are indicated by the $\pm$ symbol. Entries within 95\% of the best-performing value in each row are highlighted in \textbf{bold}.}
\label{tab:main_result}
\vspace{-0.1cm}
\resizebox{\linewidth}{!}{
\begin{tabular}{lcccccccc}
\toprule
\textbf{Datasets} & 
\multicolumn{3}{c}{\textbf{Flat Policies}} & 
\multicolumn{5}{c}{\textbf{Hierarchical Policies}} \\
\cmidrule(r){2-4} \cmidrule(l){5-9}
& 
\textbf{GCBC} & 
\textbf{CFGRL} & 
\textbf{GCIVL} & 
\textbf{OTA} & 
\textbf{Pi-HIQL} & 
\textbf{HIQL} & 
\textbf{HIQL$^{\mathrm{w/o}}$} & 
\textbf{DSP} \\
\midrule
\texttt{pointmaze-medium-navigate-v0} 
& $4.8\pm6.9$ 
& $65.2\pm5.3$ 
& $70.2\pm5.9$ 
& $85.4\pm5.0$ 
& $62.2\pm7.3$ 
& $70.6\pm7.0$ 
& $74.6\pm6.1$ 
& $\mathbf{90.6}\pm3.8$ \\

\texttt{pointmaze-large-navigate-v0} 
& $25.6\pm6.3$ 
& $74.6\pm2.9$ 
& $42.6\pm5.3$ 
& $87.6\pm9.2$ 
& $80.2\pm13.3$ 
& $39.4\pm2.3$ 
& $50.0\pm10.4$ 
& $\mathbf{93.6}\pm2.9$ \\

\texttt{pointmaze-giant-navigate-v0} 
& $2.2\pm4.9$ 
& $3.4\pm2.5$ 
& $1.0\pm2.2$ 
& $\mathbf{69.4}\pm11.6$ 
& $3.0\pm2.9$ 
& $5.0\pm5.0$ 
& $0.0\pm0.0$ 
& $43.4\pm7.2$ \\

\texttt{pointmaze-teleport-navigate-v0} 
& $23.8\pm5.7$ 
& $\mathbf{49.2}\pm5.1$ 
& $43.6\pm2.5$ 
& $40.2\pm5.4$ 
& $33.6\pm9.0$ 
& $8.8\pm6.4$ 
& $8.2\pm6.2$ 
& $\mathbf{51.0}\pm5.8$ \\

\midrule
\texttt{antmaze-medium-navigate-v0} 
& $32.6\pm8.5$ 
& $34.8\pm5.5$ 
& $74.6\pm8.0$ 
& $91.6\pm2.4$ 
& $90.8\pm2.7$ 
& $93.2\pm1.3$ 
& $92.8\pm2.4$ 
& $\mathbf{98.4}\pm0.5$ \\

\texttt{antmaze-large-navigate-v0} 
& $23.0\pm2.0$ 
& $18.4\pm5.7$ 
& $15.0\pm6.6$ 
& $\mathbf{89.2}\pm2.3$ 
& $81.8\pm3.1$ 
& $\mathbf{87.8}\pm1.5$ 
& $\mathbf{88.6}\pm2.1$ 
& $\mathbf{92.0}\pm4.1$ \\

\texttt{antmaze-giant-navigate-v0} 
& $0.0\pm0.0$ 
& $0.0\pm0.0$ 
& $0.0\pm0.0$ 
& $\mathbf{68.4}\pm3.4$ 
& $50.4\pm3.7$ 
& $58.2\pm4.7$ 
& $50.4\pm8.0$ 
& $\mathbf{69.2}\pm2.9$ \\

\texttt{antmaze-teleport-navigate-v0} 
& $26.4\pm2.2$ 
& $35.2\pm7.5$ 
& $35.6\pm5.7$ 
& $48.4\pm2.6$ 
& $49.4\pm2.7$ 
& $43.4\pm6.6$ 
& $40.8\pm4.1$ 
& $\mathbf{59.8}\pm1.3$ \\

\midrule
\texttt{humanoidmaze-medium-navigate-v0} 
& $6.8\pm1.3$ 
& $7.0\pm2.3$ 
& $32.6\pm4.0$ 
& $82.4\pm3.6$ 
& $76.2\pm3.8$ 
& $74.8\pm3.0$ 
& $52.6\pm4.3$ 
& $\mathbf{89.6}\pm2.9$ \\

\texttt{humanoidmaze-large-navigate-v0} 
& $1.0\pm1.4$ 
& $1.2\pm1.6$ 
& $2.6\pm1.8$ 
& $\mathbf{76.4}\pm4.0$ 
& $39.0\pm3.2$ 
& $23.6\pm2.9$ 
& $7.2\pm3.4$ 
& $69.4\pm3.9$ \\

\texttt{humanoidmaze-giant-navigate-v0} 
& $0.0\pm0.0$ 
& $0.4\pm0.5$ 
& $0.0\pm0.0$ 
& $\mathbf{81.4}\pm4.9$ 
& $34.2\pm3.3$ 
& $5.0\pm1.6$ 
& $0.0\pm0.0$ 
& $66.0\pm3.4$ \\

\midrule
\texttt{antsoccer-arena-navigate-v0} 
& $5.4\pm3.0$ 
& $12.2\pm4.1$ 
& $47.4\pm7.5$ 
& $39.8\pm6.5$ 
& $17.8\pm2.9$ 
& $60.6\pm4.5$ 
& $56.6\pm5.5$ 
& $\mathbf{76.6}\pm4.0$ \\

\texttt{antsoccer-medium-navigate-v0} 
& $1.4\pm0.5$ 
& $3.6\pm2.1$ 
& $11.6\pm2.3$ 
& $\mathbf{16.6}\pm3.5$ 
& $2.4\pm1.5$ 
& $8.4\pm1.5$ 
& $7.2\pm1.5$ 
& $15.2\pm2.4$ \\

\midrule
\texttt{cube-single-play-v0} 
& $5.2\pm2.2$ 
& $7.6\pm2.3$ 
& $\mathbf{53.4}\pm5.4$ 
& $10.0\pm4.2$ 
& $1.2\pm1.1$ 
& $11.8\pm2.3$ 
& $21.8\pm4.0$ 
& $41.0\pm6.1$ \\

\texttt{cube-double-play-v0} 
& $1.0\pm1.0$ 
& $1.6\pm1.1$ 
& $\mathbf{31.6}\pm3.0$ 
& $2.2\pm0.8$ 
& $0.0\pm0.0$ 
& $3.6\pm2.2$ 
& $9.4\pm4.4$ 
& $\mathbf{30.8}\pm8.4$ \\

\midrule
\texttt{scene-play-v0} 
& $4.8\pm2.6$ 
& $17.2\pm3.3$ 
& $44.6\pm2.9$ 
& $23.4\pm7.1$ 
& $16.6\pm3.1$ 
& $37.2\pm4.2$ 
& $42.0\pm5.2$ 
& $\mathbf{62.6}\pm5.9$ \\

\midrule
\texttt{puzzle-3x3-play-v0} 
& $3.0\pm1.2$ 
& $2.0\pm1.2$ 
& $4.2\pm2.5$ 
& $4.6\pm3.2$ 
& $5.6\pm1.5$ 
& $8.8\pm1.3$ 
& $\mathbf{16.0}\pm2.9$ 
& $13.4\pm2.3$ \\

\texttt{puzzle-4x4-play-v0} 
& $0.0\pm0.0$ 
& $0.0\pm0.0$ 
& $7.8\pm2.6$ 
& $3.4\pm1.8$ 
& $13.6\pm8.4$ 
& $7.8\pm2.6$ 
& $5.4\pm2.3$ 
& $\mathbf{27.6}\pm2.9$ \\

\midrule
\texttt{Total} 
& $167.0$ 
& $333.6$ 
& $518.4$ 
& $920.4$ 
& $658.0$ 
& $648.0$ 
& $623.6$ 
& $\mathbf{1090.2}$ \\

\bottomrule
\end{tabular}
}
\end{table*}

We first evaluate DSP on the locomotion suite of OGBench. As summarized in Table~\ref{tab:main_result}, DSP achieves strong performance across the evaluated locomotion tasks, showing its effectiveness for offline long-horizon planning. In low-dimensional \texttt{point} environments, DSP outperforms prior methods in the \texttt{medium} and \texttt{large} mazes and remains competitive in the more challenging \texttt{giant} setting. The learning curves in Figure~\ref{fig:main_learning_curve} further show that DSP converges more reliably than HIQL-based baselines in several difficult \texttt{giant} environments.

The gains also extend to higher-dimensional robotic domains, including \texttt{ant} and \texttt{humanoid}. DSP attains near-optimal success rates in \texttt{antmaze-medium} and \texttt{antmaze-large}, and remains among the strongest methods on several \texttt{humanoidmaze} tasks, where high-dimensional states and complex dynamics make long-horizon value-based planning more difficult. These results are consistent with the intended role of DSP: generating data-supported, goal-directed subgoals without relying on explicit value comparisons at the high level.

DSP also achieves the best or near-best performance on the \texttt{antsoccer} tasks under both \texttt{arena} and \texttt{medium} configurations. These environments combine long-range navigation with object interaction, providing a useful test of planning consistency under coupled dynamics. Overall, the locomotion results suggest that guided generative subgoal planning is particularly effective in long-horizon navigation tasks that require stable multi-step subgoal generation.

\begin{figure*}[t]
    \centering
    \includegraphics[width=\linewidth]{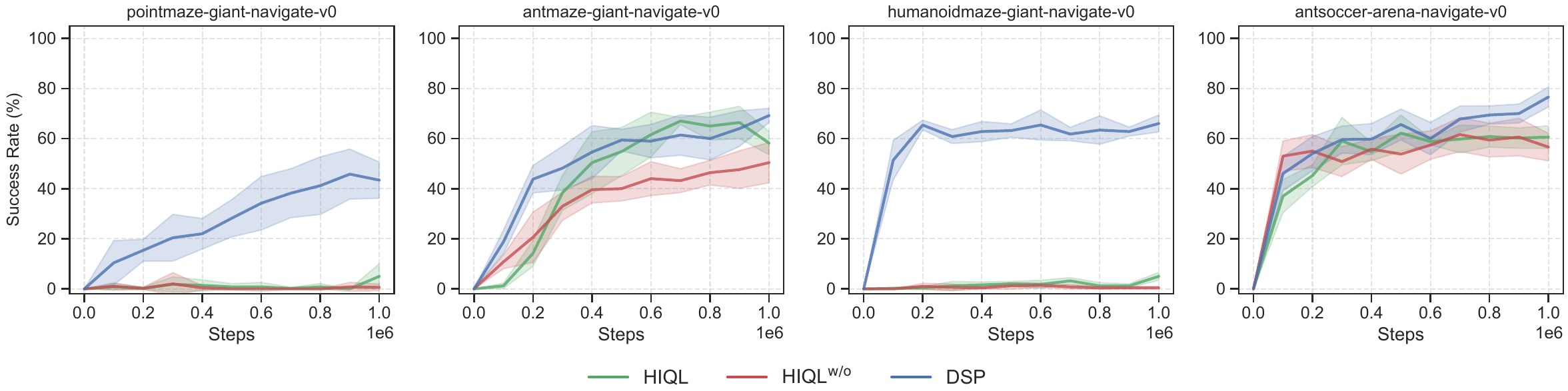}
    \caption{Learning curves on challenging OGBench navigation tasks. DSP generally converges more reliably than HIQL and HIQL$^{\mathrm{w/o}}$ in these long-horizon settings, suggesting the benefit of guided generative subgoal planning.}
    \label{fig:main_learning_curve}
\end{figure*}

\subsection{Manipulation Results}
\label{sec:6.3}

We further evaluate DSP on the manipulation suite of OGBench, spanning single-object grasping, multi-object interaction, and compositional puzzle solving. Unlike locomotion, these tasks emphasize sequencing local physical interactions rather than long-range navigation.

As shown in Table~\ref{tab:main_result}, GCIVL performs strongly on local tasks such as \texttt{cube-single} and \texttt{cube-double}, where compact goal spaces and short-horizon value estimates provide useful guidance. In these cases, flat policies remain competitive, and hierarchy may offer limited benefit.

DSP shows clearer advantages on multi-step interaction tasks such as \texttt{scene-play} and compositional tasks in the \texttt{puzzle} suite, where consistent subgoal coordination is more important. By generating data-supported subgoals without explicit high-level value comparisons, DSP performs strongly on \texttt{scene-play} and \texttt{puzzle-4x4}, while remaining competitive across the suite.

Overall, these results complement the locomotion findings: DSP is most useful when high-level planning requires stable multi-step subgoal generation, while remaining competitive on more local manipulation tasks.

\subsection{Guidance, Data Support, and Executor Reachability}
\begin{wraptable}{r}{0.53\textwidth}
    \vspace{-0.8em}
    \centering
    \caption{
        Guidance--support--reachability trade-off.
        $D$: support distance; $V$: validity;
        $R$: valid-and-reached rate.
        $V$, $R$, and Succ. are percentages.
    }
    \label{tab:guidance_reachability}
    \footnotesize
    \setlength{\tabcolsep}{2.8pt}
    \begin{tabular*}{\linewidth}{@{\extracolsep{\fill}}lccccc@{}}
    \toprule
    Datasets & $\omega$ & $D\downarrow$
    & $V\uparrow$ & $R\uparrow$ & Succ.$\uparrow$ \\
    \midrule
    \multirow[c]{4}{*}{\texttt{AntMaze-Giant}}
    & 1  & 1.6 & 100.0 & 68.8 & 46.0 \\
    & 3  & 1.7 & 100.0 & 62.1 & \textbf{69.2} \\
    & 5  & 1.8 & 98.8  & 51.6 & 56.4 \\
    & 10 & 2.1 & 96.9  & 40.6 & 27.0 \\
    \midrule
    \multirow[c]{4}{*}{\texttt{Scene-Play}}
    & 1  & 1.8 & 98.0 & 92.4 & 35.6 \\
    & 3  & 2.0 & 96.5 & 90.7 & 50.4 \\
    & 5  & 2.4 & 94.5 & 86.6 & \textbf{62.6} \\
    & 10 & 3.7 & 84.4 & 45.3 & 35.2 \\
    \bottomrule
\end{tabular*}
    \vspace{-0.8em}
\end{wraptable}
To quantify how guidance affects subgoal feasibility, we hold the trained checkpoints and all other settings fixed and vary $\omega$. We measure standardized training-support distance $D$, state-validity rate $V$, valid-and-reached rate $R$, and task success. Lower $D$ indicates greater proximity to the offline data, while higher $V$ and $R$ indicate more valid and executable subgoals. Full metric definitions are provided in Appendix~\ref{app:guidance_reachability}.

Relative to $\omega=1$, the family-default scales improve success by 23.2 and 27.0 percentage points on Ant-Giant and Scene-Play, respectively, while reducing the valid-and-reached rate by 6.7 and 5.8 points. At $\omega=10$, support distance increases and validity, reachability, and success all deteriorate. These results show that moderate guidance balances goal direction and executability, whereas excessive guidance is a practical failure mode; DSP does not provide a hard subgoal-feasibility guarantee.

\section{Conclusion}
\label{sec:7}
In this work, we propose Diffusion Subgoal Planning (DSP), a generative framework for high-level subgoal planning in offline goal-conditioned reinforcement learning. DSP targets a key limitation of hierarchical methods: high-level subgoal selection can become unreliable when guided by noisy goal-conditioned value estimates in long-horizon tasks. Instead of explicitly querying values for high-level planning, DSP learns unconditional and goal-conditioned velocity fields and uses classifier-free guidance to bias generated subgoals toward the desired goal. Combined with a low-level policy trained through implicit value learning, DSP achieves strong performance across challenging OGBench locomotion and manipulation tasks. These results suggest that guided generative subgoal modeling is a promising alternative to explicit value-based high-level planning in long-horizon offline GCRL.

\section*{Acknowledgments}
This work was supported by the Key Research and Development Program of Jiangsu Province under Grant BE2022095. We thank Qiyu Wang at University College London for helpful discussions and constructive feedback on this work. We also thank the reviewers and Area Chair for their constructive feedback, which helped improve the paper.

\clearpage
\bibliographystyle{plainnat}
\bibliography{reference}


\newpage
\appendix

\section{Limitations and Future Work}
\label{apdx:limitations}

DSP targets offline, state-based goal-conditioned RL settings where the dataset contains meaningful local reachability structure. As with many offline generative planning methods, performance may be affected when data coverage is limited or evaluation goals deviate substantially from the behavior data. DSP also uses iterative diffusion sampling for high-level subgoal generation, introducing an inference cost controlled by the number of sampling steps; in our experiments, a modest number of steps is sufficient for strong performance. 

Future work may extend DSP to partially observable domains, develop more efficient generative parameterizations for faster subgoal sampling, and explore improved learning objectives and optimization strategies for more stable and efficient training \citep{ma2026orthogonal, PreResQ-R1}. We may also explore multi-timescale hierarchical world-action modeling \citep{MHWA} for subgoal planning, with adversarial training \citep{li2026learning} and confidence-aware regularization \citep{li2026care} to improve reliability. Pretrained representations and cross-task transfer, inspired by unified video-action modeling \citep{UNIVERSE, DRIVEA}, may help reduce dependence on task-specific data coverage.

\section{Proofs}
\label{apdx:proofs}

\subsection{Proof of Proposition~\ref{prop:4.1}}
\label{apdx:proposition1_proof}
\begin{proof}
    For simplicity, we assume that $T/k$ is an integer and $k \leq T$. 
    Defining $z_1 := z_{1,T}$ and $z_2 := z_{-1,T}$,
    the probability of the flat policy $\pi$ selecting an incorrect action can be computed as follows:
    \begin{align}
        \mathcal{E}(\pi)
        &= \mathbb{P}[\hat V(s+1, g) \leq \hat V(s-1, g)] \\
        &= \mathbb{P}[\hat V(1, T) \leq \hat V(-1, T)] \\
        &= \mathbb{P}[-(T-1)(1 + \sigma z_1) \leq -(T+1)(1 + \sigma z_2)] \\
        &= \mathbb{P}[z_1 \sigma (T-1) - z_2 \sigma (T+1) \geq 2] \\
        &= \mathbb{P}[z \sigma \sqrt{T^2 + 1} \leq - \sqrt 2] \\
        &= \Phi\left(-\frac{\sqrt 2}{\sigma \sqrt{T^2 + 1}}\right),
    \end{align}
    where $z$ is a standard Gaussian random variable,
    and we use the fact that the sum of two independent Gaussian random variables
    with standard deviations of $\sigma_1$ and $\sigma_2$
    follows a normal distribution with a standard deviation of $\sqrt{\sigma_1^2 + \sigma_2^2}$.
    Similarly, the probability of the hierarchical policy $\pi^l \circ \pi^h$ selecting an incorrect action is bounded using a union bound as
    \begin{align}
        \mathcal{E}(\pi^l \circ \pi^h)
        &\leq \mathcal{E}(\pi^h) + \mathcal{E}(\pi^l) \\
        &= \mathbb{P}[\hat V(s+k, g) \leq \hat V (s-k, g)] + \mathbb{P}[\hat V(s+1, s+k) \leq \hat V(s-1, s+k)] \\
        &= \mathbb{P}[\hat V(k, T) \leq \hat V (-k, T)] + \mathbb{P}[\hat V(1, k) \leq \hat V(-1, k)] \\
        &= \Phi\left(-\frac{\sqrt 2}{\sigma \sqrt{(T/k)^2 + 1}}\right) + \Phi\left(-\frac{\sqrt 2}{\sigma \sqrt{k^2 + 1}}\right).
    \end{align}
\end{proof}

\subsection{Proof of Proposition~\ref{prop:4.2}}
\label{apdx:proposition2_proof}
\begin{proof}
We first derive the exact one-step decision error for an arbitrary
generated subgoal distribution. This general result is then specialized
to the matched-distance setting of Proposition~\ref{prop:4.2}.

\paragraph{Arbitrary generated subgoals.}
Let $w\sim\pi_{\mathrm{DSP}}^h(\cdot\mid s,g)$ be an arbitrary generated subgoal, assumed to be independent of the
low-level value-noise variables. The final goal $g=s+T$ lies to the
right of the current state $s$, so a globally incorrect primitive
action moves to the left.

First consider $w>s$ and define $d:=w-s>0.$
The distances from the two successor states to $w$ are
\begin{align}
    |s+1-w| = |d-1|, \quad |s-1-w| = d+1.
\end{align}
Let $z_+ := z_{s+1,w}, z_- := z_{s-1,w},$
where $z_+,z_-\sim\mathcal{N}(0,1)$ independently. Under the noisy-value
model,
\begin{align}
    \hat V(s+1,w)
    &=
    -|d-1|\left(1+\sigma z_+\right),\\
    \hat V(s-1,w)
    &=
    -(d+1)\left(1+\sigma z_-\right).
\end{align}
Since both $w$ and the final goal lie to the right of $s$, the
low-level policy makes a globally incorrect decision when it selects
the left action:
\begin{align}
    \hat V(s+1,w)
    \le
    \hat V(s-1,w).
\end{align}
Substituting the noisy values and rearranging gives
\begin{align}
    \sigma\left(
        |d-1|z_+-(d+1)z_-
    \right)
    \ge
    (d+1)-|d-1|.
\end{align}
Using
\begin{align}
    (d+1)-|d-1|
    &=
    2\min\{d,1\},\\
    |d-1|^2+(d+1)^2
    &=
    2(d^2+1),
\end{align}
we obtain
\begin{align}
    &\mathbb{P}\left[
        \hat V(s+1,w)
        \le
        \hat V(s-1,w)
        \mid w
    \right]
    =
    \Phi\left(
        -\frac{\sqrt{2}\min\{d,1\}}
        {\sigma\sqrt{d^2+1}}
    \right)
    =:
    p_\ell(d).
    \label{eq:conditional_low_level_error}
\end{align}

Now consider $w<s$ and define $d:=s-w>0$. By symmetry,
$p_\ell(d)$ is the probability that the low-level policy moves to the
right, away from the selected subgoal. In this case, however, the right
action is globally correct because the final goal lies to the right.
Hence, conditioned on $w<s$, the probability of a globally incorrect
left action is
\[
    1-p_\ell(d).
\]
When $w=s$, the two actions are symmetric under the noise model, so the
globally incorrect left action is selected with probability $1/2$.
Taking expectation over $w$ therefore yields
\begin{align}
    &\mathcal{E}\left(
        \pi_{\mathrm{DSP}}^\ell
        \circ
        \pi_{\mathrm{DSP}}^h
    \right)
    =
    \mathbb{E}_{w}\left[
        p_\ell(w-s)\mathbf{1}\{w>s\}
        +
        \left(1-p_\ell(s-w)\right)\mathbf{1}\{w<s\}
    \right]
    +
    \frac{1}{2}\mathbb{P}(w=s).
    \label{eq:dsp_general_error}
\end{align}
Equation~\eqref{eq:dsp_general_error} holds for an arbitrary generated
subgoal distribution and does not require Gaussianity, symmetry, or
unimodality.

\paragraph{Matched-distance DSP error.}
Under the matched-distance setting of Proposition~\ref{prop:4.2}, DSP
selects $s+k$ with probability $1-\epsilon$ and $s-k$ with probability
$\epsilon$. Since $k\ge1$, Equation~\eqref{eq:conditional_low_level_error}
gives
\[
    p_\ell(k)
    =
    \Phi\left(
        -\frac{\sqrt{2}}
        {\sigma\sqrt{k^2+1}}
    \right).
\]
If DSP selects $s+k$, the globally incorrect action is chosen with
probability $p_\ell(k)$. If DSP selects $s-k$, the globally incorrect
action is chosen with probability $1-p_\ell(k)$. Therefore,
\begin{align}
    \mathcal{E}\left(
        \pi_{\mathrm{DSP}}^\ell
        \circ
        \pi_{\mathrm{DSP}}^h
    \right)
    &=
    (1-\epsilon)p_\ell(k)
    +
    \epsilon\left(1-p_\ell(k)\right)
    \nonumber\\
    &=
    p_\ell(k)
    +
    \left(1-2p_\ell(k)\right)\epsilon,
\end{align}
which proves Equation~\eqref{eq:dsp_exact_error}.

\paragraph{Matched-distance HIQL error.}
Let $H$ denote the event that HIQL selects the wrong-direction subgoal
$s-k$, and let $L$ denote the event that its low-level policy moves away
from the selected subgoal. Proposition~\ref{prop:4.1} gives
\begin{align}
    \mathbb{P}(H)
    &=
    p_h(T,k)
    =
    \Phi\left(
        -\frac{\sqrt{2}}
        {\sigma\sqrt{(T/k)^2+1}}
    \right).
\end{align}
By symmetry, the conditional low-level error is $p_\ell(k)$ for either
selected subgoal:
\begin{align}
    \mathbb{P}(L\mid H)
    =
    \mathbb{P}(L\mid H^c)
    =
    p_\ell(k).
\end{align}
Because the high- and low-level comparisons use independent value-noise
variables, $H$ and $L$ are independent.

The final primitive action is globally incorrect in exactly two cases:
the high-level direction is correct and the low-level policy moves away
from its subgoal, or the high-level direction is incorrect and the
low-level policy moves toward its subgoal. Hence,
\begin{align}
    \mathcal{E}\left(
        \pi^\ell\circ\pi^h
    \right)
    &=
    \mathbb{P}(H^c\cap L)
    +
    \mathbb{P}(H\cap L^c)
    \nonumber\\
    &=
    \left(1-p_h(T,k)\right)p_\ell(k)
    +
    p_h(T,k)\left(1-p_\ell(k)\right)
    \nonumber\\
    &=
    p_\ell(k)
    +
    \left(1-2p_\ell(k)\right)p_h(T,k),
\end{align}
which proves Equation~\eqref{eq:hiql_exact_error}.

Since $p_\ell(k)<1/2$, we have
\begin{align}
    &\mathcal{E}\left(
        \pi_{\mathrm{DSP}}^\ell
        \circ
        \pi_{\mathrm{DSP}}^h
    \right)
    <
    \mathcal{E}\left(
        \pi^\ell\circ\pi^h
    \right)
    \quad\Longleftrightarrow\quad
    \epsilon<p_h(T,k),
\end{align}
which proves Equation~\eqref{eq:dsp_hiql_condition}.

\paragraph{Gaussian directional-error special case.}
For the closed-form special case in
Equation~\eqref{eq:gaussian_direction_error}, consider an auxiliary
directional proposal
\[
    u\sim
    \mathcal{N}
    \left(
        s+k,\sigma_{\mathrm{data}}^2
    \right),
\]
where the matched-distance planner selects $s+k$ when $u>s$ and $s-k$
otherwise. This auxiliary proposal is used only to parameterize the
wrong-direction probability $\epsilon$. We have
\begin{align}
    \epsilon
    &=
    \mathbb{P}(u\le s)
    \nonumber\\
    &=
    \mathbb{P}\left[
        \frac{u-(s+k)}
        {\sigma_{\mathrm{data}}}
        \le
        -\frac{k}{\sigma_{\mathrm{data}}}
    \right]
    \nonumber\\
    &=
    \Phi\left(
        -\frac{k}{\sigma_{\mathrm{data}}}
    \right).
\end{align}
Condition~\eqref{eq:dsp_hiql_condition} therefore becomes
\begin{align}
    \Phi\left(
        -\frac{k}{\sigma_{\mathrm{data}}}
    \right)
    <
    \Phi\left(
        -\frac{\sqrt{2}}
        {\sigma\sqrt{(T/k)^2+1}}
    \right).
\end{align}
By the strict monotonicity of $\Phi$, this is equivalent to
\begin{align}
    \sqrt{(T/k)^2+1}
    >
    \frac{\sqrt{2}\sigma_{\mathrm{data}}}
    {\sigma k}.
    \label{eq:threshold_intermediate}
\end{align}
When
\[
    \frac{\sqrt{2}\sigma_{\mathrm{data}}}
    {\sigma k}>1,
\]
rearranging Equation~\eqref{eq:threshold_intermediate} yields
\begin{align}
    T >
    k\sqrt{
        \left(
            \frac{\sqrt{2}\sigma_{\mathrm{data}}}
            {\sigma k}
        \right)^2
        -1
    },
\end{align}
which is Equation~\eqref{eq:dsp_horizon_condition}. When
$\frac{\sqrt{2}\sigma_{\mathrm{data}}}{\sigma k}\le1$,
the left-hand side of Equation~\eqref{eq:threshold_intermediate} is
strictly greater than one for every $T>0$, so no additional positive
lower bound on $T$ is required.
\end{proof}

\subsection{Proof of Proposition~\ref{prop:5.1}}
\label{apdx:proposition3_proof}

\begin{proof}
We separate the proof into two parts. We first establish the exact
relationship between the conditional density ratio and the
behavior-policy goal-reaching value. We then connect classifier-free
guidance to $Q_\beta$-weighted subgoal extraction.

\paragraph{Density ratio as normalized goal-reaching value.}

By construction of the future-state relabeling process,
$Q_\beta(s,w,g)$ is the conditional probability mass or density of
sampling $g$ from the future continuation after first selecting $w$:
\begin{align}
    p_{\mathcal D}(g\mid s,w)
    =
    Q_\beta(s,w,g).
    \label{eq:proof_goal_conditional_q}
\end{align}
Moreover,
$p_{\mathcal D}(w\mid s)=\pi_\beta(w\mid s)$. Marginalizing over the
behavior subgoal distribution gives
\begin{align}
    p_{\mathcal D}(g\mid s)
    =
    \int
    p_{\mathcal D}(g\mid s,w)
    p_{\mathcal D}(w\mid s)
    \,dw                                                    
    =
    \int
    Q_\beta(s,w,g)
    \pi_\beta(w\mid s)
    \,dw                                                    
    =
    V_\beta(s,g).
    \label{eq:proof_goal_marginal_v}
\end{align}
Bayes' rule therefore yields
\begin{align}
    p_{\mathcal D}(w\mid s,g)
    =
    \frac{
        p_{\mathcal D}(g\mid s,w)
        p_{\mathcal D}(w\mid s)
    }{
        p_{\mathcal D}(g\mid s)
    }                                                       
    =
    \pi_\beta(w\mid s)
    \frac{
        Q_\beta(s,w,g)
    }{
        V_\beta(s,g)
    }.
    \label{eq:proof_bayes_q_policy}
\end{align}
Dividing both sides by
$p_{\mathcal D}(w\mid s)=\pi_\beta(w\mid s)$ gives
\begin{align}
    \frac{
        p_{\mathcal D}(w\mid s,g)
    }{
        p_{\mathcal D}(w\mid s)
    }
    =
    \frac{
        Q_\beta(s,w,g)
    }{
        V_\beta(s,g)
    }.
    \label{eq:proof_density_q}
\end{align}
Using
\begin{align}
    Q_\beta(s,w,g)
    =
    V_\beta(s,g)
    +
    A_\beta(s,w,g),
\end{align}
we obtain
\begin{align}
    \frac{
        p_{\mathcal D}(w\mid s,g)
    }{
        p_{\mathcal D}(w\mid s)
    }
    =
    1+
    \frac{
        A_\beta(s,w,g)
    }{
        V_\beta(s,g)
    }.
    \label{eq:proof_density_adv}
\end{align}
Taking logarithms yields
\begin{align}
    \widetilde A_\beta(s,w,g)
    &=
    \log
    \frac{
        p_{\mathcal D}(w\mid s,g)
    }{
        p_{\mathcal D}(w\mid s)
    }                                                       
    =
    \log
    \frac{
        Q_\beta(s,w,g)
    }{
        V_\beta(s,g)
    }                                                       
    =
    \log
    \left(
        1+
        \frac{
            A_\beta(s,w,g)
        }{
            V_\beta(s,g)
        }
    \right).
    \label{eq:proof_log_relative_adv}
\end{align}

For fixed $(s,g)$, $V_\beta(s,g)$ does not depend on $w$. Wherever
$Q_\beta(s,w,g)>0$,
\begin{align}
    \frac{
        \partial
        \widetilde A_\beta(s,w,g)
    }{
        \partial
        A_\beta(s,w,g)
    }
    &=
    \frac{
        1
    }{
        V_\beta(s,g)
        +
        A_\beta(s,w,g)
    }                                                       \notag \\
    &=
    \frac{
        1
    }{
        Q_\beta(s,w,g)
    }
    >0.
    \label{eq:proof_monotone_adv}
\end{align}
Thus, $\widetilde A_\beta$ is a strictly increasing transformation of
the behavior-policy RL advantage. In particular,
\begin{align}
    \widetilde A_\beta(s,w,g)>0
    \quad\Longleftrightarrow\quad
    A_\beta(s,w,g)>0,
\end{align}
and both quantities induce the same ordering over $w$ for fixed
$(s,g)$.

\paragraph{Classifier-free guidance as $Q_\beta$-weighted extraction.}

We next connect the density-ratio identity to the guided velocity field
used by DSP. Under the score-approximation interpretation adopted in
Proposition~\ref{prop:5.1}, the unconditional and goal-conditioned
velocity fields satisfy
\begin{align}
    v_\theta(w^i,i,s,\emptyset)
    &\approx
    \nabla_{w^i}
    \log p_i(w^i\mid s), \quad
    v_\theta(w^i,i,s,g)
    \approx
    \nabla_{w^i}
    \log p_i(w^i\mid s,g),
\end{align}
where $p_i$ denotes the corresponding distribution at diffusion time
$i$. Therefore, their difference satisfies
\begin{align}
    &v_\theta(w^i,i,s,g)
    -
    v_\theta(w^i,i,s,\emptyset) 
    \approx
    \nabla_{w^i}
    \log
    \frac{
        p_i(w^i\mid s,g)
    }{
        p_i(w^i\mid s)
    }.
    \label{eq:proof_velocity_density_ratio}
\end{align}

Substituting these approximations into the CFG velocity field gives
\begin{align}
    \hat v_\theta(w^i,i)
    &=
    v_\theta(w^i,i,s,\emptyset)
    +
    \omega
    \left(
        v_\theta(w^i,i,s,g)
        -
        v_\theta(w^i,i,s,\emptyset)
    \right)                                               \notag \\
    &\approx
    \nabla_{w^i}
    \left[
        \log p_i(w^i\mid s)
        +
        \omega
        \log
        \frac{
            p_i(w^i\mid s,g)
        }{
            p_i(w^i\mid s)
        }
    \right]                                                \notag \\
    &=
    \nabla_{w^i}
    \log
    \left[
        p_i(w^i\mid s)
        \left(
            \frac{
                p_i(w^i\mid s,g)
            }{
                p_i(w^i\mid s)
            }
        \right)^\omega
    \right].
\end{align}

Under the corresponding endpoint-density interpretation of CFG, the
induced subgoal distribution is
\begin{align}
    p_\omega(w\mid s,g)
    &\propto
    p(w\mid s)
    \left(
        \frac{
            p(w\mid s,g)
        }{
            p(w\mid s)
        }
    \right)^\omega.                                       \label{eq:proof_cfg_density_ratio}
\end{align}
Using
$p(w\mid s)=\pi_\beta(w\mid s)$ and the density-ratio identity
\begin{align}
    \frac{
        p(w\mid s,g)
    }{
        p(w\mid s)
    }
    =
    \frac{
        Q_\beta(s,w,g)
    }{
        V_\beta(s,g)
    },
\end{align}
we obtain
\begin{align}
    p_\omega(w\mid s,g)
    \propto
    \pi_\beta(w\mid s)
    \left(
        \frac{
            Q_\beta(s,w,g)
        }{
            V_\beta(s,g)
        }
    \right)^\omega                                        
    \propto
    \pi_\beta(w\mid s)
    Q_\beta(s,w,g)^\omega,
    \label{eq:proof_cfg_q_weight}
\end{align}
where $V_\beta(s,g)^\omega$ is absorbed into the normalization constant
because it does not depend on $w$. Equivalently,
\begin{align}
    p_\omega(w\mid s,g)
    \propto
    \pi_\beta(w\mid s)
    \exp
    \left(
        \omega
        \widetilde A_\beta(s,w,g)
    \right),
\end{align}
where $\widetilde A_\beta(s,w,g) = \log \frac{Q_\beta(s,w,g)}{V_\beta(s,g)}.$
\end{proof}

\section{Implementation Details}
\label{apdx:impl_details}
\subsection{Goal Mixed Sampling}
We adopt the standard HER \citep{HER,OGBench} trick provided by OGBench. To train both value networks and actors, we employ a mixture of three distinct goal distributions:
\begin{itemize}
    \item ${p_{{\mathrm{cur}}}}\left( {g\left| s \right.} \right)$: A degenerate distribution at the current state $s$ (i.e., $g = s$).
    \item ${p_{{\mathrm{traj}}}}\left( {g\left| s \right.} \right)$: A uniform distribution over future states visited within the same trajectory as the current state $s$.
    \item ${p_{{\mathrm{rand}}}}\left( {g\left| s \right.} \right)$: A uniform distribution over arbitrary states sampled from the entire dataset, independent of the current state $s$.
\end{itemize}
The specific mixing ratios for each task are detailed in Table~\ref{tab:common_hyperparameter}.

\subsection{Target Value Network}
In our proposed method, we retain the conventional use of target value networks to further stabilize the training process. Previous studies \citep{UVFA, RIS} have demonstrated that this technique facilitates faster convergence with negligible computational overhead.

\subsection{Architectures}
In our experiments, the value networks and policies across all methods are implemented as three-layer MLPs that take the concatenation of states and goals as input.

For most hierarchical baselines, an additional subgoal encoder is used to map subgoals into a 10-dimensional latent space, which serves as the interface between the high-level planner and the low-level policy. In contrast, DSP directly operates in the original subgoal space and therefore does not introduce an explicit subgoal encoder. We also report results for HIQL w/o rep, which removes this subgoal encoding for comparison.

Empirical results for HIQL w/o rep in Table~\ref{tab:main_result} indicate that subgoal encoding can noticeably improve performance in high-dimensional state spaces, with the effect being particularly evident in challenging environments such as \texttt{humanoidmaze}. Although computational constraints prevent a systematic ablation of architectural choices, this observation is consistent with prior work on representation learning \citep{QRL,CRL,METRA,ProQ}, which suggests that appropriate dimensional compression can facilitate learning and planning in complex environments.

\section{Hyperparameters}
\label{apdx:hyperparameters}
To support reproducibility, Table \ref{tab:common_hyperparameter} reports the shared hyperparameters, while Table \ref{tab:specific_hyperparameters} lists environment-specific settings for the guidance scale $\omega$, policy temperature $\alpha$, and subgoal interval $k$. We select $\omega$ via a coarse sweep on representative tasks and then fix it across all tasks and seeds within each environment family; $k$ is likewise shared within each family. The number of flow integration steps is fixed globally at $N=20$.

\begin{table}[H]
    \centering

    \setlength{\abovecaptionskip}{3pt}
    \setlength{\belowcaptionskip}{2pt}

    \caption{Common hyperparameters for DSP in our experiments.}
    \label{tab:common_hyperparameter}

    \footnotesize
    \setlength{\tabcolsep}{3.5pt}
    \renewcommand{\arraystretch}{0.82}

    \begin{tabular}{@{}ll@{}}
        \toprule
        \textbf{Hyperparameter} & \textbf{Value} \\
        \midrule
        Learning rate & $0.0003$ \\
        Optimizer & Adam \\
        Batch size & $1024$ \\
        Total gradient steps & $1000000$ \\
        MLP dimensions & $[512, 512, 512]$ \\
        Activation function & GELU \\
        Target network smoothing coefficient $\kappa$ & $0.005$ \\
        Discount factor $\gamma$ & $0.995$ (humanoidmaze), $0.99$ (others) \\
        \midrule
        Expectile $\varepsilon$ & $0.7$ \\
        Diffusion time sampling distribution & $\mathcal{U}(0,1)$ \\
        Diffusion step $N$ & $20$ \\
        Dropout probability $p_{\mathrm{unc}}$ & $0.1$ \\
        \midrule
        Value goal mix ratio
        $({p_{\mathrm{cur}}},{p_{\mathrm{traj}}},{p_{\mathrm{rand}}})$
        & $(0.2, 0.5, 0.3)$ \\
        \midrule
        Policy goal mix ratio
        $({p_{\mathrm{cur}}},{p_{\mathrm{traj}}},{p_{\mathrm{rand}}})$
        & $(0, 1.0, 0)$ \\
        \bottomrule
    \end{tabular}
\end{table}

\begin{table}[H]
    \centering
    \setlength{\abovecaptionskip}{3pt}
    \setlength{\belowcaptionskip}{2pt}
    \caption{
        Task specific hyperparameters for DSP in our experiments.
    }
    \label{tab:specific_hyperparameters}
    \footnotesize
    \setlength{\tabcolsep}{3.2pt}
    \renewcommand{\arraystretch}{0.82}

    \begin{tabular*}{0.72\linewidth}{
    @{\extracolsep{\fill}}llccc@{}
}
        \toprule
        \textbf{Environment} & \textbf{Datasets} & $\omega$ & $\alpha$ & $k$ \\
        \midrule
        \multirow{8}{*}{\texttt{pointmaze}} 
        & \texttt{pointmaze-medium-navigate-v0} & $3.0$ & $3.0$ & $25$ \\
        & \texttt{pointmaze-large-navigate-v0} & $3.0$ & $3.0$ & $25$ \\
        & \texttt{pointmaze-giant-navigate-v0} & $3.0$ & $3.0$ & $25$ \\
        & \texttt{pointmaze-teleport-navigate-v0} & $3.0$ & $3.0$ & $25$ \\
        & \texttt{pointmaze-medium-stitch-v0} & $3.0$ & $3.0$ & $25$ \\
        & \texttt{pointmaze-large-stitch-v0} & $3.0$ & $3.0$ & $25$ \\
        & \texttt{pointmaze-giant-stitch-v0} & $3.0$ & $3.0$ & $25$ \\
        & \texttt{pointmaze-teleport-stitch-v0} & $3.0$ & $3.0$ & $25$ \\
        \midrule
        \multirow{8}{*}{\texttt{antmaze}} 
        & \texttt{antmaze-medium-navigate-v0} & $3.0$ & $3.0$ & $25$ \\
        & \texttt{antmaze-large-navigate-v0} & $3.0$ & $3.0$ & $25$ \\
        & \texttt{antmaze-giant-navigate-v0} & $3.0$ & $3.0$ & $25$ \\
        & \texttt{antmaze-teleport-navigate-v0} & $3.0$ & $3.0$ & $25$ \\
        & \texttt{antmaze-medium-stitch-v0} & $3.0$ & $3.0$ & $25$ \\
        & \texttt{antmaze-large-stitch-v0} & $3.0$ & $3.0$ & $25$ \\
        & \texttt{antmaze-giant-stitch-v0} & $3.0$ & $3.0$ & $25$ \\
        & \texttt{antmaze-teleport-stitch-v0} & $3.0$ & $3.0$ & $25$ \\
        \midrule
        \multirow{6}{*}{\texttt{humanoidmaze}} 
        & \texttt{humanoidmaze-medium-navigate-v0} & $1.5$ & $3.0$ & $100$ \\
        & \texttt{humanoidmaze-large-navigate-v0} & $1.5$ & $3.0$ & $100$ \\
        & \texttt{humanoidmaze-giant-navigate-v0} & $1.5$ & $3.0$ & $100$ \\
        & \texttt{humanoidmaze-medium-stitch-v0} & $1.5$ & $3.0$ & $100$ \\
        & \texttt{humanoidmaze-large-stitch-v0} & $1.5$ & $3.0$ & $100$ \\
        & \texttt{humanoidmaze-giant-stitch-v0} & $1.5$ & $3.0$ & $100$ \\
        \midrule
        \multirow{4}{*}{\texttt{antsoccer}} 
        & \texttt{antsoccer-arena-navigate-v0} & $3.0$ & $3.0$ & $25$ \\
        & \texttt{antsoccer-medium-navigate-v0} & $3.0$ & $3.0$ & $25$ \\
        & \texttt{antsoccer-arena-stitch-v0} & $3.0$ & $3.0$ & $25$ \\
        & \texttt{antsoccer-medium-stitch-v0} & $3.0$ & $3.0$ & $25$ \\
        \midrule
        \multirow{2}{*}{\texttt{cube}} 
        & \texttt{cube-single-play-v0} & $5.0$ & $3.0$ & $20$ \\
        & \texttt{cube-double-play-v0} & $5.0$ & $3.0$ & $20$ \\
        \midrule
        \multirow{1}{*}{\texttt{scene}} 
        & \texttt{scene-play-v0} & $5.0$ & $3.0$ & $20$ \\
        \midrule
        \multirow{2}{*}{\texttt{puzzle}}
        & \texttt{puzzle-3x3-play-v0} & $5.0$ & $3.0$ & $20$ \\
        & \texttt{puzzle-4x4-play-v0} & $5.0$ & $3.0$ & $20$ \\
        \midrule
        \multirow{6}{*}{\texttt{D4RL-antmaze}}
        & \texttt{antmaze-umaze-v2} & $3.0$ & $3.0$ & $25$ \\
        & \texttt{antmaze-umaze-diverse-v2} & $3.0$ & $3.0$ & $25$ \\
        & \texttt{antmaze-medium-play-v2} & $3.0$ & $3.0$ & $25$ \\
        & \texttt{antmaze-medium-diverse-v2} & $3.0$ & $3.0$ & $25$ \\
        & \texttt{antmaze-large-play-v2} & $3.0$ & $3.0$ & $25$ \\
        & \texttt{antmaze-large-diverse-v2} & $3.0$ & $3.0$ & $25$ \\
        \bottomrule
    \end{tabular*}
\end{table}

\section{Computational Resources}
\label{apdx:computational_resources}
All experiments were conducted on two compute servers, each equipped with four NVIDIA GeForce RTX 4090 GPUs (24 GB of VRAM per GPU). Unless otherwise specified, all reported training times correspond to runs executed on a single GPU. For locomotion tasks, training DSP on the largest environment (\texttt{humanoidmaze-giant-navigate-v0}) requires approximately one hour on a single GPU. For manipulation tasks, the corresponding training time is around 40 minutes.

\section{Environments and Datasets}
\label{apdx:env_and_dat}
This section provides a detailed description of each task. Visualizations of the environments are presented in Figure~\ref{fig:environments}. For further technical specifications, we refer readers to the OGBench white paper.

\textbf{Maze.} \texttt{Maze} is a challenging sparse-reward, long-horizon locomotion task requiring the agent to navigate from an arbitrary starting position to a specified goal. The environments vary by agent complexity: the 2-DoF \texttt{point}, the 8-DoF \texttt{ant}, and the 21-DoF \texttt{humanoid}. Each maze is categorized by size: \texttt{medium}, \texttt{large}, and \texttt{giant}. Additionally, the \texttt{teleport} configuration mirrors the large layout but incorporates random teleporters—some leading to dead ends—to test robustness against environmental stochasticity. The datasets are distinct based on collection methods; specifically, the navigate datasets consist of full trajectories collected by a noisy expert policy attempting to reach randomly sampled goals. In contrast, the \texttt{stitch} datasets are composed of short goal-reaching trajectories with a maximum length of four cell units, requiring the agent to stitch together multiple trajectory segments to solve long-horizon tasks.

\textbf{AntSoccer.} \texttt{Antsoccer} is introduced to further evaluate multi-goal generalization. In this environment, a quadrupedal Ant robot must dribble a ball to a target location. This task exceeds the difficulty of \texttt{antmaze}, as it requires simultaneous navigation and precise ball manipulation. The environment includes two layouts: \texttt{arena}, an open space without walls, and \texttt{medium}, which corresponds to the layout in \texttt{antmaze}.

\textbf{Cube.} \texttt{Cube} involves complex robotic manipulation focused on pick-and-place operations, where a robotic arm must arrange cubes into a specific configuration. The training dataset consists of \texttt{play}-style data generated by a scripted policy that randomly moves and stacks cubes. During evaluation, the agent is tasked with moving, stacking, swapping, or arranging cubes to match a goal configuration. Success requires learning generalizable multi-object behaviors and long-horizon reasoning from unstructured, stochastic trajectories.

\begin{figure*}
    \centering
    \includegraphics[width=0.18\linewidth]{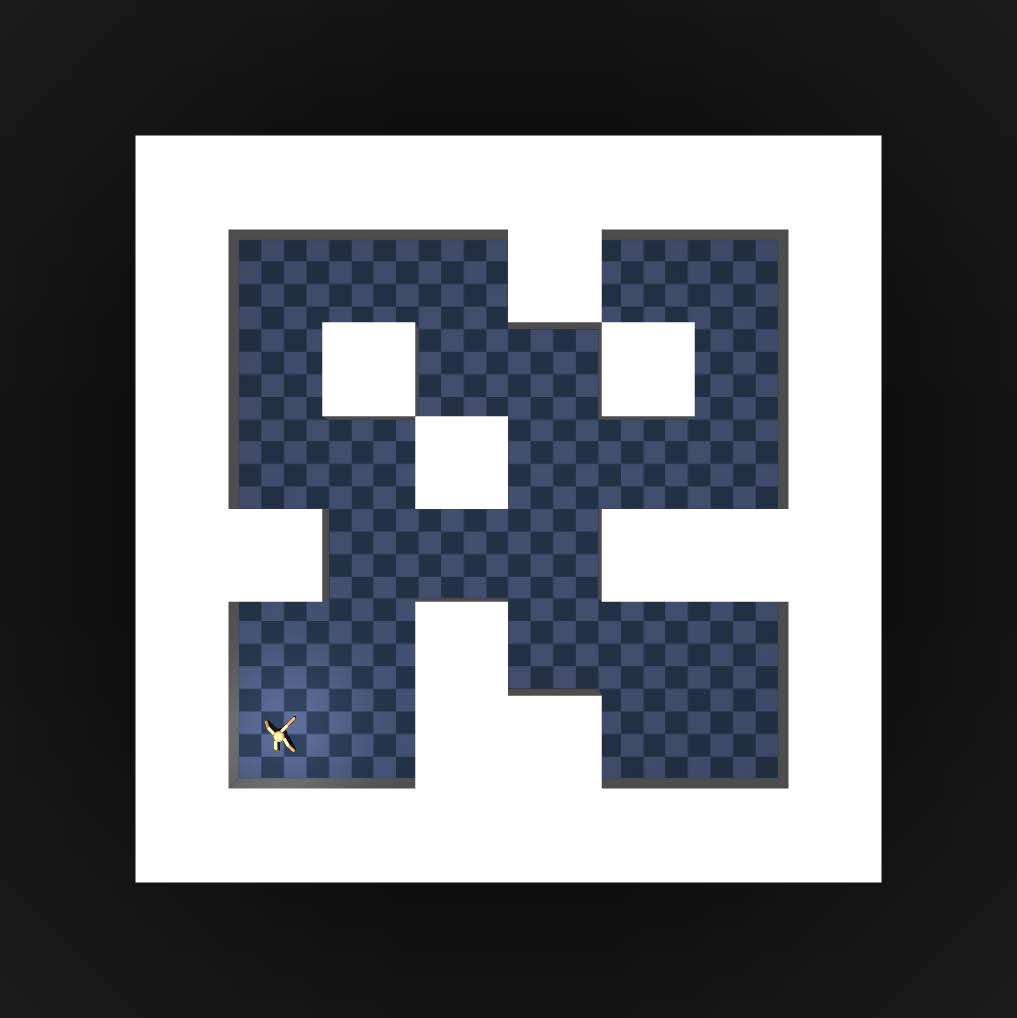}
    \includegraphics[width=0.18\linewidth]{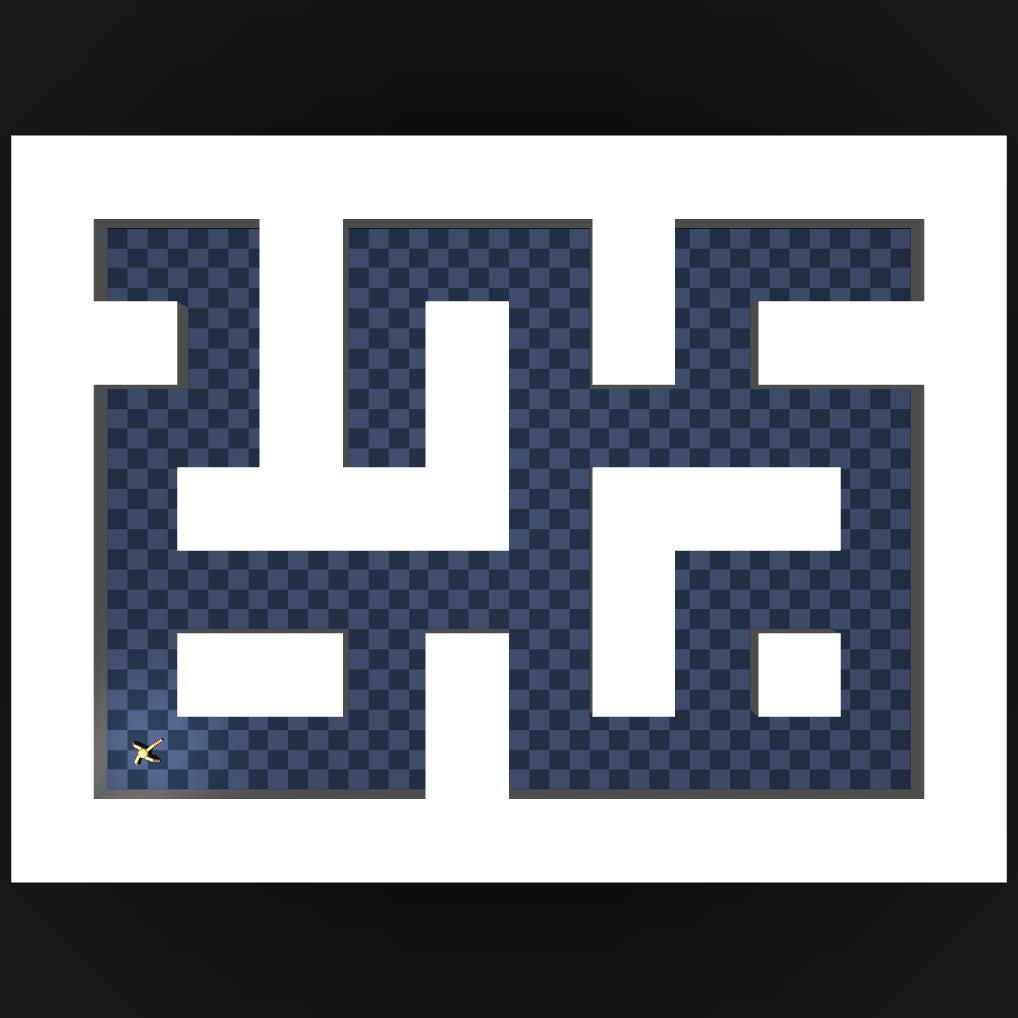}
    \includegraphics[width=0.18\linewidth]{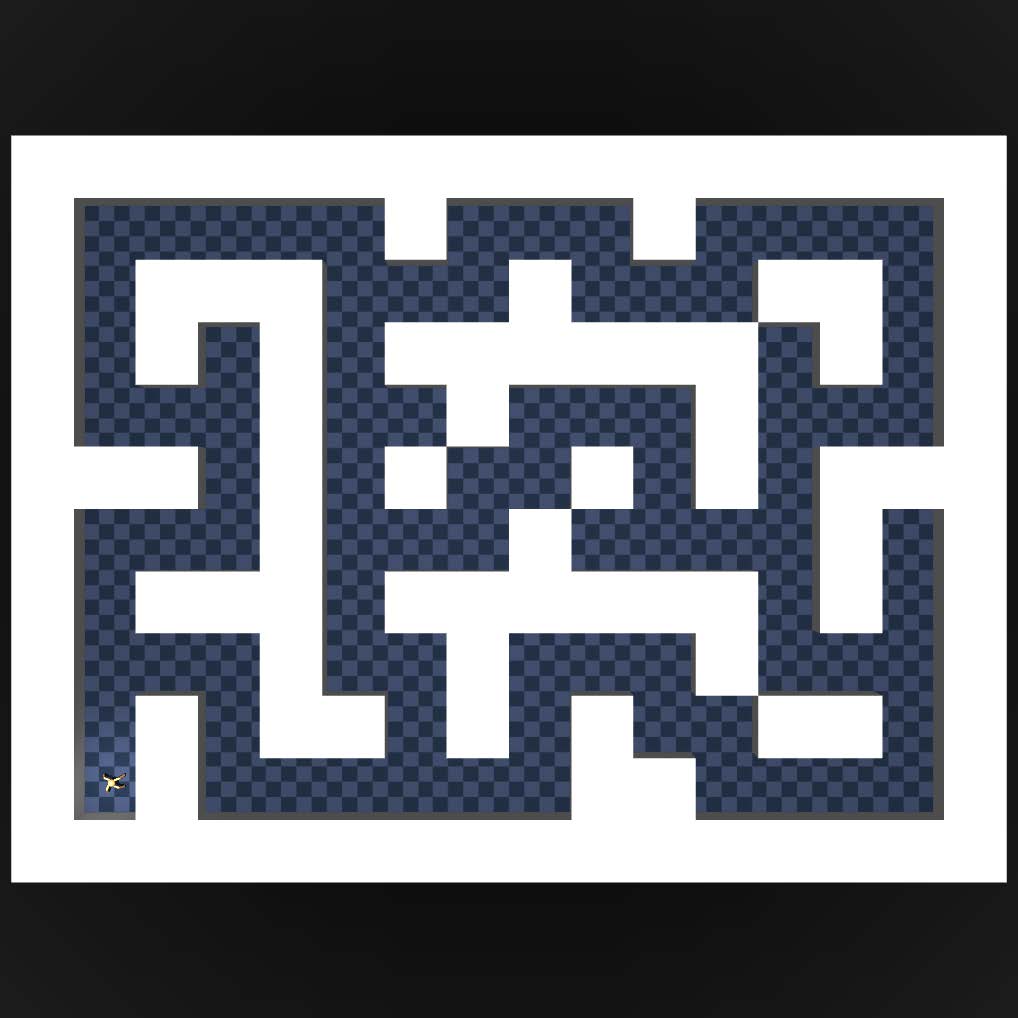}
    \includegraphics[width=0.18\linewidth]{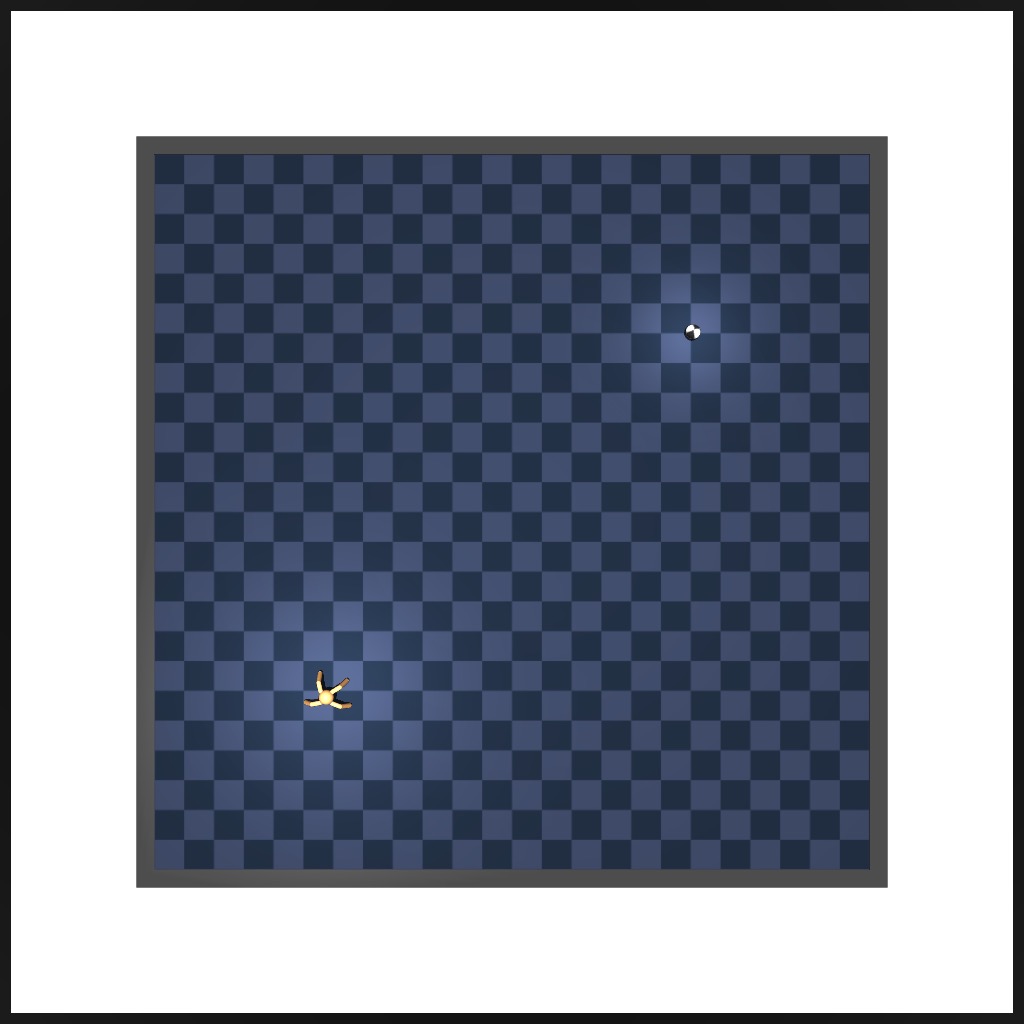}
    \includegraphics[width=0.18\linewidth]{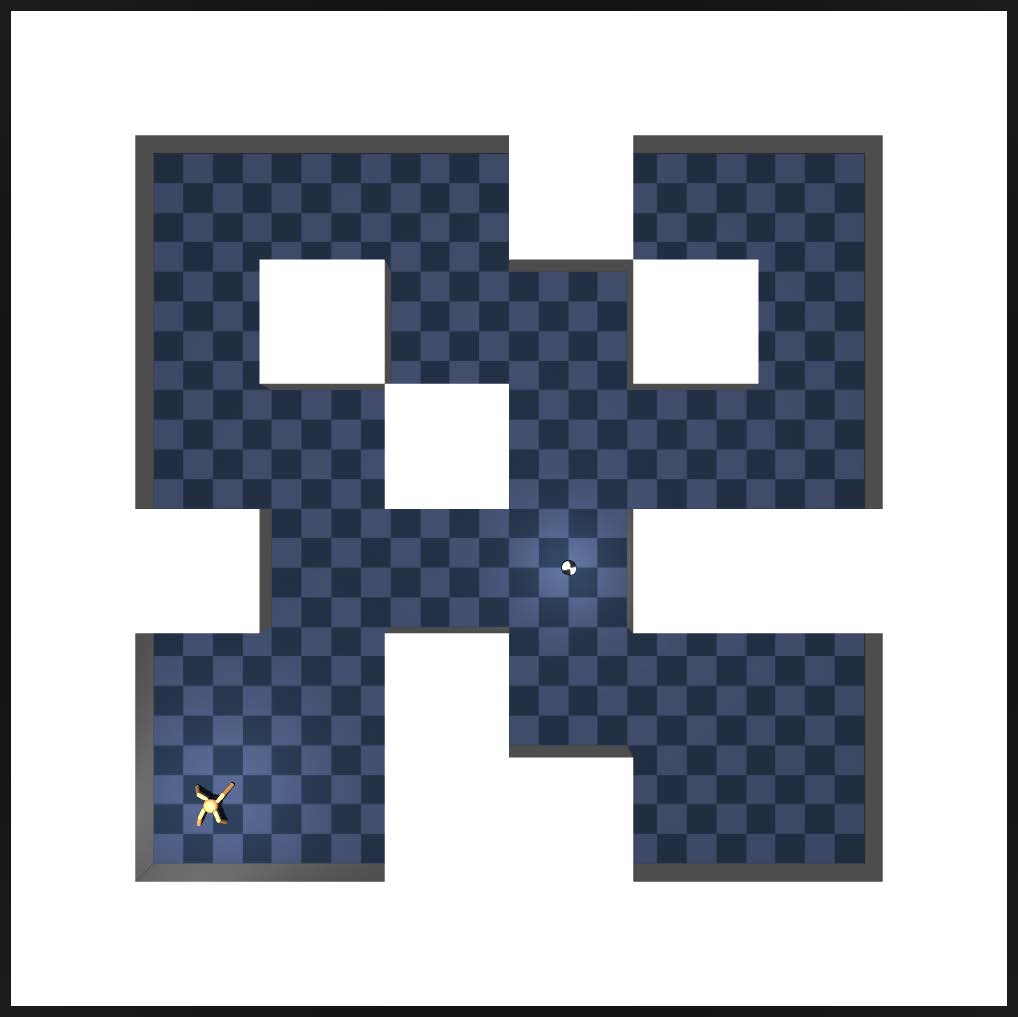}
    \includegraphics[width=0.18\linewidth]{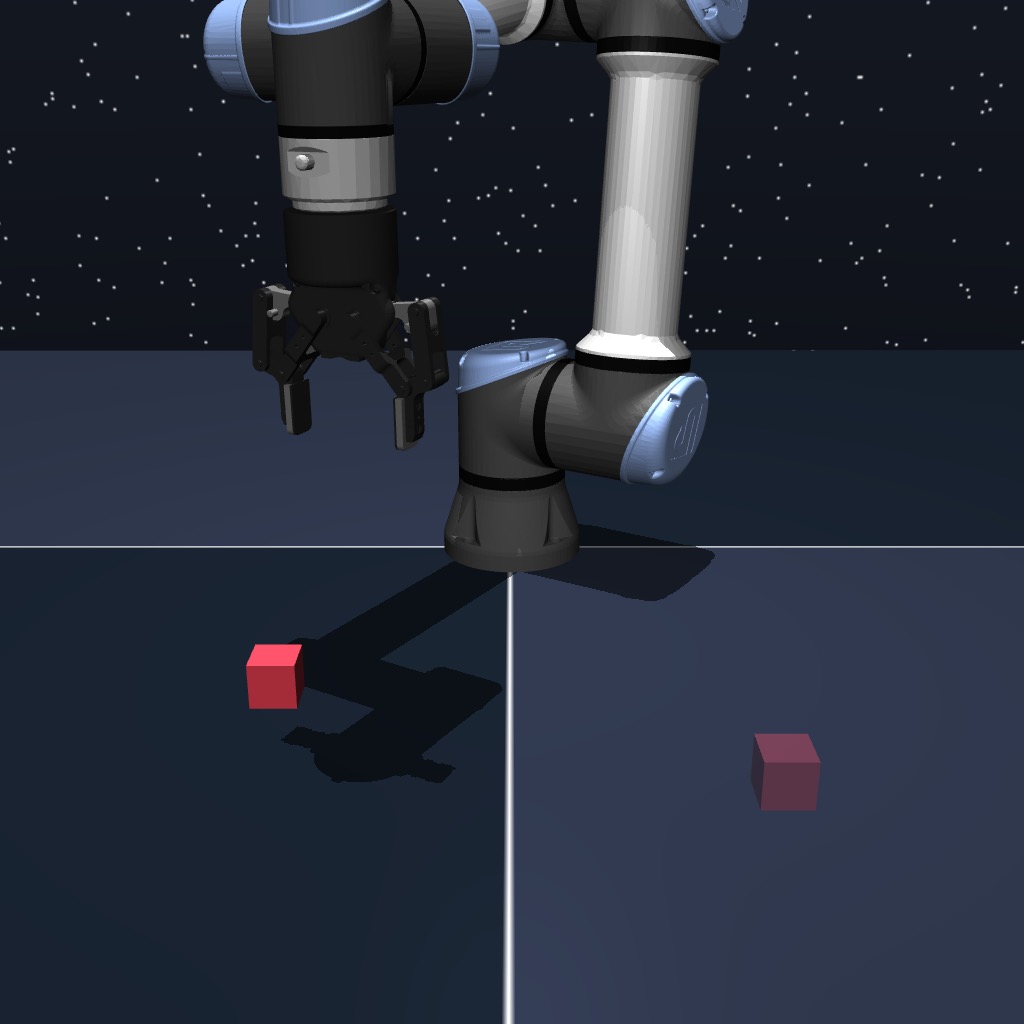}
    \includegraphics[width=0.18\linewidth]{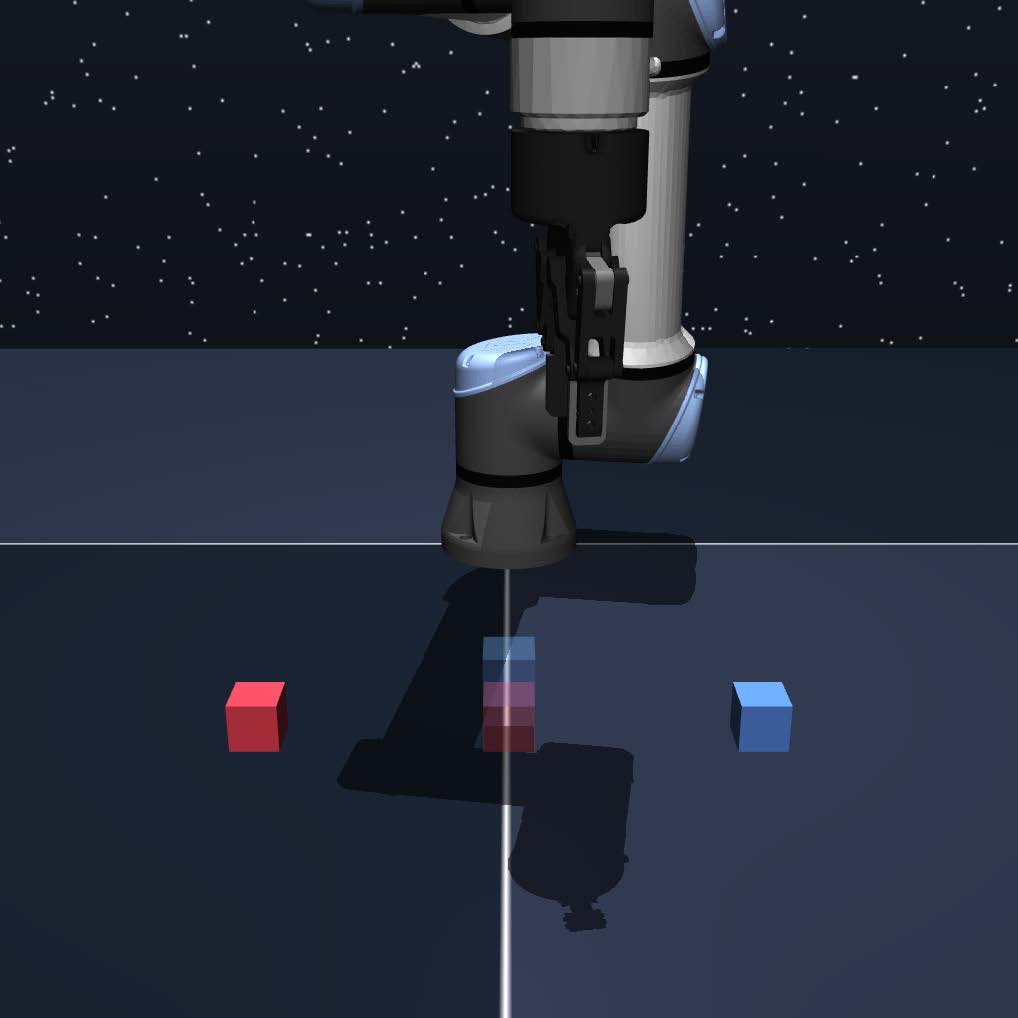}
    \includegraphics[width=0.18\linewidth]{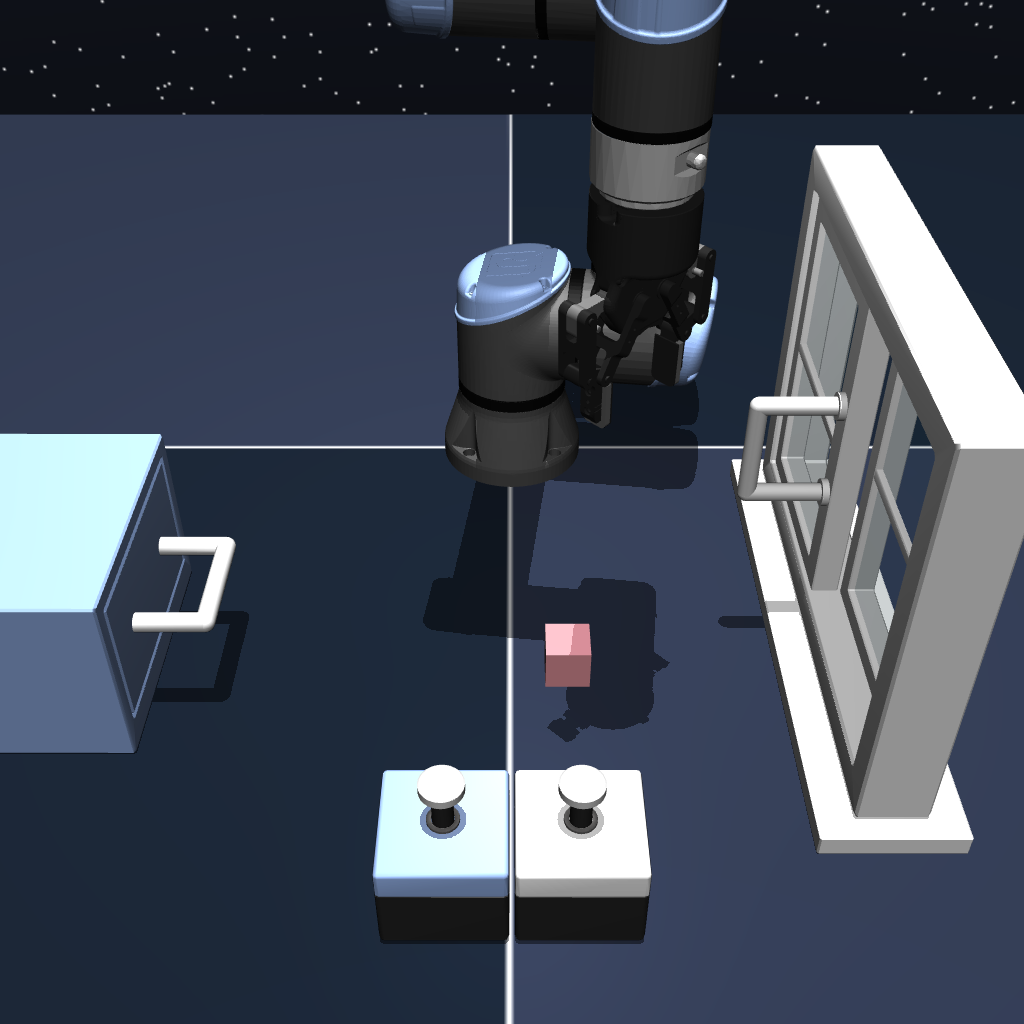}
    \includegraphics[width=0.18\linewidth]{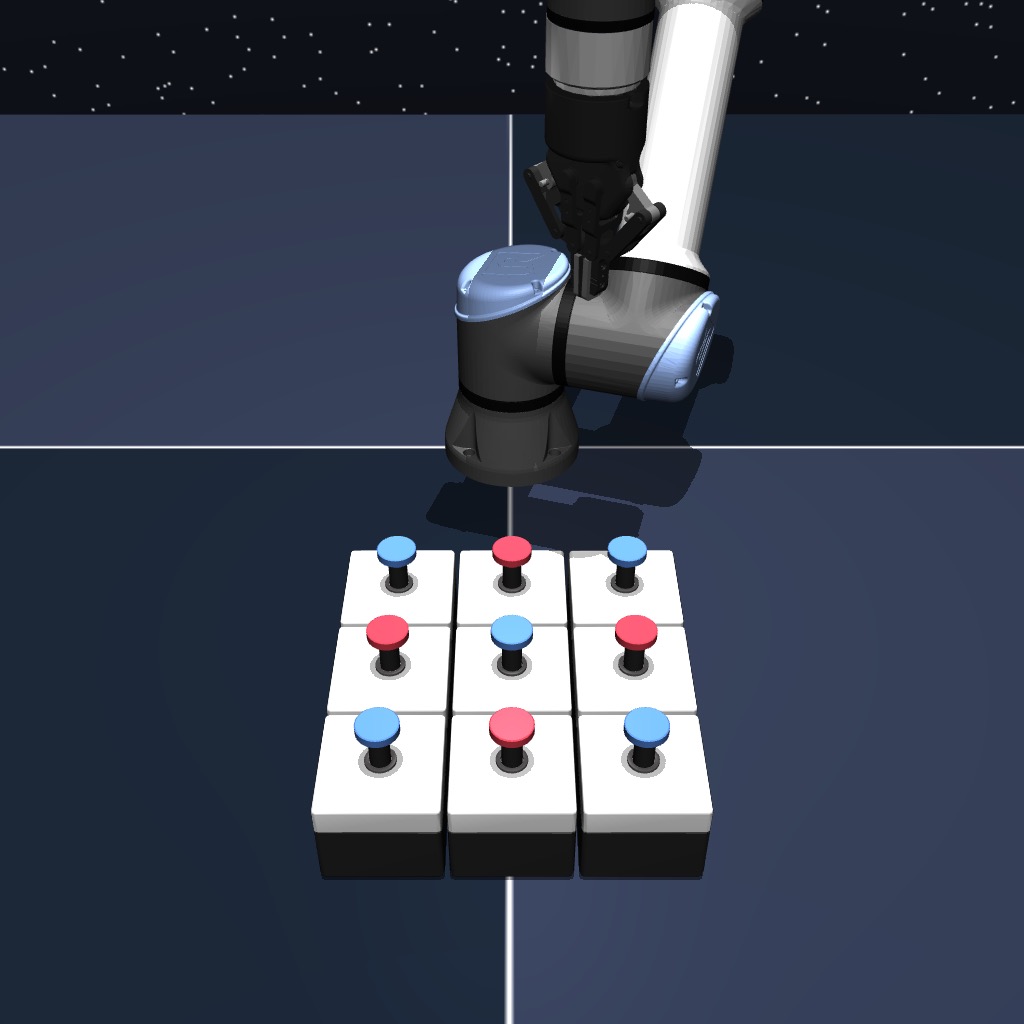}
    \includegraphics[width=0.18\linewidth]{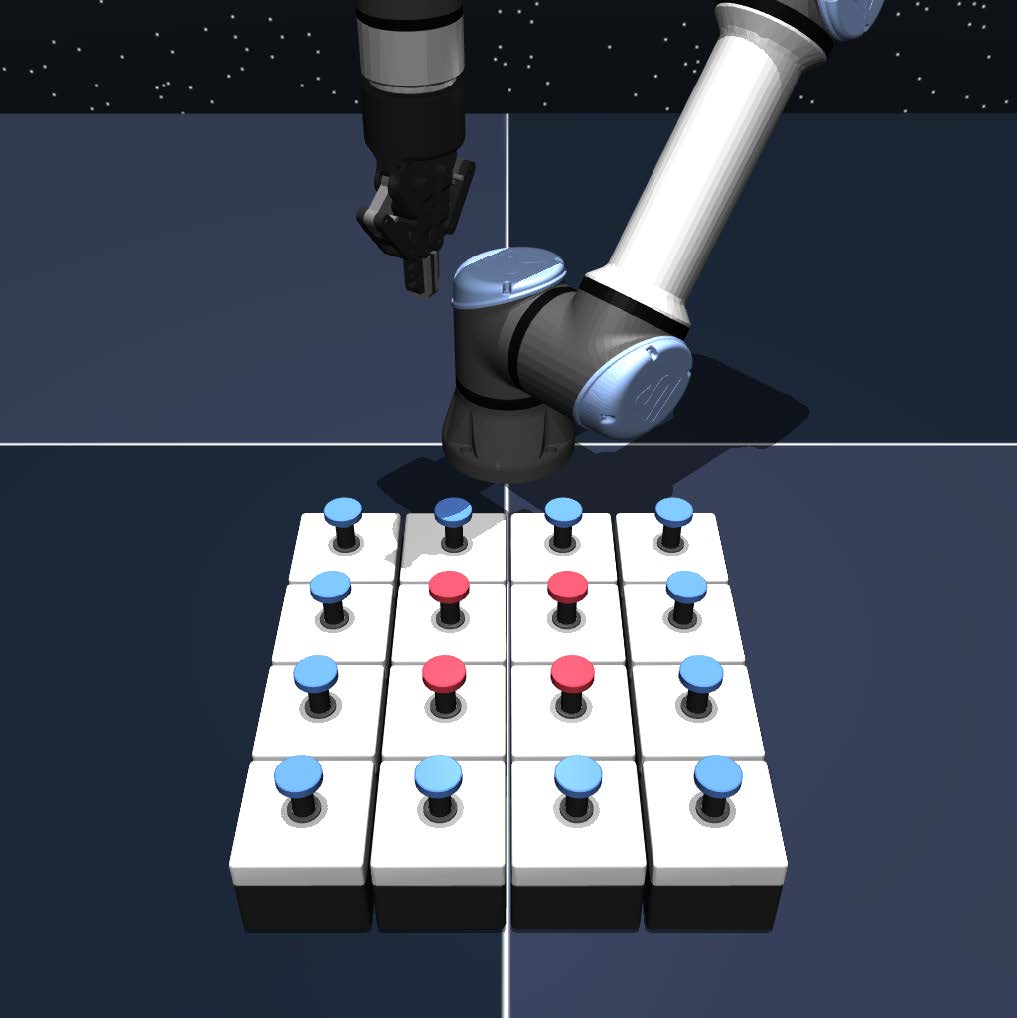}
    \caption{Illustration of the OGBench environments, covering diverse locomotion and manipulation tasks used for evaluation.}
	\label{fig:environments}
\end{figure*}

\textbf{Scene.} The \texttt{scene} task evaluates sequential reasoning and long-horizon planning involving everyday objects such as cubes, windows, drawers, and button locks. Pressing a button toggles the locking state of a corresponding object (e.g., a drawer or window). Similar to \texttt{cube}, the dataset is collected via a \texttt{play}-style scripted policy. At test time, the agent must complete a sequence of sub-tasks to achieve the final goal, necessitating the ability to compose skills temporally.

\textbf{Puzzle.} \texttt{Puzzle} requires a robotic arm to solve a "Lights Out" style problem. The workspace contains a 2D array of buttons (e.g., a 3x3 grid); pressing a button toggles its color and that of its orthogonal neighbors. The objective is to achieve a specific color configuration through a sequence of presses. This task demands both precise low-level continuous control and high-level combinatorial generalization, posing a significant challenge to the agent's long-term reasoning capabilities.

\section{Offline GCRL Baseline Algorithms}
\label{apdx:baselines}
A brief overview of the baseline algorithms compared in this study is provided below:

\textbf{Goal-conditioned behavioral cloning (GCBC).} \citep{GCSL} GCBC is a standard imitation learning approach that clones behaviors by utilizing hindsight goal relabeling with future states observed within the same trajectory.

\textbf{Classifier-free guidance reinforcement learning (CFGRL).} \citep{CFGRL} CFGRL employs a diffusion model as the policy network and executes action inference through classifier-free guidance sampling, enhancing the generation of high-fidelity actions.

\textbf{Goal-conditioned implicit V learning (GCIVL).} \citep{OGBench} GCIVL introduces a V-only objective to regress the optimal state-value function. By omitting Q-value learning, it does not explicitly marginalize over non-causal factors, which can introduce optimistic bias in stochastic environments.

\textbf{Option-aware temporally abstracted V learning (OTA).} \citep{OTA} OTA incorporates an auxiliary high-level value function to facilitate high-level policy extraction. Its core contribution involves leveraging option-aware temporal abstraction to derive this value function, thereby providing more precise guidance for the high-level policy.

\textbf{Physics-informed value learner (Pi-HIQL).} \citep{Pi-HIQL} Drawing inspiration from the Eikonal equation, Pi-HIQL incorporates an Eikonal regularizer into implicit V-learning to encourage geometric structure in the learned value function. This method aims to induce geometric inductive biases within the value function by leveraging ground-truth physical constraints.

\textbf{Hierarchical implicit Q learning (HIQL).} \citep{HIQL} The theoretical framework of HIQL is detailed in Section~\ref{sec:3}. Regarding implementation, standard HIQL employs a latent subgoal representation for policy extraction. Specifically, the high-level and low-level policies are parameterized as ${\pi ^h}:{\cal S} \times {\cal S} \to \Delta \left( {\cal Z} \right)$ and ${\pi ^\ell }:{\cal S} \times {\cal Z} \to \Delta \left( {\cal A} \right)$, respectively, where $\mathcal{Z}$ denotes the latent subgoal space.

\textbf{HIQL without subgoal representation (HIQL$^\mathrm{w/o}$).} This variant removes the latent subgoal representation mechanism used in the standard HIQL implementation. Instead, it adheres strictly to the objectives defined in Equation~\ref{equ:hiql_high} and Equation~\ref{equ:hiql_low} of Section~\ref{sec:3} for policy learning, operating without the auxiliary latent mapping.

\section{Subgoal Support and Executor Reachability}
\label{app:guidance_reachability}
Neither DSP nor HIQL provides a hard feasibility guarantee for generated
subgoals. To quantify how classifier-free guidance affects data support and
low-level execution, we hold the trained checkpoints and all other settings
fixed and vary the guidance scale $\omega$.

For $M$ generated subgoals $\{w_i\}_{i=1}^{M}$, we define the standardized
training-support distance as
\begin{equation}
    D =
    \frac{1}{M}
    \sum_{i=1}^{M}
    \min_{x \in \mathcal{S}_{\mathcal D}}
    \left\|
        \frac{w_i-x}{\sigma_{\mathcal D}}
    \right\|_2,
\end{equation}
where $\mathcal{S}_{\mathcal D}$ denotes the states contained in the offline
dataset and $\sigma_{\mathcal D}$ is their coordinate-wise standard deviation.
A smaller $D$ indicates that generated subgoals remain closer to the training
support.

We further measure the state-validity rate
\begin{equation}
    V =
    \frac{1}{M}
    \sum_{i=1}^{M}
    \mathbf{1}\left[\operatorname{Valid}(w_i)\right],
\end{equation}
and the valid-and-reached rate
\begin{equation}
    R =
    \frac{1}{M}
    \sum_{i=1}^{M}
    \mathbf{1}\left[\operatorname{Valid}(w_i)\right]
    \mathbf{1}\left[d(s_{i,k},w_i)\leq\delta\right],
\end{equation}
where $s_{i,k}$ is the state reached after executing the low-level policy
toward $w_i$ for at most $k$ steps. Both $V$ and $R$ are reported as percentages.

As shown in \autoref{tab:guidance_reachability}, the family-default guidance
scales improve task success relative to $\omega=1$, while moderately reducing
the valid-and-reached rate. With excessive guidance ($\omega=10$), support
distance increases and validity, reachability, and success all deteriorate.
Thus, moderate guidance balances goal direction and executability, whereas
excessive guidance constitutes a practical failure mode.

\section{Hyperparameter Sensitivity}
\label{apdx:hyper_sensitivity}

\subsection{Guidance Scale}
\begin{figure}[h]
    \centering
    \subfigure[Guidance Scale $\omega$]{
        \includegraphics[width=0.31\linewidth]{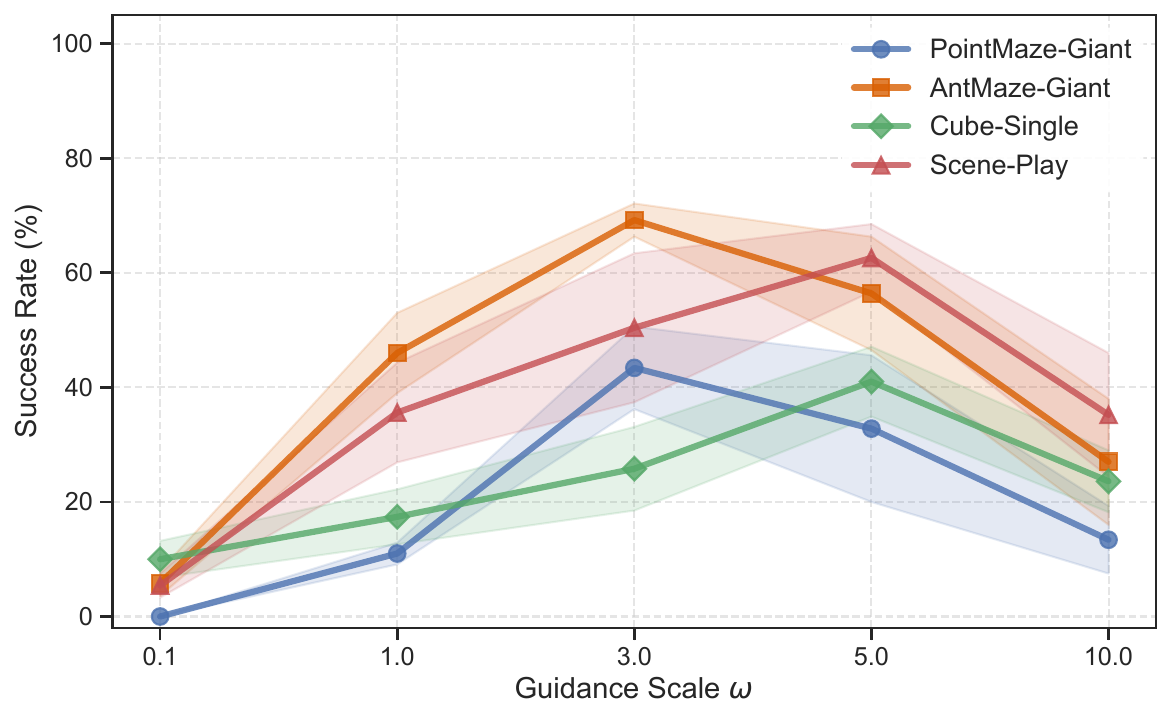}
        \label{fig:cfg_sensitivity}
    }
    \subfigure[Diffusion Step $N$]{
        \includegraphics[width=0.31\linewidth]{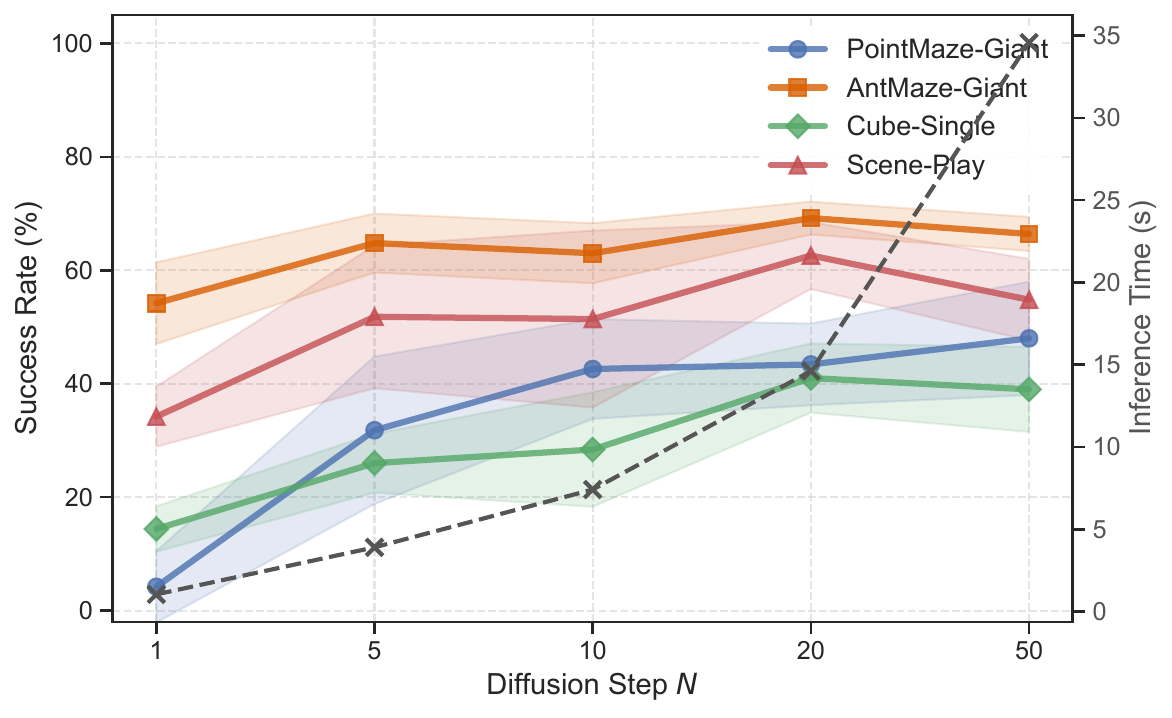}
        \label{fig:N_sensitivity}
    }
    \subfigure[Subgoal Steps $k$]{
        \includegraphics[width=0.31\linewidth]{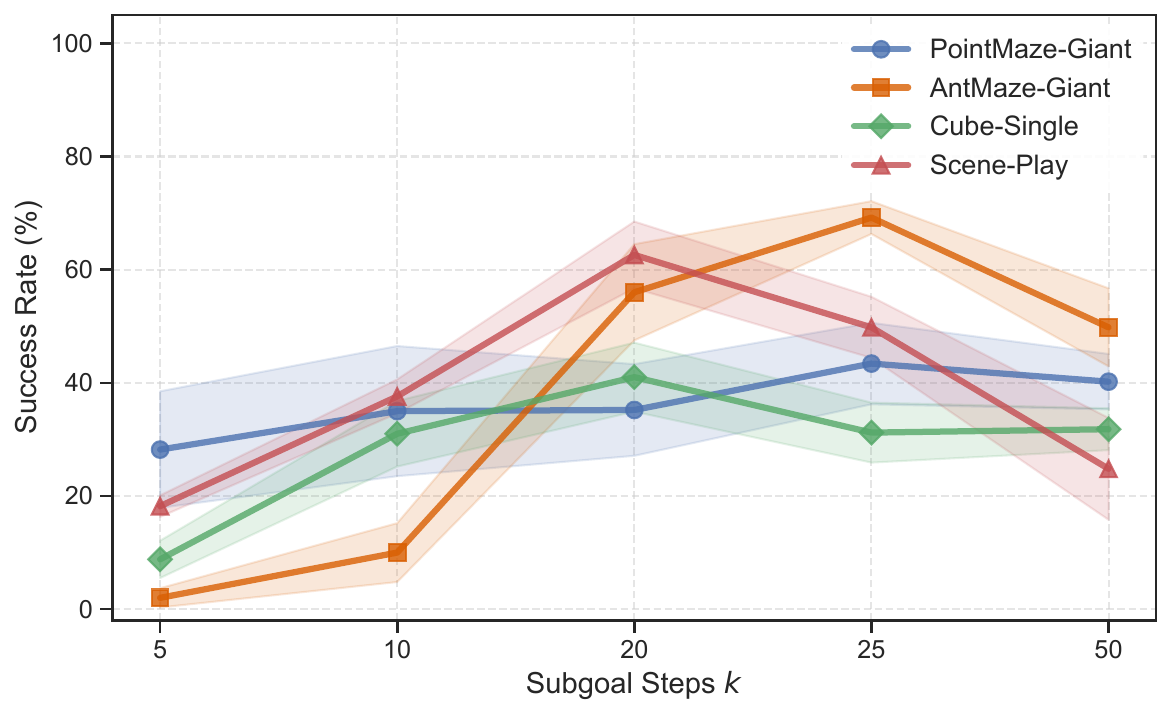}
        \label{fig:k_sensitivity}
    }
    \vspace{-0.3cm}
    \caption{Hyperparameter sensitivity analysis of DSP. (a): sensitivity to the guidance scale $\omega$ under classifier-free guidance, showing a consistent inverted-U trend across representative tasks. (b): sensitivity to the number of diffusion steps $N$, illustrating the trade-off between performance gains and inference cost. (c): sensitivity to the subgoal steps $k$, showing a broad performance plateau with degradation for overly large $k$. Shaded regions indicate one standard deviation over 5 random seeds.}
    \label{fig:hyperparameter_sensitivity}
\end{figure}

Figure~\ref{fig:cfg_sensitivity} presents a sensitivity analysis of the guidance scale $\omega$ across four representative tasks. The results exhibit a consistent inverted-U trend, from which we draw three key observations. First, increasing $\omega$ from the baseline value of $1.0$ (which approximates standard conditional sampling) to the range of $3.0$ to $5.0$ yields substantial performance gains across all tasks. For instance, in \texttt{pointmaze-giant-navigate-v0}, the success rate increases from $11.0\%$ to $43.4\%$. This is consistent with the implicit advantage-weighted interpretation of CFG in DSP, where stronger guidance increases preference for goal-relevant subgoals. Second, we observe clear task-dependent preferences for the guidance strength. Locomotion tasks tend to peak around $\omega=3.0$, whereas manipulation tasks benefit from stronger guidance and achieve their best performance at $\omega=5.0$. We hypothesize that manipulation tasks involve complex contact dynamics and narrow optimal regions, requiring stronger directional signals to steer generation toward effective subgoals. In contrast, navigation tasks span larger spatial scales, where overly strong guidance may overly constrain subgoal diversity and hinder exploration of feasible paths. Finally, performance often degrades when $\omega$ becomes excessively large. This behavior aligns with prior findings on CFG~\citep{CFG}: excessive guidance pushes samples toward low-density regions of the data distribution, increasing the likelihood of generating unreachable subgoals and ultimately impairing low-level execution.

\subsection{Diffusion Step}
\label{apdx:hyper_sensitivity_N}
We further study the effect of the number of diffusion steps $N$ on both the performance and inference efficiency of DSP. Figure~\ref{fig:N_sensitivity} reports the average success rates and standard deviations for $N \in \{1, 5, 10, 20, 50\}$. Across all tasks, DSP exhibits a clear pattern in which performance improves substantially as $N$ increases, followed by saturation and occasional mild degradation. Increasing the number of diffusion steps from $N=1$ to $N=10$ or $N=20$ consistently leads to performance gains across all tasks, indicating that additional diffusion steps allow the conditional generative model to better capture the structural information required for long-horizon subgoal generation. This effect is particularly pronounced on more challenging tasks such as \texttt{pointmaze-giant} and \texttt{scene-play}. In contrast, further increasing the number of diffusion steps to $N=50$ does not yield consistent improvements, and even results in slight performance drops on some tasks (e.g., \texttt{antmaze-giant} and \texttt{scene-play}), suggesting diminishing returns and potential instability introduced by excessive sampling.

At the same time, inference time increases approximately linearly with the number of diffusion steps, rising from about 1 second at $N=1$ to over 35 seconds at $N=50$, leading to a substantial computational overhead. Here, the reported inference time is measured during evaluation over 10000 repeated inferences, where each inference generates one subgoal and one action. Considering the trade-off between performance and efficiency, $N=20$ achieves near-optimal or optimal performance on most tasks while maintaining a reasonable inference cost. We therefore adopt $N=20$ as the default number of diffusion steps in all experiments unless otherwise specified.

\subsection{Subgoal Steps}
We study the sensitivity of DSP to the subgoal steps $k$, which controls the temporal distance of high-level subgoals. As shown in Figure~\ref{fig:k_sensitivity}, performance consistently improves as $k$ increases from very small values, indicating that short-horizon subgoals are insufficient for effective long-range planning. As $k$ increases further, performance reaches a broad plateau (e.g., $k=20$--$25$), suggesting that DSP is not highly sensitive to the exact choice of $k$ once it exceeds a reasonable scale.

However, when $k$ becomes overly large, performance may degrade. This is because distant subgoals are harder for the low-level policy to reliably reach, and such long-range transitions are less supported in the offline dataset, leading to increased uncertainty in subgoal generation.

In practice, we adopt a simple and consistent heuristic across domains: larger $k$ for long-horizon navigation tasks and smaller $k$ for manipulation tasks requiring fine-grained control. This choice is consistent with common practices in hierarchical reinforcement learning and does not require per-task tuning.

\section{High-Level Subgoal Generation: Ablations and Comparisons}
\label{app:high_level_generation}

\subsection{Controlled Comparison of High-Level Mechanisms}
\label{app:controlled_high_level_ablation}
In DSP, training subgoals are constructed from $k$-step future states in the offline dataset and are not assumed to be optimal. CFG therefore does not provide optimal labels during training; instead, it introduces an inference-time goal-directed bias in the learned generative subgoal policy. To isolate this effect, we conduct a controlled ablation on the high-level planner while keeping the low-level policy training objective fixed as HIQL-style AWR.
\begin{table}[h]
    \centering
    \caption{
        Ablation on high-level subgoal generation.
        The table reports average binary success rate (\%) over 5 seeds.
        MLE and diffusion variants are trained on the same $k$-step
        future-state subgoal targets.
    }
    \label{tab:high_level_ablation}

    \scriptsize
    \setlength{\tabcolsep}{3pt}
    \renewcommand{\arraystretch}{0.90}

    \begin{tabular}{@{}lcccc@{}}
        \toprule
        \textbf{Datasets}
        & \textbf{MLE}
        & \textbf{Diffusion $(\omega=1)$}
        & \textbf{HIQL}
        & \textbf{DSP} \\
        \midrule

        \texttt{antmaze-medium-navigate-v0}
        & $75.8\pm2.5$
        & $77.0\pm7.5$
        & $93.2\pm1.3$
        & $\mathbf{98.4}\pm0.5$ \\

        \texttt{antmaze-large-navigate-v0}
        & $58.0\pm5.1$
        & $58.2\pm7.6$
        & $87.8\pm1.5$
        & $\mathbf{92.0}\pm4.1$ \\

        \texttt{antmaze-giant-navigate-v0}
        & $14.8\pm4.4$
        & $43.8\pm4.0$
        & $58.2\pm4.7$
        & $\mathbf{69.2}\pm2.9$ \\

        \midrule

        \texttt{cube-single-play-v0}
        & $15.0\pm3.7$
        & $19.2\pm3.8$
        & $11.8\pm2.3$
        & $\mathbf{41.0}\pm6.1$ \\

        \texttt{cube-double-play-v0}
        & $1.2\pm1.1$
        & $4.8\pm0.8$
        & $3.6\pm2.2$
        & $\mathbf{30.8}\pm8.4$ \\

        \texttt{scene-play-v0}
        & $6.8\pm1.8$
        & $25.6\pm4.2$
        & $37.2\pm4.2$
        & $\mathbf{62.6}\pm5.9$ \\

        \bottomrule
    \end{tabular}
\end{table}

We compare four high-level planning mechanisms: (i) an MLE-based subgoal generator trained to imitate $k$-step future states, (ii) diffusion-based conditional sampling with $\omega=1$, (iii) HIQL-style value-based high-level planning, and (iv) DSP with CFG-guided subgoal generation. All methods use the same type of low-level AWR executor. The results are reported in Table~\ref{tab:high_level_ablation}.

The MLE-based subgoal generator performs substantially worse than HIQL and DSP, especially on long-horizon navigation tasks, indicating that directly imitating $k$-step future states is insufficient. Diffusion-based conditional sampling improves over MLE in several settings, suggesting that modeling a richer subgoal distribution is beneficial. However, the gap between diffusion with $\omega=1$ and DSP shows that standard conditional generation does not explain the full performance gains. CFG-guided sampling further biases generation toward goal-relevant subgoals and achieves the strongest performance across all evaluated tasks.

These results suggest that DSP is not merely a $k$-step subgoal imitation method. Rather, its performance comes from combining data-supported generative subgoal modeling with inference-time goal-directed guidance, consistent with the advantage-like interpretation of CFG discussed in Section~\ref{sec:5.3}.

\subsection{Comparison with Diffusion-Based Hierarchical Planners}
\label{app:diffusion_planner_comparison}
\begin{table}[h]
    \centering

    \setlength{\abovecaptionskip}{3pt}
    \setlength{\belowcaptionskip}{2pt}

    \caption{
        Success rates (\%) under the matched OGBench protocol, reported as
        mean $\pm$ standard deviation over five seeds.
        Bold indicates the highest mean in each setting.
    }
    \label{tab:diffusion_planner_comparison}

    \footnotesize
    \setlength{\tabcolsep}{4pt}
    \renewcommand{\arraystretch}{0.88}

    \begin{tabular*}{0.96\linewidth}{
        @{\extracolsep{\fill}}llcccc@{}
    }
        \toprule
        Datasets & Representation & Merlin & HDMI & HD & DSP \\
        \midrule

        \texttt{PointMaze-Large}
        & \emph{Full} $=$ XY
        & $43.2\pm8.1$
        & $61.2\pm4.1$
        & $73.6\pm6.8$
        & $\mathbf{93.6}\pm2.9$ \\

        \texttt{PointMaze-Giant}
        & \emph{Full} $=$ \emph{XY}
        & $0.4\pm0.5$
        & $26.4\pm5.4$
        & $12.6\pm6.9$
        & $\mathbf{43.4}\pm7.2$ \\

        \midrule

        \multirow[c]{2}{*}{\texttt{AntMaze-Large}}
        & \emph{Full}
        & $22.6\pm5.2$
        & $53.4\pm5.6$
        & $64.0\pm8.2$
        & $\mathbf{92.0}\pm4.1$ \\

        & \emph{XY}
        & $34.2\pm5.8$
        & $65.0\pm3.7$
        & $74.4\pm5.5$
        & $\mathbf{96.0}\pm2.0$ \\

        \midrule

        \multirow[c]{2}{*}{\texttt{AntMaze-Giant}}
        & \emph{Full}
        & $0.0\pm0.0$
        & $0.0\pm0.0$
        & $0.2\pm0.4$
        & $\mathbf{69.2}\pm2.9$ \\

        & \emph{XY}
        & $0.0\pm0.0$
        & $22.8\pm3.5$
        & $8.6\pm3.7$
        & $\mathbf{77.2}\pm2.7$ \\

        \bottomrule
    \end{tabular*}
\end{table}

We compare DSP with Merlin \cite{Merlin}, HDMI \cite{HDMI}, and HD
\citep{HD}. All methods use the same OGBench datasets,
$1$M training updates, five seeds, and the same five-goal evaluation protocol.
We reimplement the three baselines in JAX while retaining their
method-specific architectures and tuning their key hyperparameters.

Because these methods originally use different goal representations, we
evaluate two interfaces on AntMaze: \emph{Full} uses full OGBench goal
observations and full-state subgoals, while \emph{XY} uses only two-dimensional
positions for both goals and subgoals. PointMaze states are already
two-dimensional, so \emph{Full} and \emph{XY} are equivalent.

As shown in \autoref{tab:diffusion_planner_comparison}, DSP achieves the
highest mean success in all six evaluated task--representation settings.
HDMI and HD generally improve under the lower-dimensional \emph{XY} interface,
confirming that goal representation affects high-level modeling difficulty.
DSP remains strongest under both \emph{Full} and \emph{XY} representations in this matched
evaluation protocol.

\section{Inference-Time Analysis}
\label{apdx:infer_time}

We analyze the inference-time efficiency of DSP and compare it with hierarchical baselines under matched budgets. In addition to the performance–latency trade-off of DSP shown in Appendix~\ref{apdx:hyper_sensitivity_N}, we provide a direct comparison across methods.

\begin{figure}[h]
    \centering
    \includegraphics[width=0.6\linewidth]{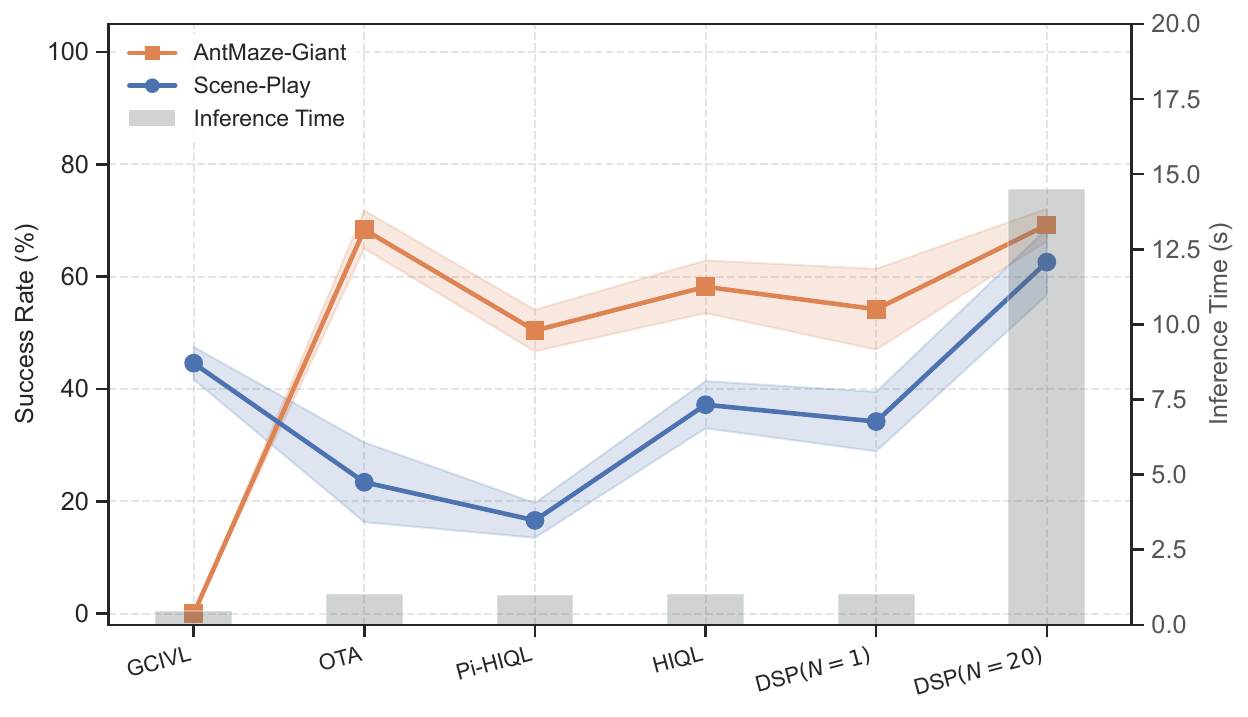}
    \vspace{-0.3cm}
    \caption{Comparison of success rate and inference time across methods. The left axis shows success rates on \texttt{antmaze-giant} and \texttt{scene-play}, while the right axis shows inference time. DSP with $N=1$ corresponds to the closest matched-budget setting, while increasing the number of diffusion steps (e.g., $N=20$) leads to improved performance at higher inference cost.}
    \label{fig:inference_comparison}
\end{figure}
\begin{table}[h]
    \centering
    \renewcommand{\arraystretch}{1.5}
    \caption{Training time comparison of different methods measured on the same hardware.}
    \resizebox{0.7\linewidth}{!}{\begin{tabular}{lccccc}
    \hline
    \textbf{Method} & \textbf{GCIVL} & \textbf{OTA} & \textbf{Pi-HIQL} & \textbf{HIQL} & \textbf{DSP} \\
    \hline
    Training Time (min) & 24 & 56 & 53 & 42 & 57 \\
    \hline
    \end{tabular}}
\label{tab:training_time}
\end{table}

As shown in Fig.~\ref{fig:inference_comparison}, DSP with $N=1$ achieves performance comparable to HIQL under similar inference time, indicating that DSP does not rely on increased computation to be effective. Increasing the number of diffusion steps further improves performance (e.g., $+11$ on \texttt{antmaze-giant} and $+25$ on \texttt{scene-play} for $N=20$), demonstrating a favorable performance–latency trade-off.

We note that inference time is measured during evaluation over 10,000 repeated inferences, where each inference consists of generating one subgoal and one action. We also report the training time of different methods on the same hardware for reference, as shown in Table~\ref{tab:training_time}.

\section{Stitching Experiments}
\label{apdx:stitch}
We evaluate the stitching capability of DSP on the OGBench stitch datasets, which test whether an algorithm can solve long-horizon tasks by composing shorter offline trajectory fragments \citep{OGBench, PCDT}. Table~\ref{tab:stitch_result} compares DSP with flat and hierarchical offline GCRL baselines. We include several DSP variants to isolate the effects of goal sampling and guidance strength. The superscript ``mix'' denotes the variant trained with the OGBench actor goal mixing strategy for stitching tasks, i.e., $(p_{\mathrm{cur}},p_{\mathrm{traj}},p_{\mathrm{rand}})=(0,0.5,0.5)$. The subscript $1$ denotes standard conditional sampling with guidance scale $\omega=1$. DSP without a subscript uses the task-specific guidance scale reported in Appendix~\ref{apdx:hyperparameters}.

\begin{table*}[h]
\centering
\caption{Complete comparison between DSP and the offline GCRL baselines. The table reports the average binary success rate (\%) across five test-time goals for each task, averaged over 5 seeds. Standard deviations are indicated by the $\pm$ symbol. Entries within 95\% of the best-performing value in each row are highlighted in \textbf{bold}.}
\label{tab:stitch_result}
\resizebox{\linewidth}{!}{\begin{tabular}{lccccccccc}
\toprule
\textbf{Datasets} & 
\multicolumn{3}{c}{\textbf{Flat Policies}} & 
\multicolumn{6}{c}{\textbf{Hierarchical Policies}} \\
\cmidrule(r){2-4} \cmidrule(l){5-10}&
\textbf{GCBC} & 
\textbf{CFGRL} & 
\textbf{GCIVL} & 
\textbf{HIQL} & 
\textbf{HIQL$^{\mathrm{w/o}}$} & 
\textbf{DSP$^{\mathrm{mix}}_1$} &
\textbf{DSP$_1$} &
\textbf{DSP$^{\mathrm{mix}}$} &
\textbf{DSP} \\
\midrule
\texttt{pointmaze-medium-stitch-v0} & $23.2\pm16.7$ & $24.6\pm8.3$ & $20.8\pm4.0$ & $62.0\pm6.8$ & $\mathbf{69.4}\pm13.6$ & $2.4\pm0.9$ & $54.6\pm3.0$ & $5.0\pm2.0$ & $65.8\pm5.8$ \\
\texttt{pointmaze-large-stitch-v0} & $4.4\pm8.3$ & $13.0\pm3.9$ & $8.0\pm11.0$ & $22.0\pm7.5$ & $8.2\pm7.7$ & $0.0\pm0.0$ & $38.0\pm1.9$ & $3.6\pm1.1$ & $\mathbf{44.2}\pm4.3$ \\
\texttt{pointmaze-giant-stitch-v0} & $0.0\pm0.0$ & $0.0\pm0.0$ & $0.0\pm0.0$ & $0.0\pm0.0$ & $0.0\pm0.0$ & $0.0\pm0.0$ & $0.0\pm0.0$ & $0.0\pm0.0$ & $\mathbf{1.2}\pm0.8$ \\
\texttt{pointmaze-teleport-stitch-v0} & $26.8\pm8.7$ & $13.8\pm8.9$ & $28.6\pm5.0$ & $33.2\pm2.5$ & $8.4\pm4.9$ & $19.6\pm4.2$ & $43.4\pm7.5$ & $14.0\pm5.7$ & $\mathbf{51.8}\pm4.9$\\
\midrule
\texttt{antmaze-medium-stitch-v0} & $52.6\pm5.7$ & $44.4\pm5.3$ & $32.6\pm10.4$ & $\mathbf{89.8}\pm3.8$ & $\mathbf{89.6}\pm3.4$ & $66.6\pm9.2$ & $75.8\pm6.6$ & $65.8\pm5.5$ & $\mathbf{85.8}\pm3.3$ \\
\texttt{antmaze-large-stitch-v0} & $1.8\pm2.5$ & $4.4\pm4.7$ & $5.0\pm4.9$ & $\mathbf{63.6}\pm8.6$ & $\mathbf{66.4}\pm3.8$ & $19.6\pm5.3$ & $22.0\pm6.9$ & $48.2\pm2.9$ & $51.8\pm4.9$ \\
\texttt{antmaze-giant-stitch-v0} & $0.0\pm0.0$ & $0.0\pm0.0$ & $0.0\pm0.0$ & $0.0\pm0.0$ & $0.8\pm0.8$ & $1.8\pm2.0$ & $2.6\pm1.3$ & $\mathbf{9.0}\pm4.0$ & $\mathbf{8.6}\pm2.6$ \\
\texttt{antmaze-teleport-stitch-v0} & $31.4\pm7.7$ & $13.0\pm6.0$ & $27.2\pm5.4$ & $34.0\pm3.2$ & $33.8\pm3.4$ & $26.8\pm6.6$ & $\mathbf{45.2}\pm6.0$ & $26.8\pm4.7$ & $39.2\pm6.1$ \\
\midrule
\texttt{humanoidmaze-medium-stitch-v0} & $26.2\pm5.4$ & $18.6\pm6.8$ & $31.8\pm6.1$ & $\mathbf{73.8}\pm3.9$ & $62.4\pm3.0$ & $53.2\pm3.3$ & $58.6\pm4.0$ & $50.8\pm6.3$ & $\mathbf{73.2}\pm4.3$ \\
\texttt{humanoidmaze-large-stitch-v0} & $4.2\pm1.7$ & $4.0\pm2.1$ & $5.0\pm3.8$ & $\textbf{19.2}\pm2.9$ & $8.2\pm1.9$ & $2.4\pm1.3$ & $4.0\pm2.1$ & $8.2\pm3.6$ & $\mathbf{19.8}\pm4.4$ \\
\texttt{humanoidmaze-giant-stitch-v0} & $0.0\pm0.0$ & $0.0\pm0.0$ & $0.0\pm0.0$ & $1.8\pm1.9$ & $0.0\pm0.0$ & $0.2\pm0.4$& $0.0\pm0.0$ & $17.6\pm7.0$ & $\mathbf{35.2}\pm5.9$ \\
\midrule
\texttt{antsoccer-arena-stitch-v0} & $21.6\pm3.9$ & $12.6\pm6.4$ & $3.8\pm1.8$ & $14.2\pm2.8$ & $35.6\pm3.4$ & $15.6\pm3.8$ & $16.6\pm4.7$ & $24.4\pm5.7$ & $\mathbf{40.2}\pm5.2$ \\
\texttt{antsoccer-medium-stitch-v0} & $3.8\pm2.8$ & $1.2\pm1.1$ & $0.4\pm0.9$ & $4.0\pm1.2$ & $6.0\pm1.2$ & $2.6\pm1.5$ & $6.8\pm2.3$ & $5.2\pm1.8$ & $\mathbf{10.6}\pm3.3$ \\
\bottomrule
\end{tabular}}
\end{table*}

Overall, DSP performs competitively on stitching tasks and substantially improves over flat policy baselines on most environments. With task-specific guidance, DSP achieves the best or near-best performance in several settings, including \texttt{pointmaze-large}, \texttt{pointmaze-teleport}, \texttt{antmaze-giant}, \texttt{humanoidmaze-giant}, and both \texttt{antsoccer} stitching tasks. These results suggest that guided generative subgoal planning can support trajectory composition when the learned subgoal distribution captures sufficient local connectivity.

The comparison between DSP$_1$ and DSP shows that guidance strength is important in stitching regimes. In many larger or more dynamically complex environments, using the task-specific guidance scale improves performance over standard conditional sampling, suggesting that properly tuned goal-directed guidance helps select more useful subgoals from the learned conditional distribution. This is consistent with prior diffusion-based decision-making work suggesting that conditional generative models can exhibit implicit dynamic-programming-like behavior through trajectory or subtrajectory generation \citep{DD, SDD}. In DSP, this effect appears at the subgoal level: CFG amplifies the goal-conditioned component of the learned subgoal distribution, helping compose locally supported transitions into longer-horizon behavior.

At the same time, DSP does not uniformly dominate Bellman-backup-based hierarchical methods. HIQL and HIQL$^{\mathrm{w/o}}$ remain stronger on some \texttt{antmaze} and \texttt{humanoidmaze} stitch tasks, likely because value-based methods can propagate goal information across trajectory fragments through Bellman backups. By contrast, DSP relies more directly on the transition structure captured by its generative subgoal model. Thus, the stitching results show that DSP can perform effective trajectory composition in several regimes, rather than replacing Bellman-backup-based stitching in all settings.

Finally, the DSP$^{\mathrm{mix}}$ variants indicate that the default OGBench actor goal mixing strategy is not always suitable for generative high-level planning. While random goal relabeling can benefit value-based policy extraction, excessive random goals may shift the learned conditional distribution toward broad connectivity rather than directional, goal-relevant subgoals. The structured goal sampling used by DSP better preserves the goal-directed signal amplified by CFG, highlighting the interaction between goal sampling and inference-time guidance.

\section{Data Coverage and Behavior Quality}
\label{app:data_coverage_quality}

\subsection{Reduced-Data Coverage}
\label{app:reduced_data_coverage}

Using a fixed subset seed, we construct nested datasets by randomly ordering
complete trajectories and retaining them until reaching each target data
fraction. This preserves trajectory boundaries and future-state relabeling,
and all methods use identical subsets.

OGBench evaluates AntMaze and HumanoidMaze success using XY position with a
tolerance of $0.5$. We therefore define $0.5$-unit bins as
$b_{0.5}(x,y)=(\lfloor x/0.5\rfloor,\lfloor y/0.5\rfloor)$ and measure

\begin{equation}
    C_{\mathrm{XY}}(\mathcal D_r)
    =
    \frac{
        \left|
        \{b_{0.5}(x_t,y_t):(x_t,y_t)\in\mathcal D_r\}
        \right|
    }{
        \left|
        \{b_{0.5}(x_t,y_t):(x_t,y_t)\in\mathcal D_{100}\}
        \right|
    }.
\end{equation}

$C_{\mathrm{XY}}$ measures relative spatial occupancy, not full-state
coverage. Data fraction additionally affects visitation density, local
transitions, and same-trajectory future-state pairs.

\begin{table}[h]
    \centering
    \caption{
        Success rates (\%) under nested complete-trajectory subsets.
        Results are mean $\pm$ standard deviation over five seeds.
        $C_{\mathrm{XY}}$ denotes relative XY-bin coverage.
    }
    \label{tab:reduced_data_coverage}
    \small
    \setlength{\tabcolsep}{5.5pt}
    \begin{tabular}{@{}lccccc@{}}
        \toprule
        Datasets & Fraction & $C_{\mathrm{XY}}$
        & HIQL & HIQL$^{\mathrm{w/o}}$ & DSP \\
        \midrule
        \multirow[c]{4}{*}{\texttt{AntMaze-Large}}
        & 1.00 & 1.00 & $87.8\pm1.5$ & $88.6\pm2.1$
        & $\mathbf{92.0}\pm4.1$ \\
        & 0.50 & 0.96 & $87.0\pm4.3$ & $81.4\pm3.2$
        & $\mathbf{90.2}\pm1.9$ \\
        & 0.25 & 0.91 & $79.8\pm2.7$ & $70.2\pm5.1$
        & $\mathbf{83.0}\pm2.1$ \\
        & 0.10 & 0.85 & $19.4\pm4.7$ & $39.6\pm5.6$
        & $\mathbf{52.8}\pm6.6$ \\
        \midrule
        \multirow[c]{3}{*}{\texttt{HumanoidMaze-Giant}}
        & 1.00 & 1.00 & $5.0\pm1.6$ & $0.0\pm0.0$
        & $\mathbf{66.0}\pm3.4$ \\
        & 0.10 & 0.94 & $1.0\pm0.7$ & $0.0\pm0.0$
        & $\mathbf{44.4}\pm4.2$ \\
        & 0.01 & 0.50 & $0.0\pm0.0$ & $0.0\pm0.0$
        & $\mathbf{5.6}\pm3.0$ \\
        \bottomrule
    \end{tabular}
\end{table}

DSP achieves the highest mean at every evaluated fraction. However, it drops
from $92.0$ to $52.8$ on AntMaze-Large and from $66.0$ to $5.6$ on
HumanoidMaze-Giant, confirming that severe loss of data support remains a
failure condition. These results support DSP's relative competitiveness under
the evaluated reduced-data settings, rather than coverage independence.

\subsection{Official Explore Datasets}
\label{app:explore_datasets}

The official OGBench Explore datasets contain non-goal-directed trajectories:
the locomotion direction is randomly resampled every 10 steps, while actions
are produced by a pretrained Ant controller and perturbed with action noise of
magnitude $1.0$ \citep{OGBench}. They therefore test poor global trajectory
quality while retaining locally executable transitions.

\begin{table}[h]
    \centering
    \caption{
        Success rates (\%) on the official AntMaze Explore datasets,
        reported as mean $\pm$ standard deviation over five seeds.
    }
    \label{tab:explore_results}
    \small
    \setlength{\tabcolsep}{9pt}
    \begin{tabular}{@{}lccc@{}}
        \toprule
        Datasets & HIQL & HIQL$^{\mathrm{w/o}}$ & DSP \\
        \midrule
        \texttt{antmaze-medium-explore-v0}
        & $32.4\pm10.0$ & $33.4\pm9.7$
        & $\mathbf{60.4}\pm9.3$ \\
        \texttt{antmaze-large-explore-v0}
        & $1.6\pm3.6$ & $6.0\pm6.7$
        & $\mathbf{16.4}\pm7.5$ \\
        \texttt{antmaze-teleport-explore-v0}
        & $26.8\pm16.6$ & $\mathbf{30.4}\pm7.3$
        & $21.4\pm4.6$ \\
        \bottomrule
    \end{tabular}
\end{table}

DSP achieves the highest mean on Medium and Large, but not on Teleport.
Teleport randomly sends the agent to one of multiple exits, including a dead
end; hence, a future state observed in the data is not necessarily an outcome
the policy can reliably select. Future-state relabeling may therefore include
favorable but uncontrollable outcomes, while replanning can only respond after
the realized exit is observed. The result identifies stochastic,
action-uncontrollable transitions as a limitation, rather than establishing
uniform robustness across all Explore settings.

\section{Evaluation with Visual Observations}
\label{app:visual_evaluation}

We evaluate DSP on four visual OGBench tasks using only
$64\times64\times3$ RGB observations. Following the OGBench visual protocol,
HIQL and DSP use the same IMPALA encoder architecture. DSP does not perform
diffusion directly in pixel space; instead, it generates subgoals in the
IMPALA latent space. HIQL additionally maps encoded state--subgoal pairs to
its fixed 10-dimensional subgoal representation.

\begin{table}[h]
    \centering
    \caption{
        Success rates (\%) on visual OGBench tasks, reported as
        mean $\pm$ standard deviation over five seeds.
        Bold indicates the highest mean in each row.
    }
    \label{tab:visual_results}
    \small
    \setlength{\tabcolsep}{10pt}
    \begin{tabular}{@{}lcc@{}}
        \toprule
        Datasets & HIQL & DSP \\
        \midrule
        \texttt{visual-antmaze-medium-navigate}
        & $92.0\pm2.1$
        & $\mathbf{96.2}\pm2.5$ \\

        \texttt{visual-antmaze-large-navigate}
        & $51.8\pm5.5$
        & $\mathbf{68.4}\pm6.2$ \\

        \texttt{visual-antmaze-giant-navigate}
        & $4.4\pm2.8$
        & $\mathbf{13.8}\pm4.4$ \\

        \texttt{visual-scene-play}
        & $43.8\pm3.9$
        & $\mathbf{48.6}\pm1.5$ \\
        \bottomrule
    \end{tabular}
\end{table}

As shown in \autoref{tab:visual_results}, DSP achieves higher mean success on
all four tasks, including improvements of $16.6$ and $9.4$ percentage points
on Visual AntMaze-Large and AntMaze-Giant, respectively. These results show
that DSP can be applied to visual observations through a shared latent encoder,
without requiring diffusion directly in pixel space.

\section{Visualization of Subgoal Trajectories in Maze Environments}
\label{apdx:visualization}
In this section, we present additional visualizations comparing the subgoal trajectories of HIQL$^{\mathrm{w/o}}$ and DSP across the maze environments. Specifically, we display the planning behaviors for all five evaluation tasks, followed by a detailed analysis.

\subsection{Medium Maze}
\begin{figure}[h]
    \centering
    \subfigure[Task 1]{
        \begin{minipage}[b]{0.182\linewidth}
            \centering
            \includegraphics[width=\linewidth]{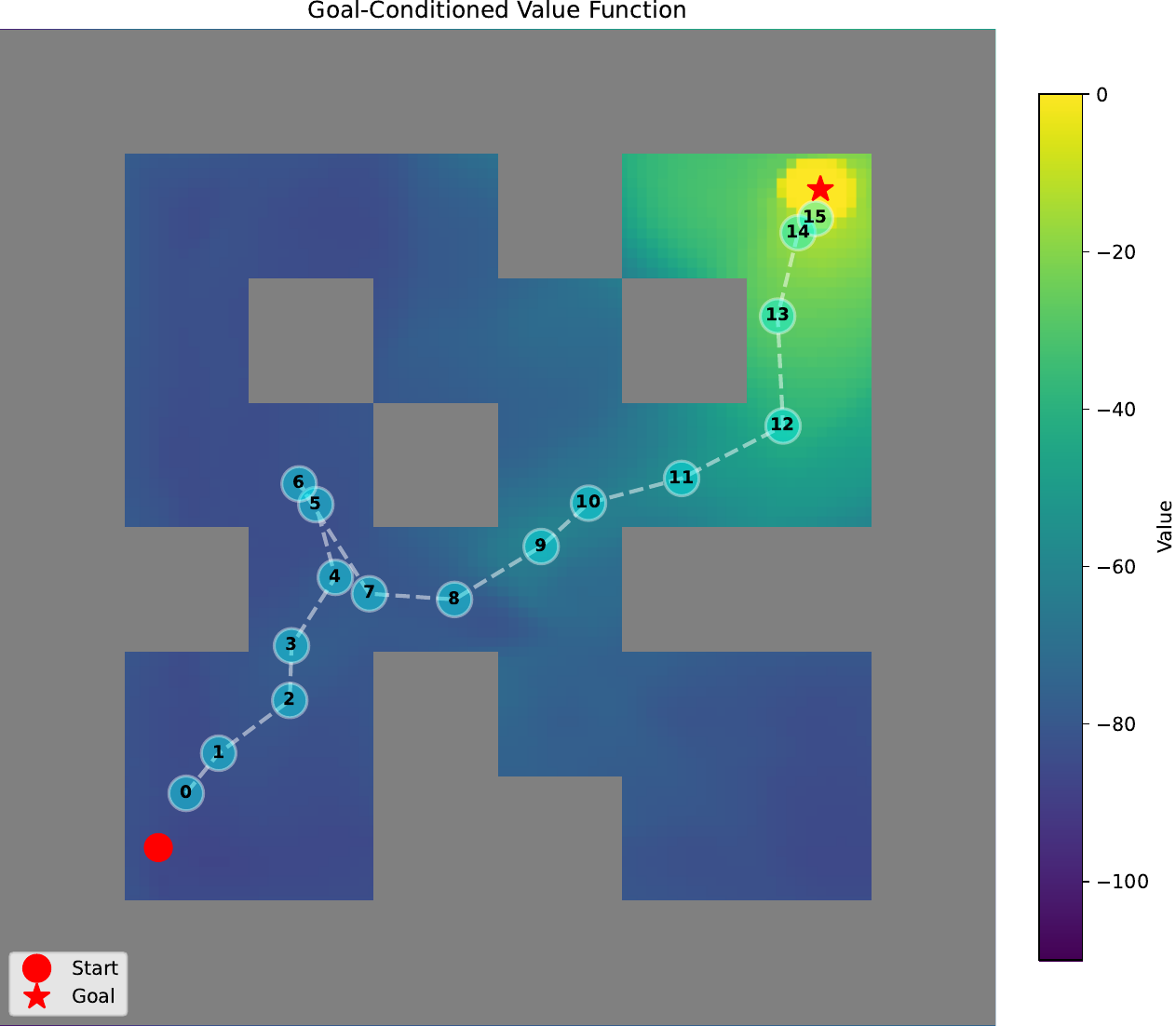} \\
            \includegraphics[width=\linewidth]{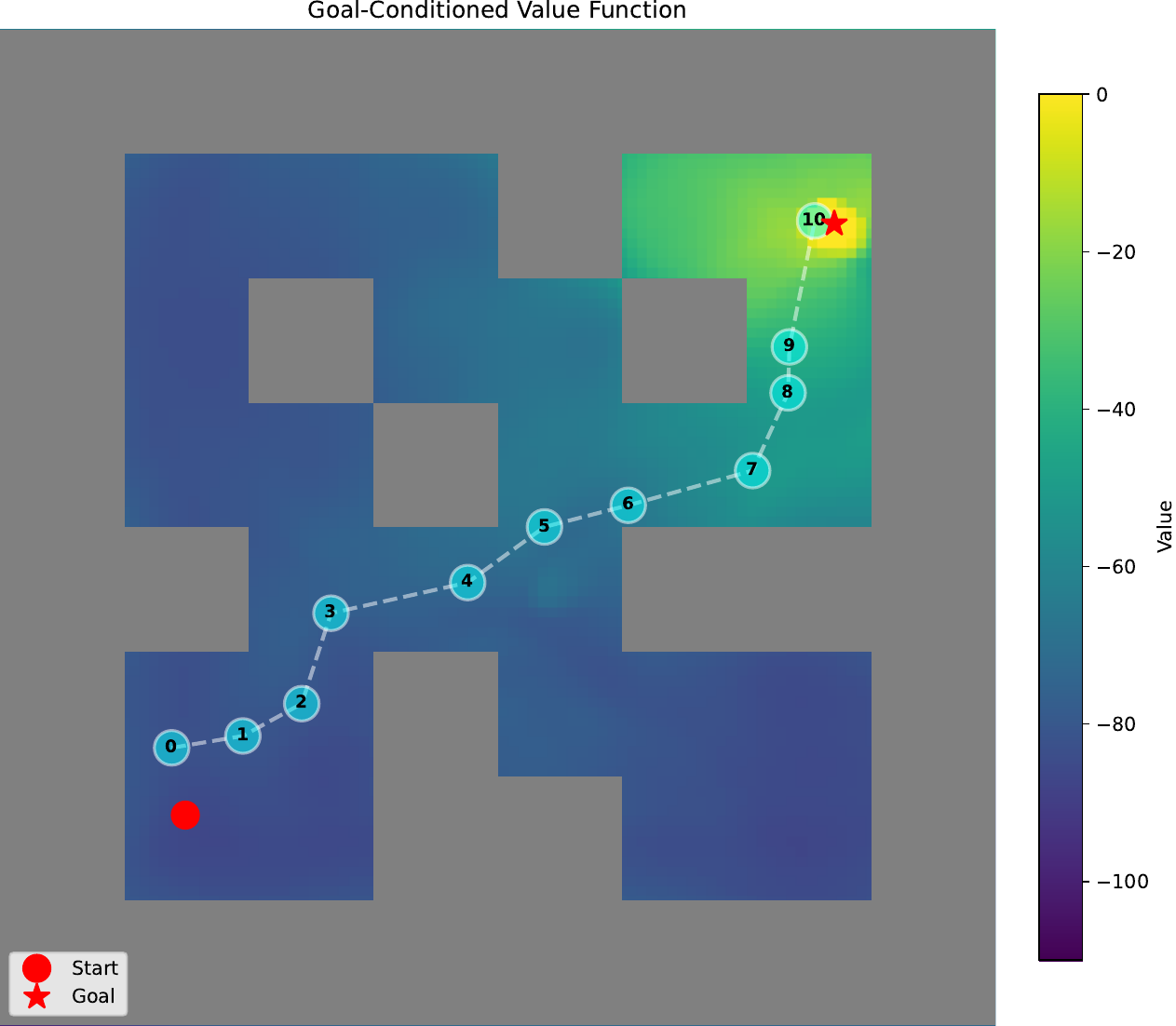}
        \end{minipage}
    }
    \hfill
    \subfigure[Task 2]{
        \begin{minipage}[b]{0.182\linewidth}
            \centering
            \includegraphics[width=\linewidth]{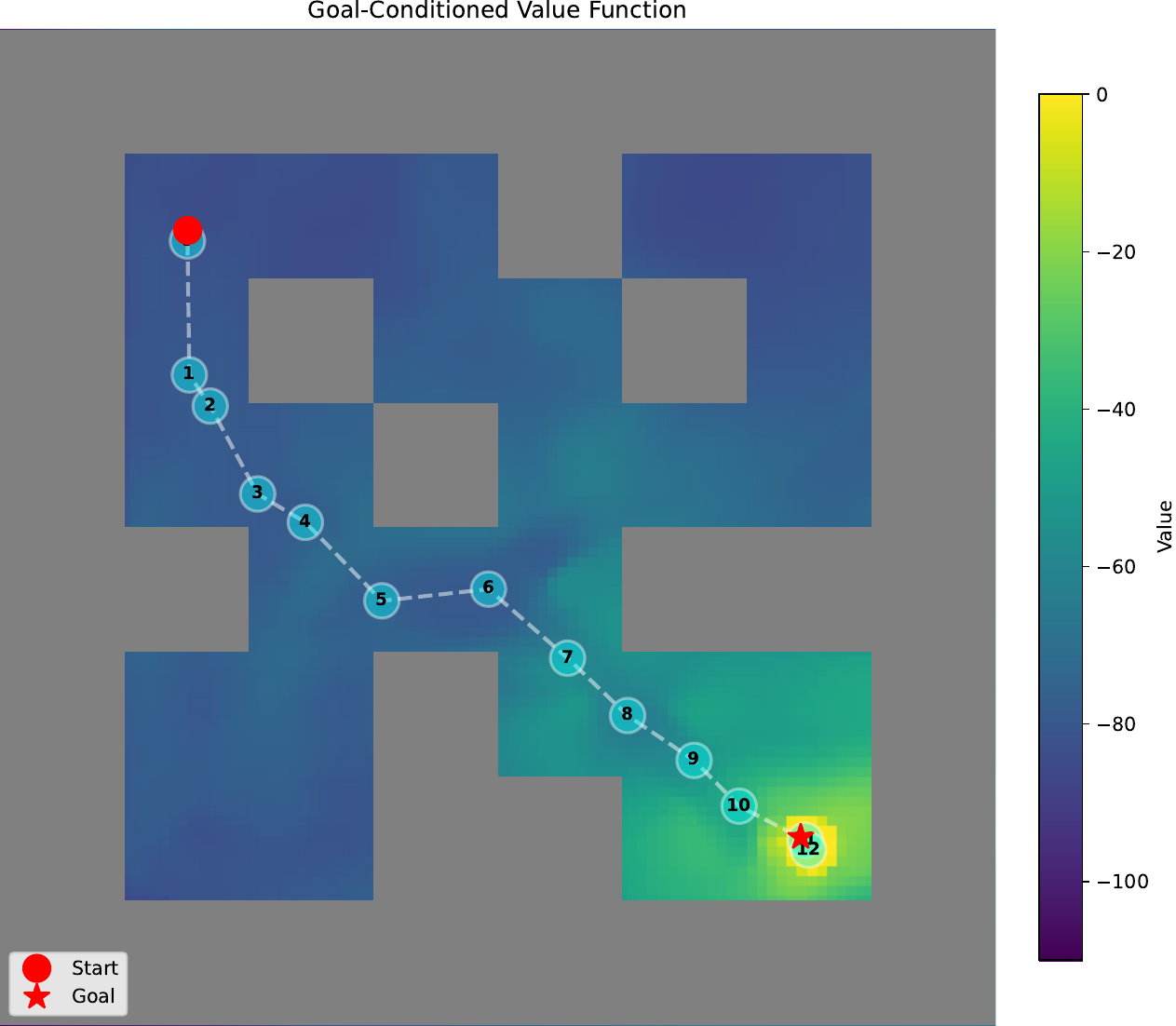} \\
            \includegraphics[width=\linewidth]{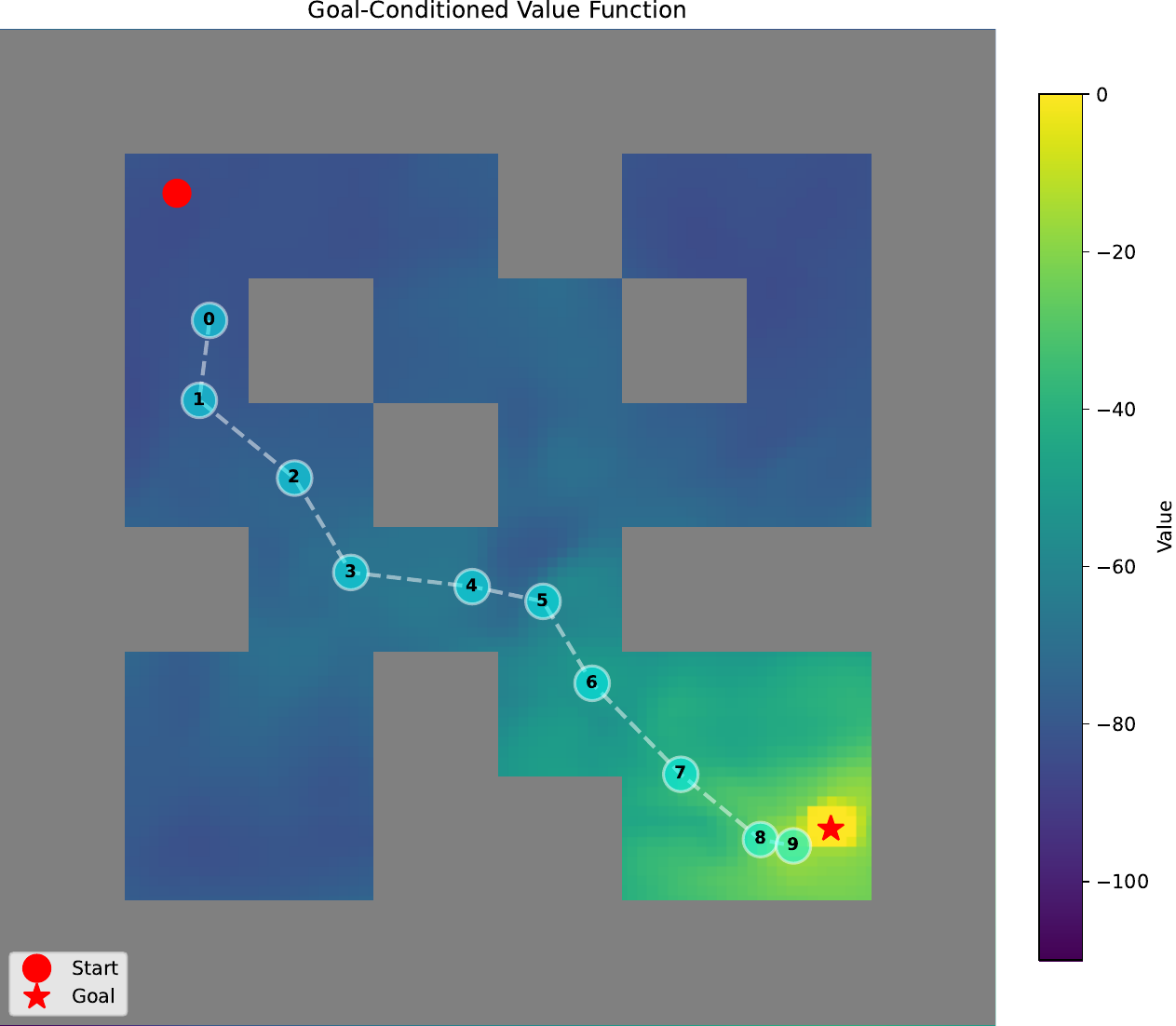}
        \end{minipage}
    }
    \hfill
    \subfigure[Task 3]{
        \begin{minipage}[b]{0.182\linewidth}
            \centering
            \includegraphics[width=\linewidth]{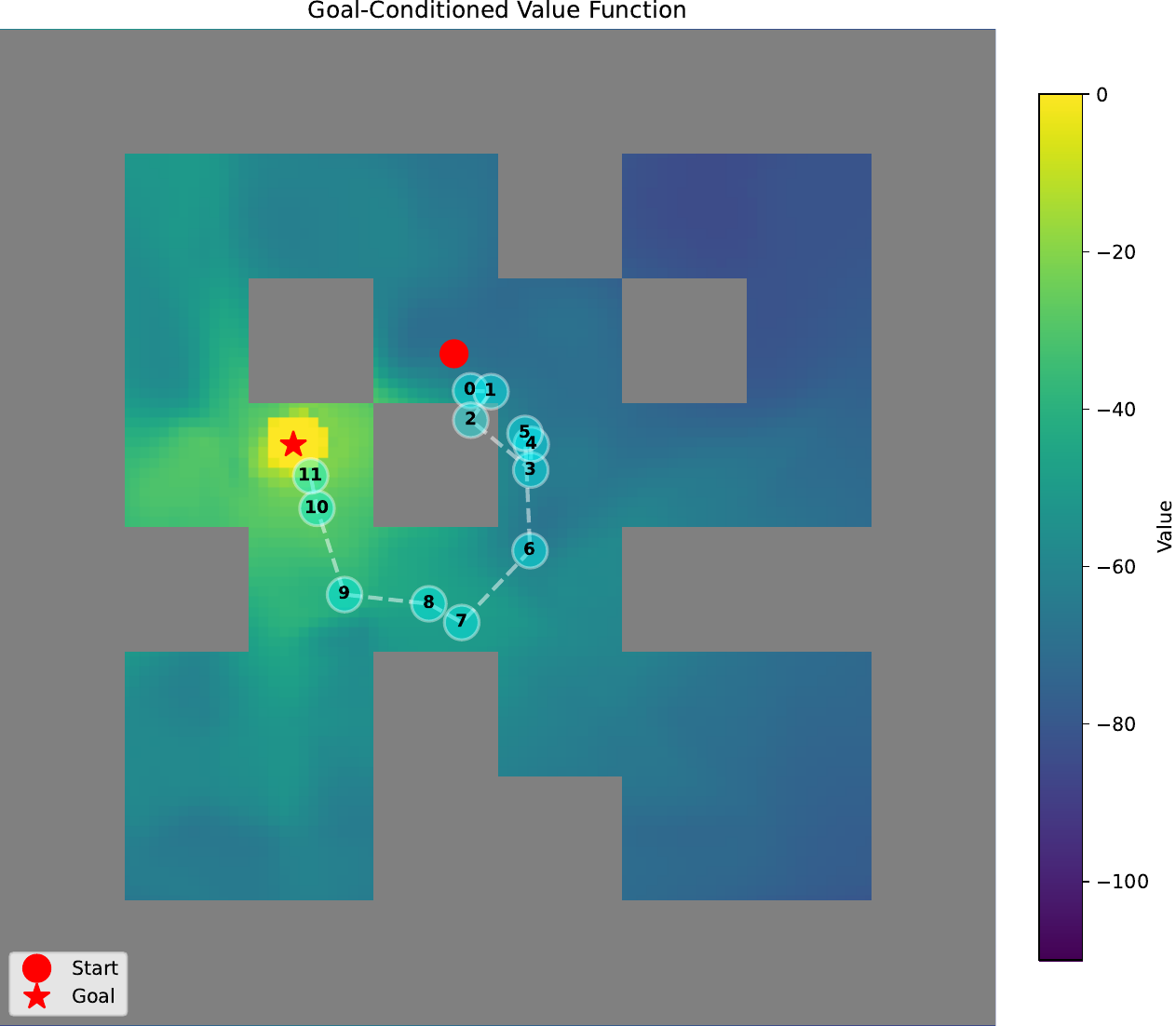} \\
            \includegraphics[width=\linewidth]{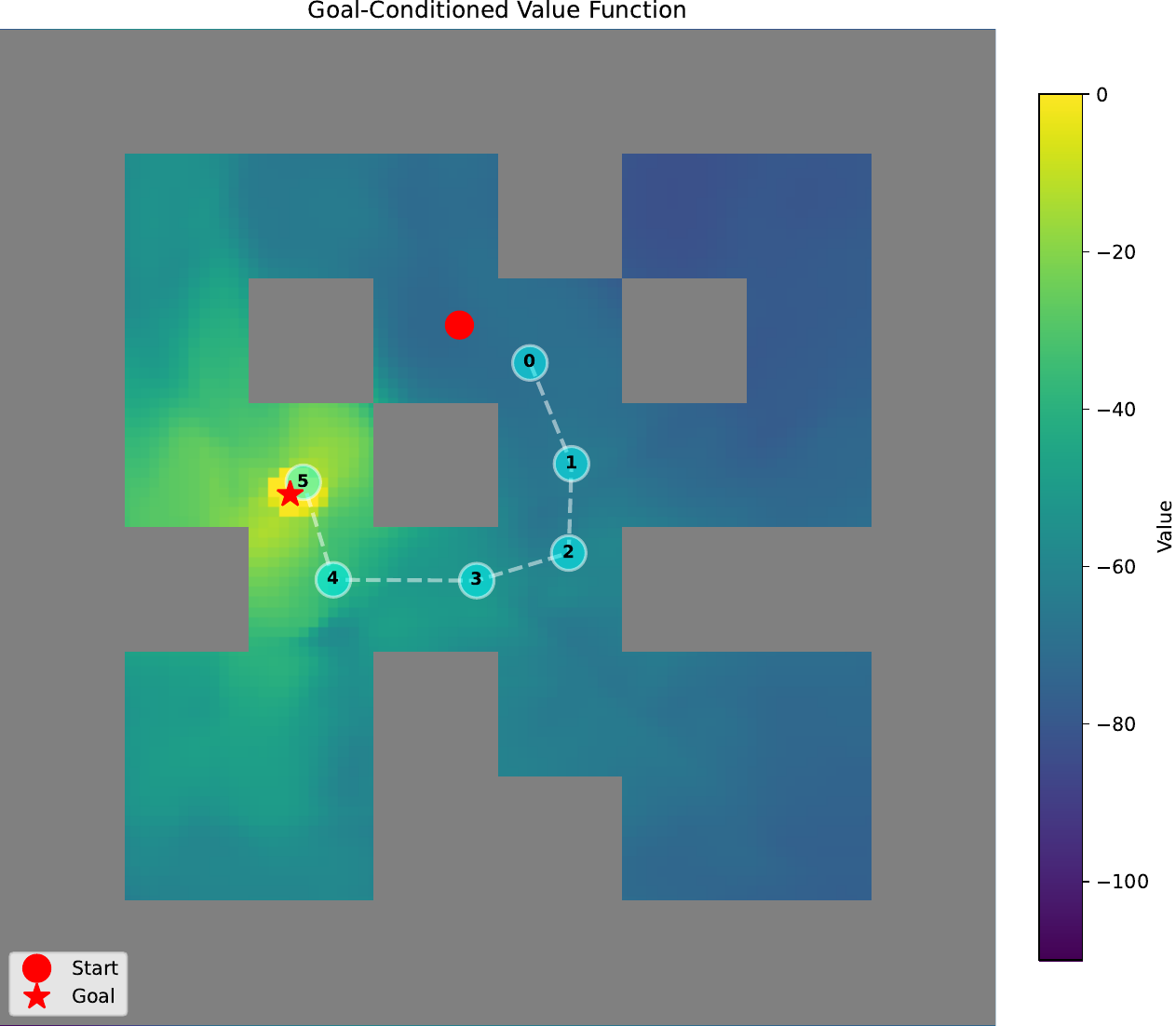}
        \end{minipage}
    }
    \hfill
    \subfigure[Task 4]{
        \begin{minipage}[b]{0.182\linewidth}
            \centering
            \includegraphics[width=\linewidth]{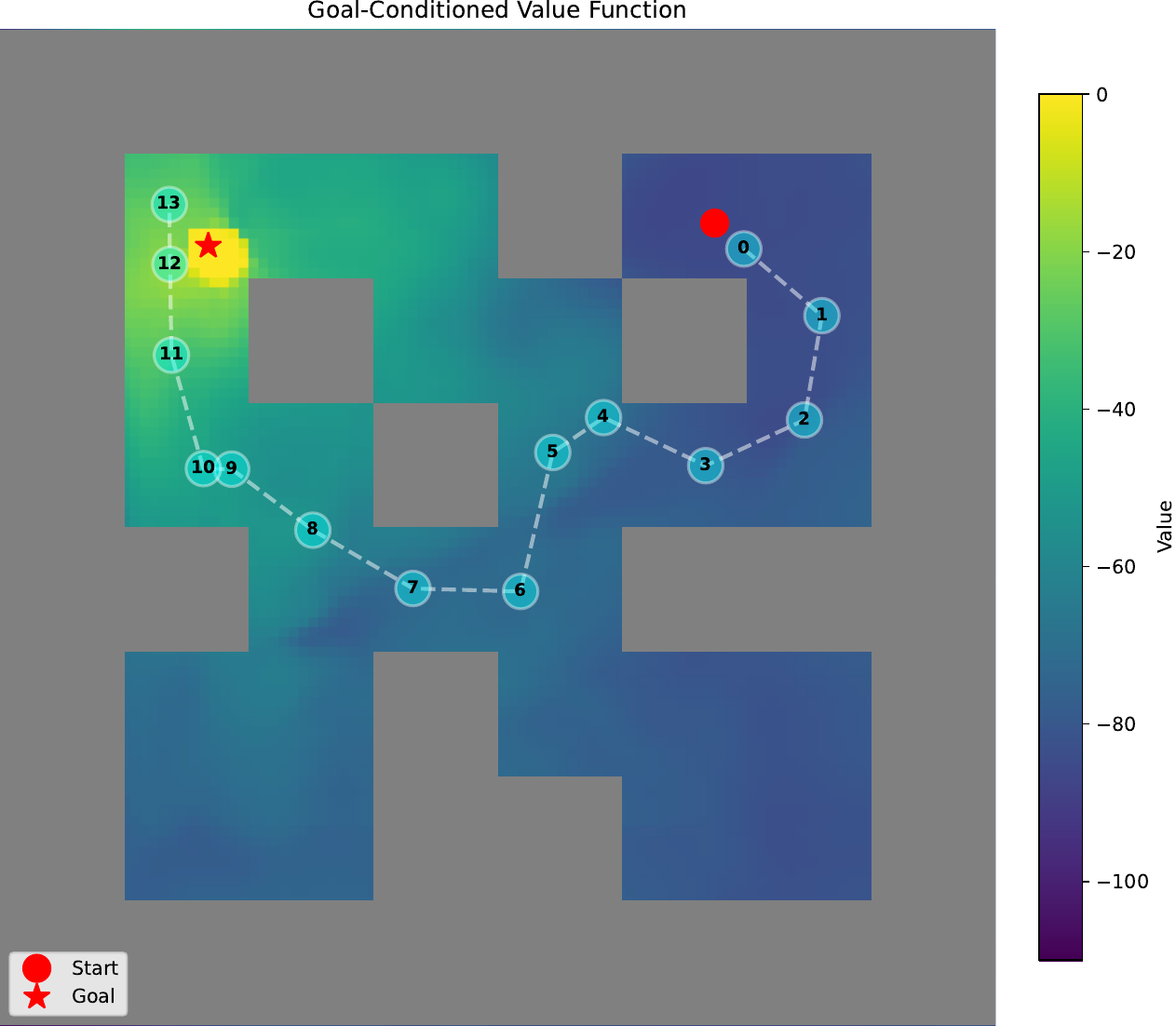} \\
            \includegraphics[width=\linewidth]{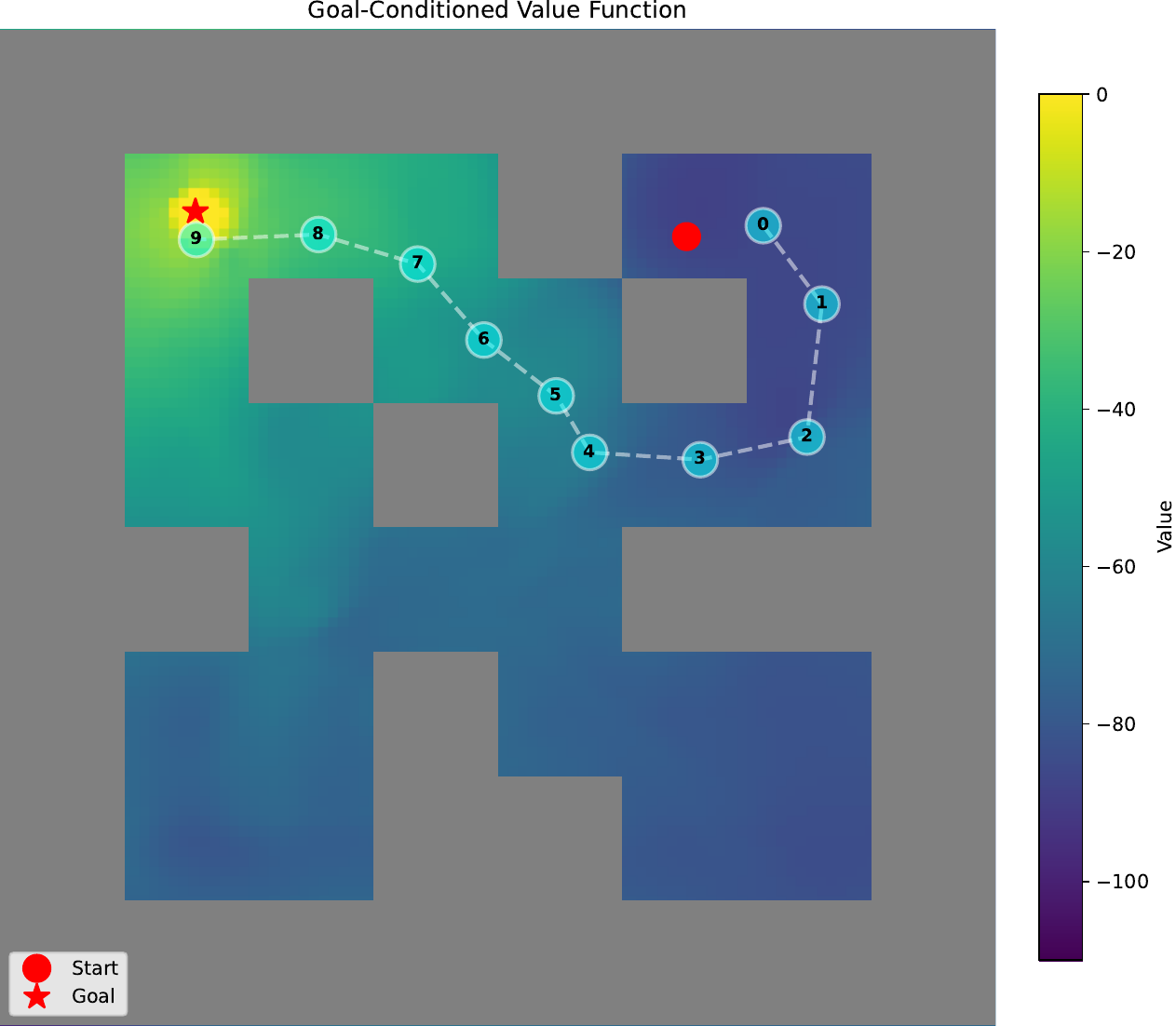}
        \end{minipage}
    }
    \hfill
    \subfigure[Task 5]{
        \begin{minipage}[b]{0.182\linewidth}
            \centering
            \includegraphics[width=\linewidth]{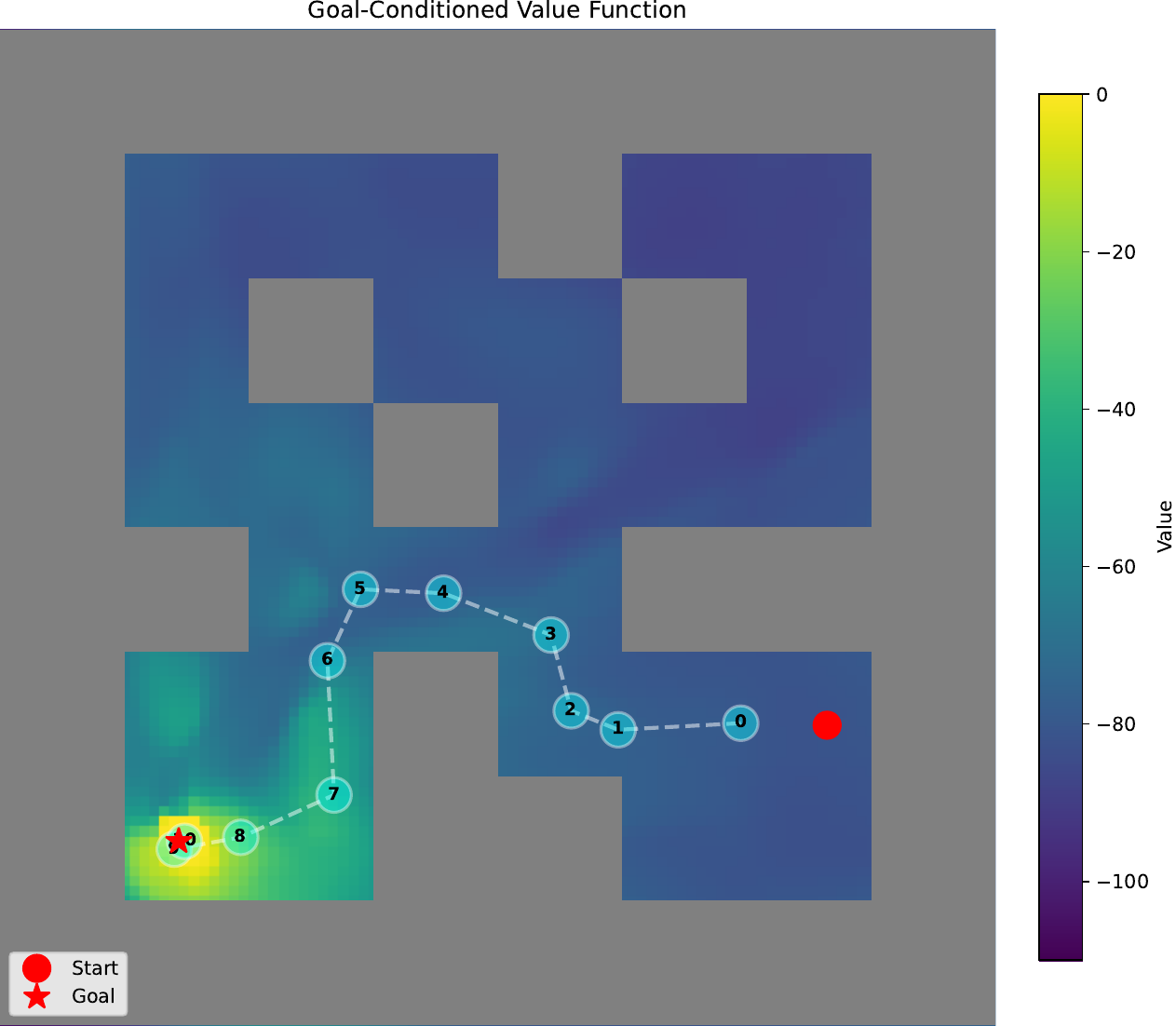} \\
            \includegraphics[width=\linewidth]{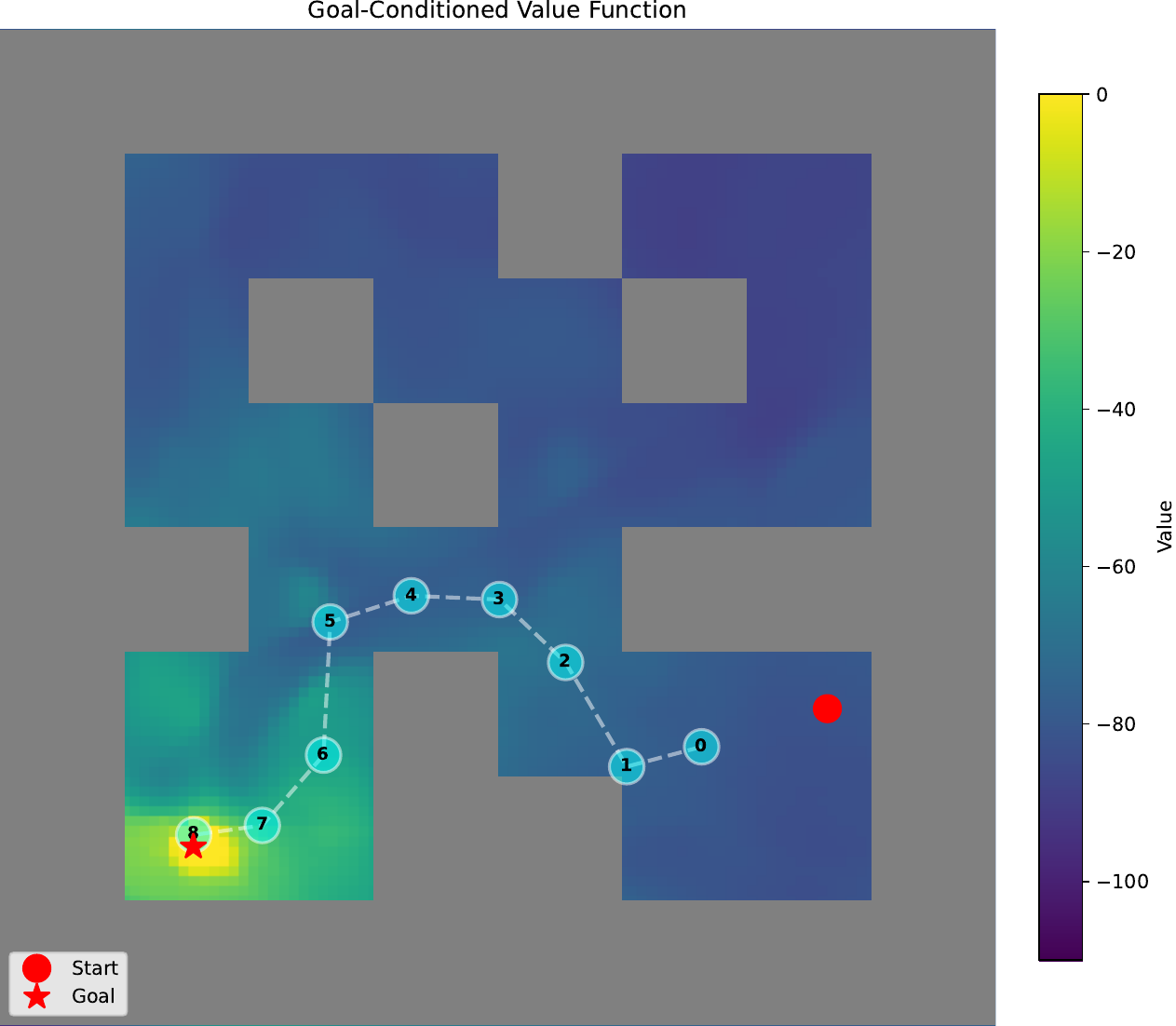}
        \end{minipage}
    }
    \caption{Visualization in the \texttt{medium} maze. Top: HIQL$^{\mathrm{w/o}}$, Bottom: DSP. DSP generates smoother trajectories that better follow the topology.} 
    \label{fig:medium_comparison}
\end{figure}
We first visualize the results in the \texttt{medium}-sized maze. As shown in Figure~\ref{fig:medium_comparison}, while both methods generally reach the goals, we observe distinct behavioral differences. HIQL$^{\mathrm{w/o}}$ occasionally exhibits artifacts characteristic of value-based planning, such as suboptimal detours or subgoals projected into unreachable areas (e.g., walls). In contrast, DSP produces trajectories that are geometrically consistent and better follow the maze topology, generating smooth paths even in this shorter-horizon setting.

\subsection{Large Maze}
Next, we examine the \texttt{large}-sized maze, where the planning horizon increases. As illustrated in Figure~\ref{fig:large_comparison}, the impact of the horizon becomes visible. HIQL$^{\mathrm{w/o}}$ begins to manifest noticeable detours and inefficient paths, showing less consistent planning as the scale increases. Although the value function eventually corrects the agent's direction, the guidance is noisy. DSP, however, maintains its planning stability, producing topologically coherent and smooth trajectories without the jaggedness observed in the baseline.
\begin{figure}[h]
    \centering
    \subfigure[Task 1]{
        \begin{minipage}[b]{0.182\linewidth}
            \centering
            \includegraphics[width=\linewidth]{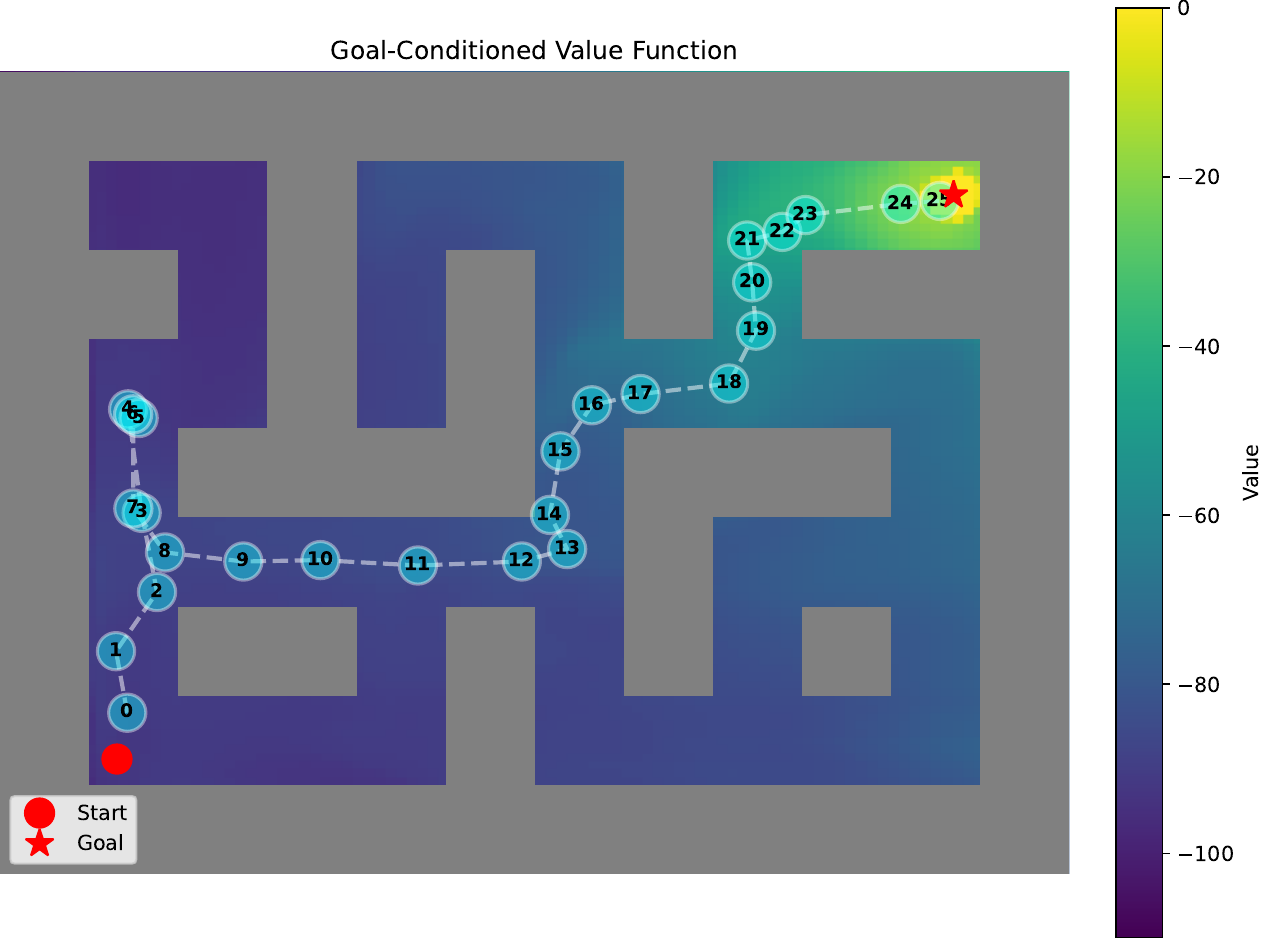} \\
            \includegraphics[width=\linewidth]{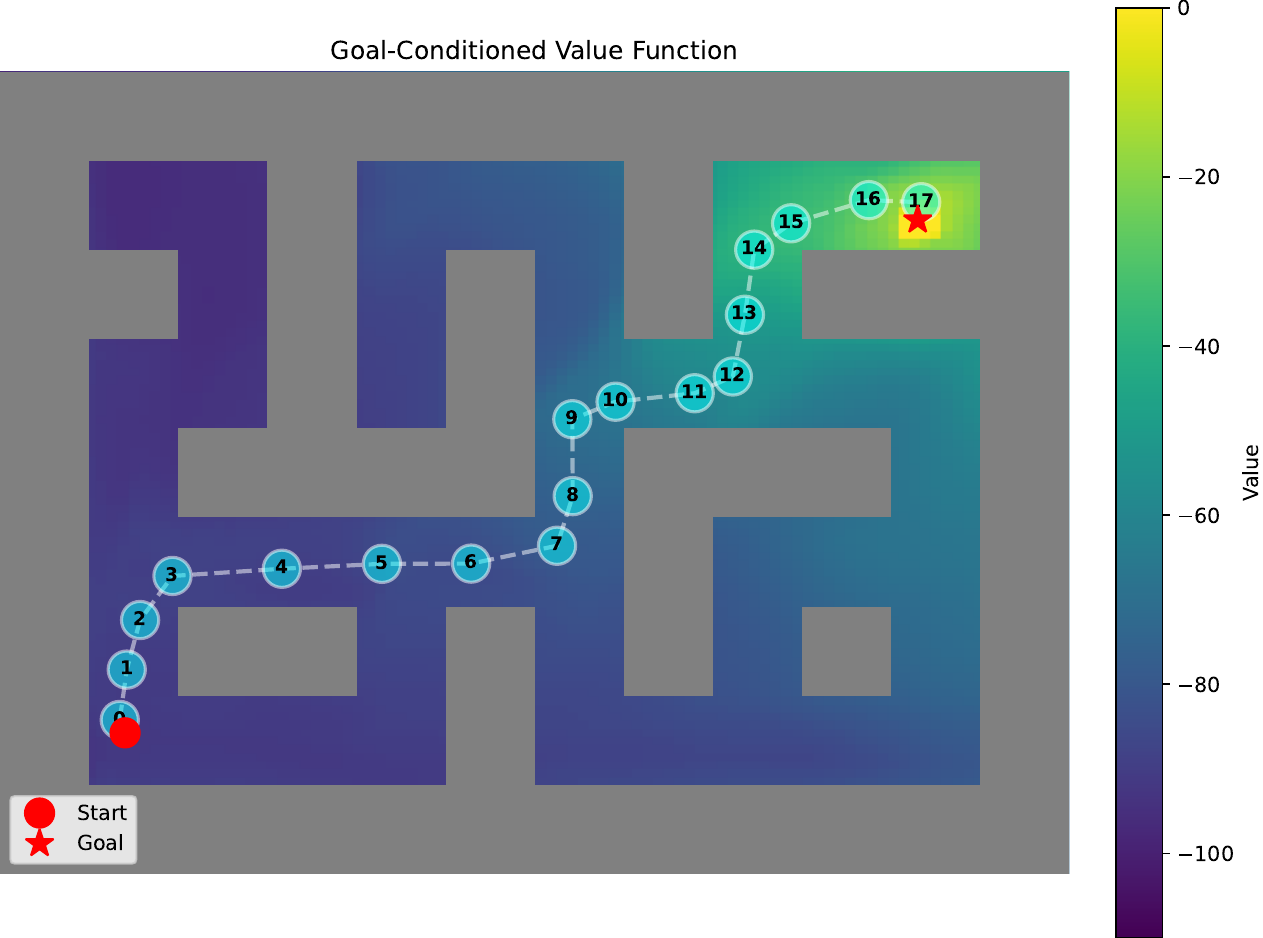}
        \end{minipage}
    }
    \hfill
    \subfigure[Task 2]{
        \begin{minipage}[b]{0.182\linewidth}
            \centering
            \includegraphics[width=\linewidth]{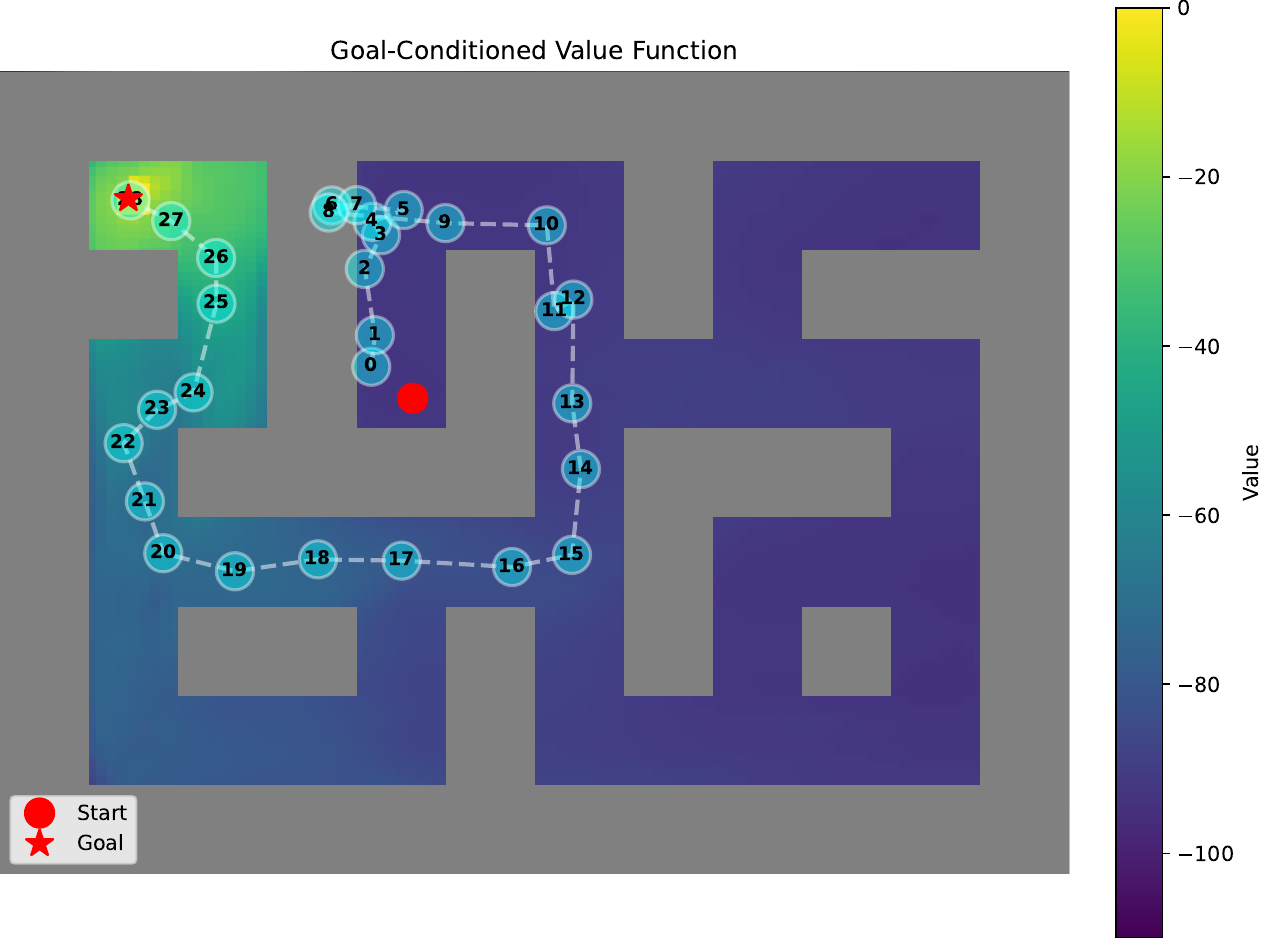} \\
            \includegraphics[width=\linewidth]{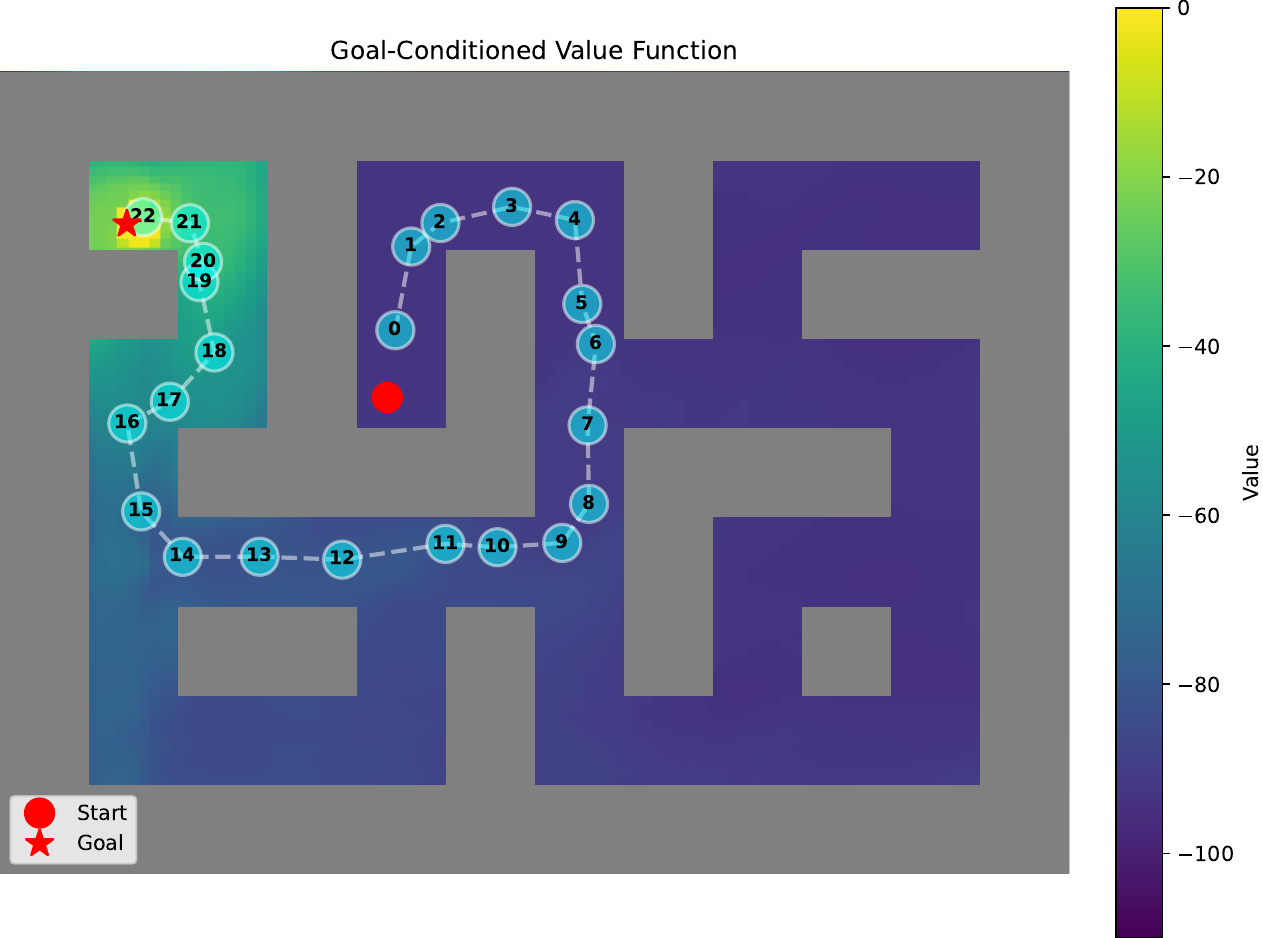}
        \end{minipage}
    }
    \hfill
    \subfigure[Task 3]{
        \begin{minipage}[b]{0.182\linewidth}
            \centering
            \includegraphics[width=\linewidth]{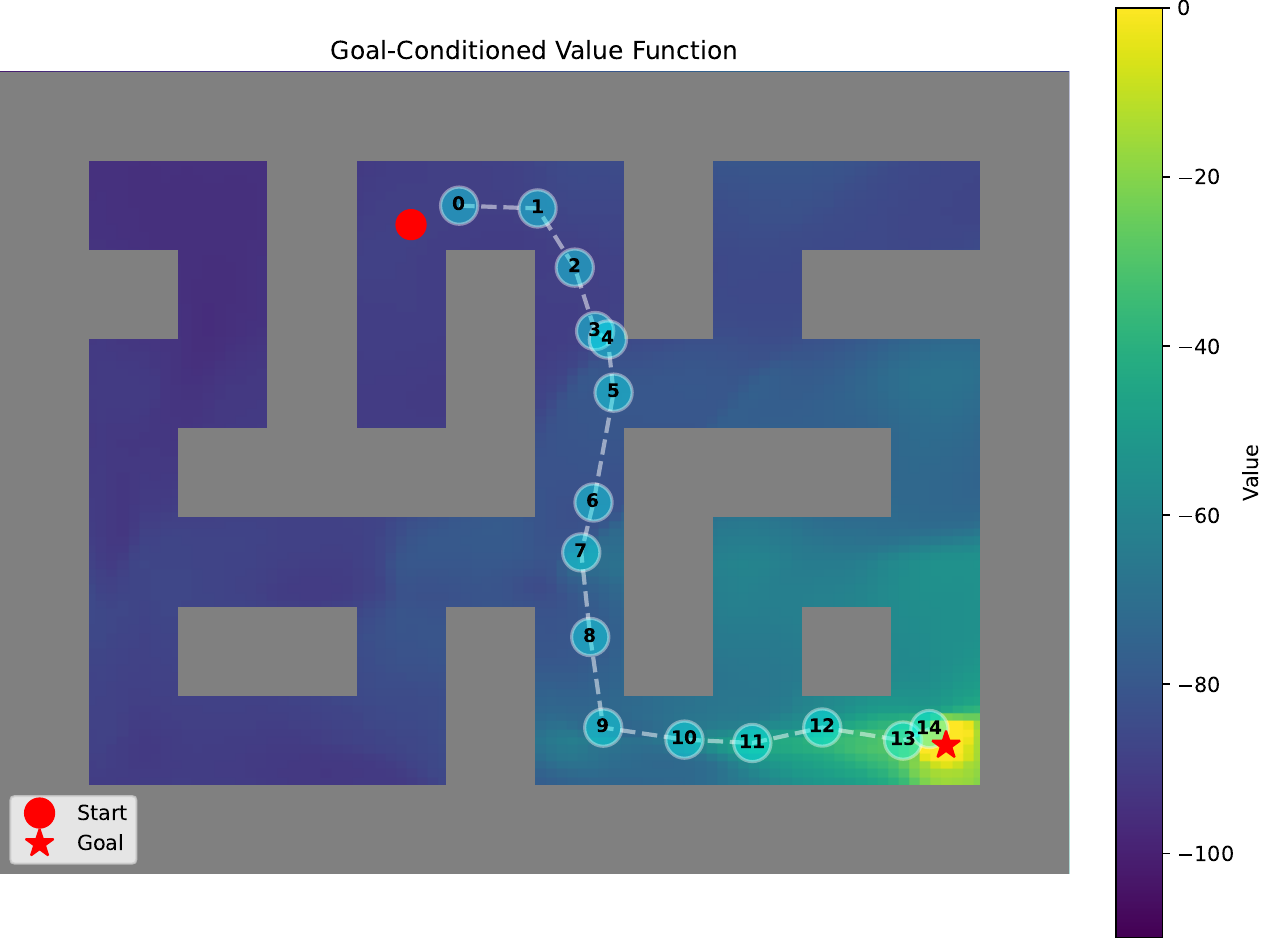} \\
            \includegraphics[width=\linewidth]{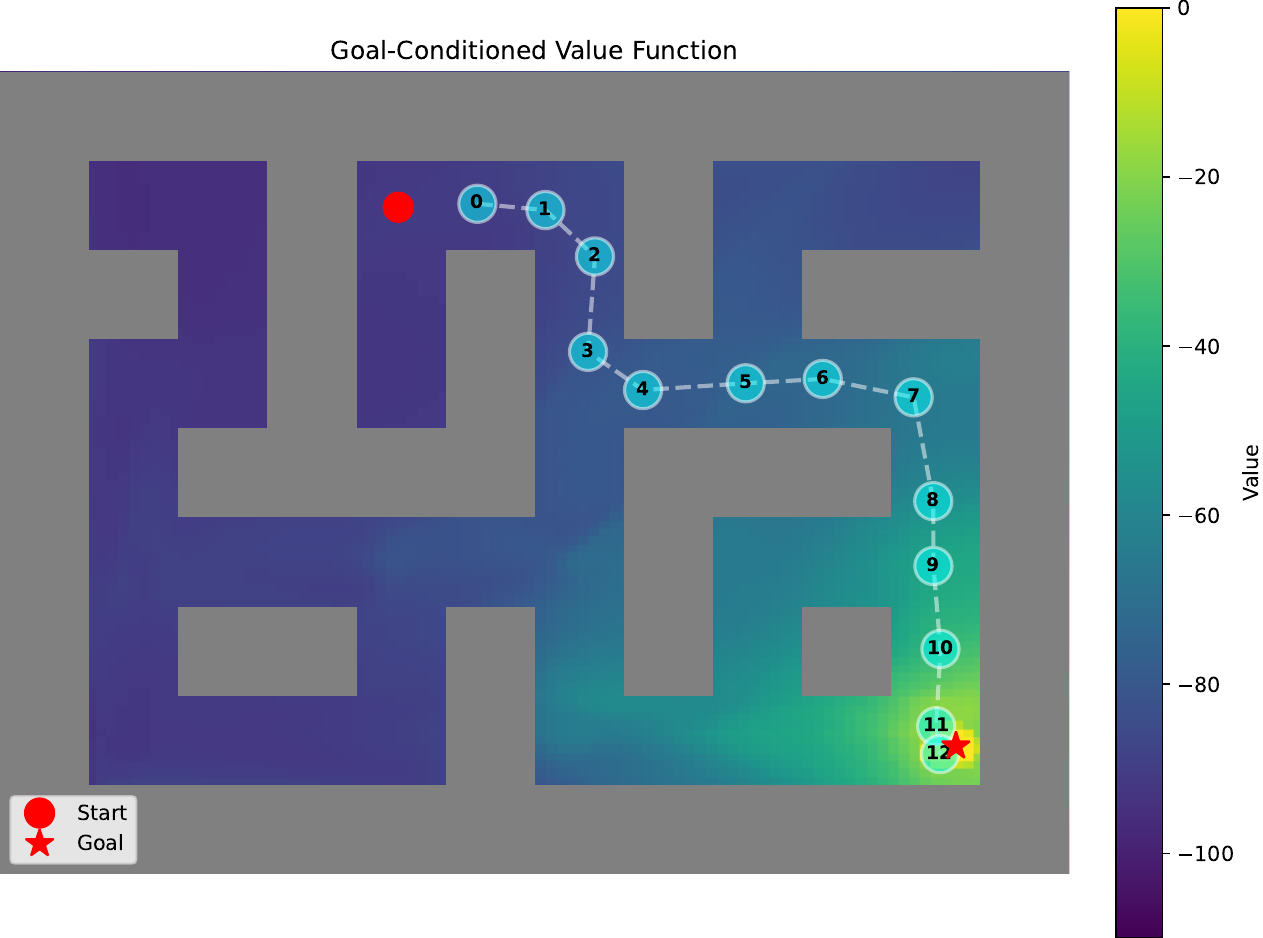}
        \end{minipage}
    }
    \hfill
    \subfigure[Task 4]{
        \begin{minipage}[b]{0.182\linewidth}
            \centering
            \includegraphics[width=\linewidth]{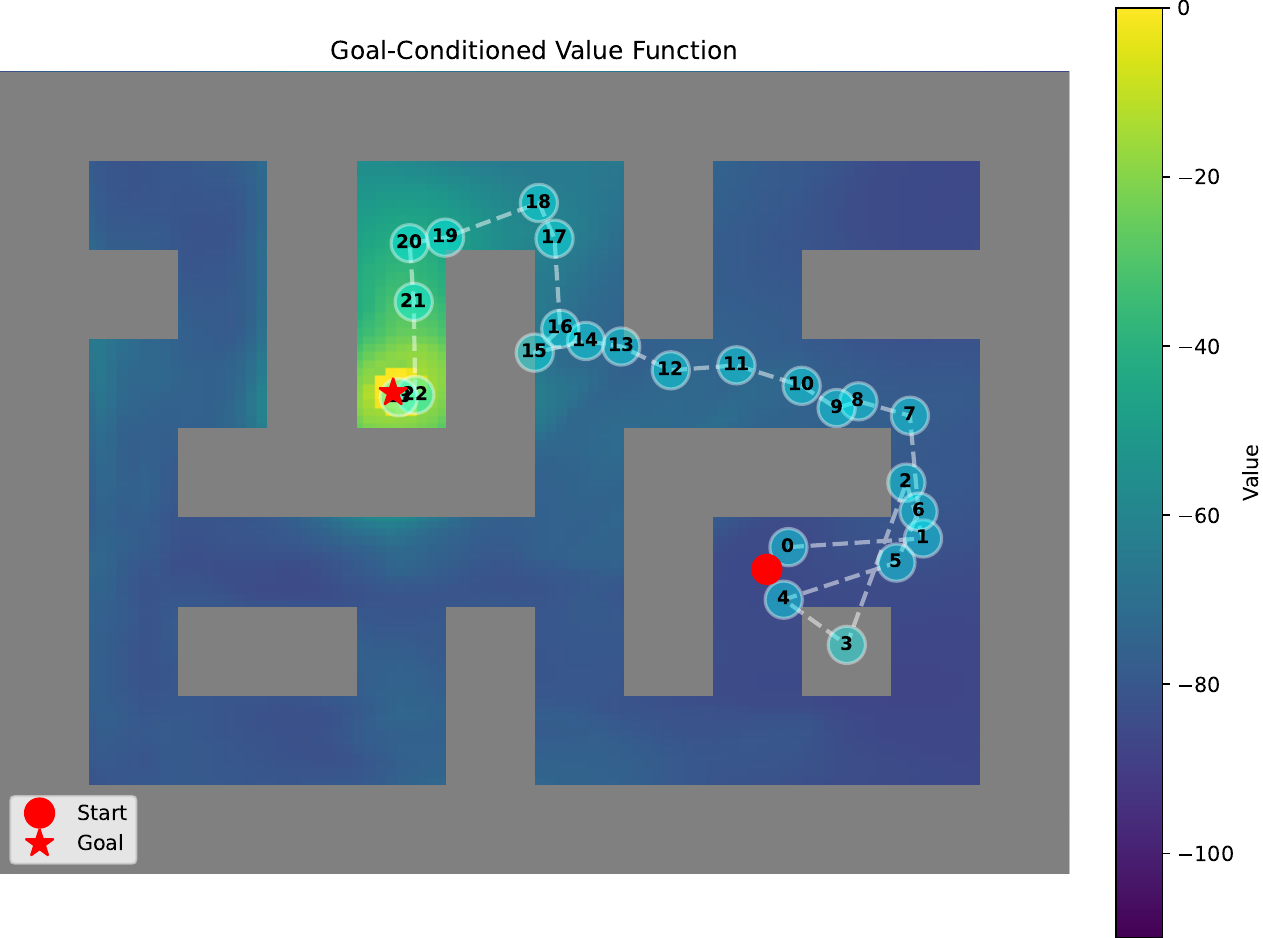} \\
            \includegraphics[width=\linewidth]{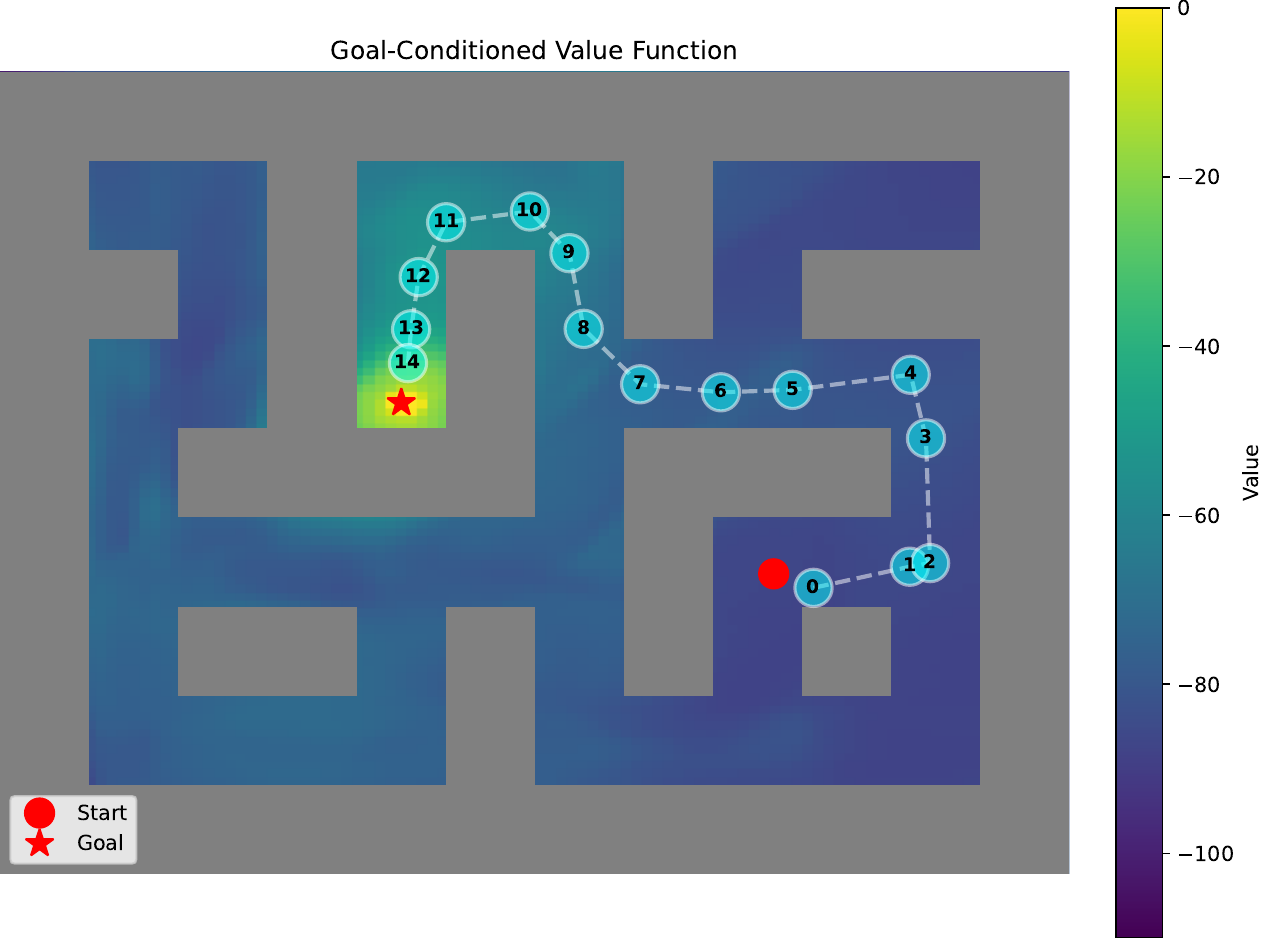}
        \end{minipage}
    }
    \hfill
    \subfigure[Task 5]{
        \begin{minipage}[b]{0.182\linewidth}
            \centering
            \includegraphics[width=\linewidth]{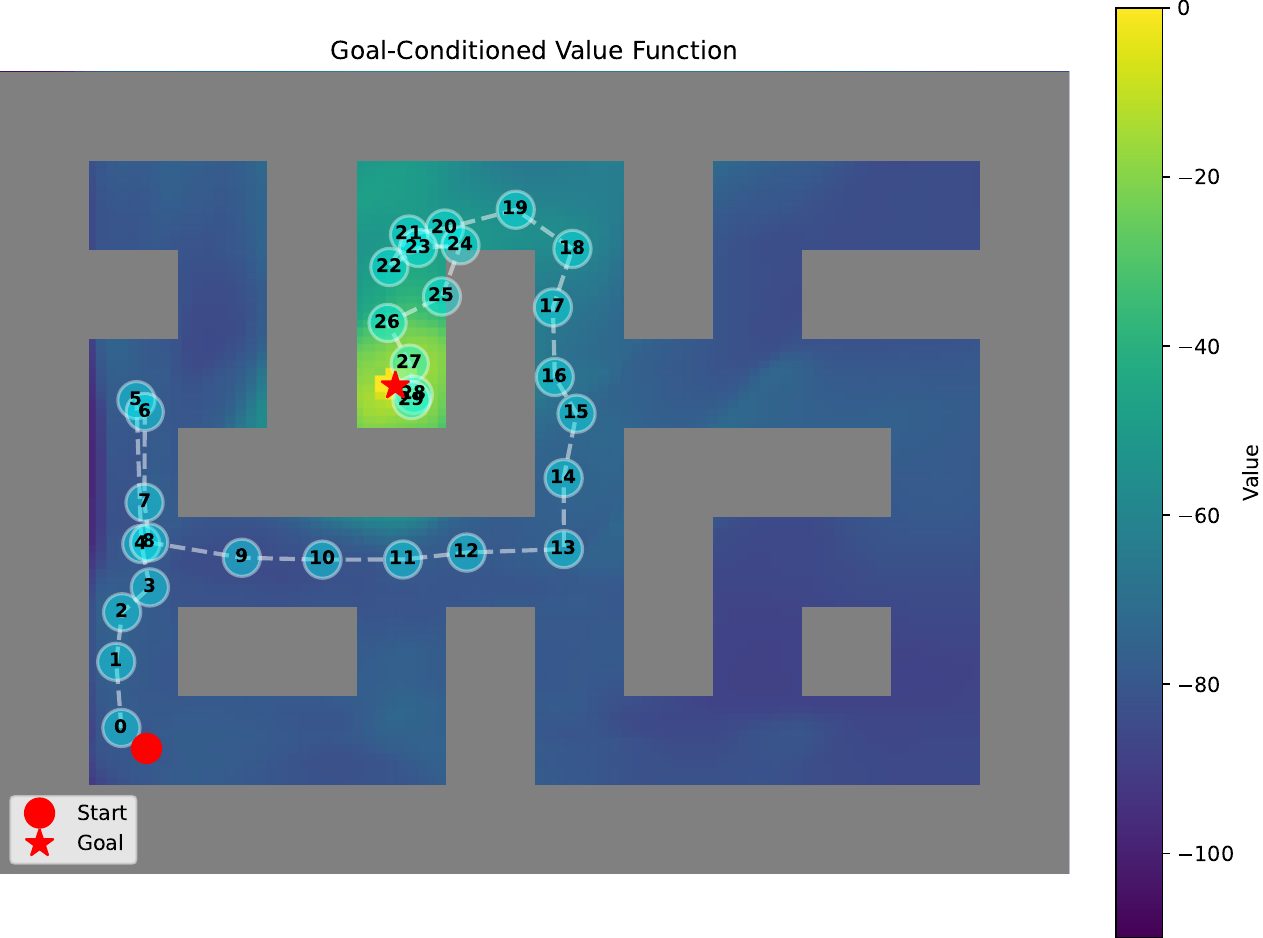} \\
            \includegraphics[width=\linewidth]{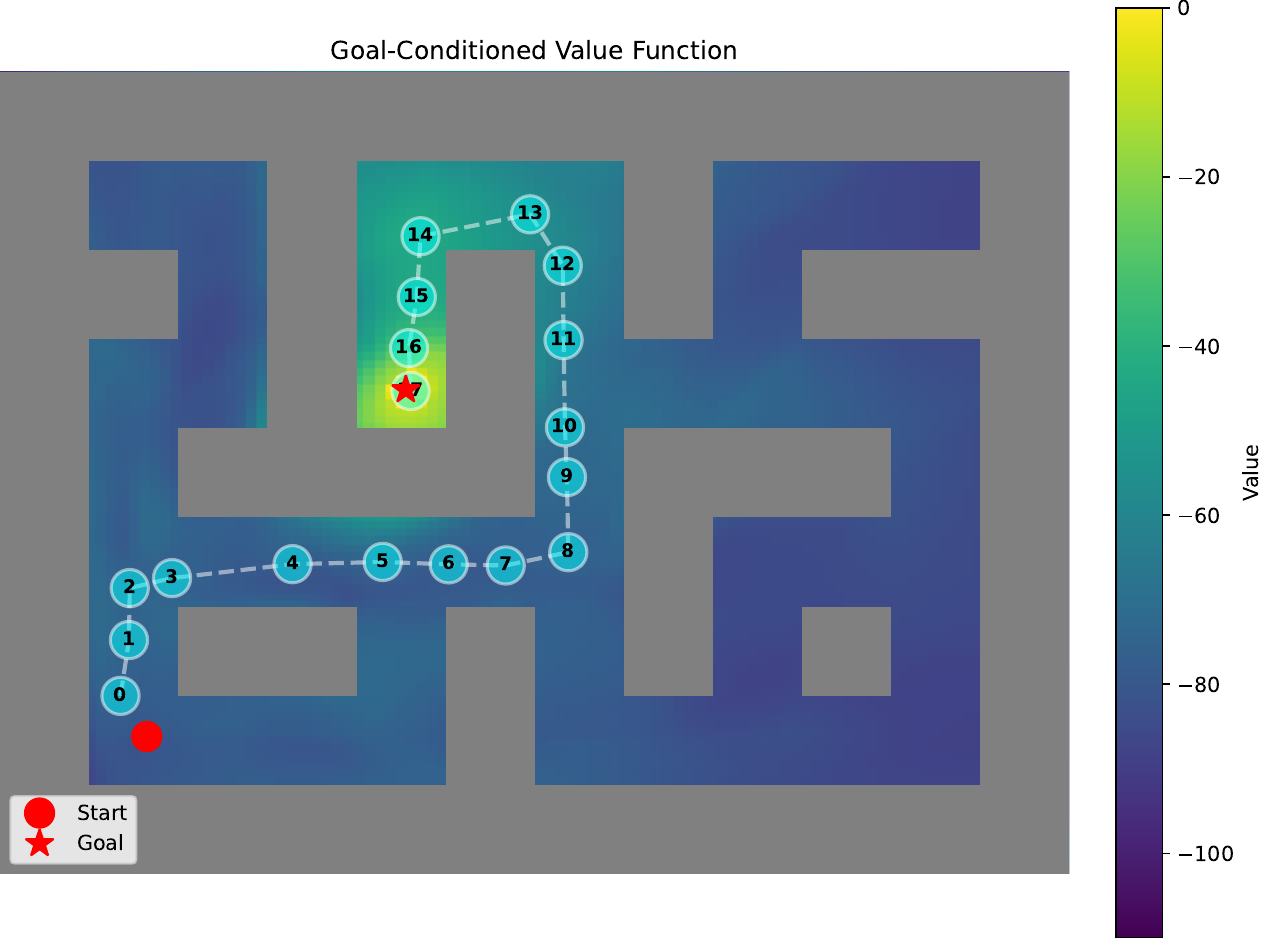}
        \end{minipage}
    }
    \caption{Visualization in the \texttt{large} maze. Top: HIQL$^{\mathrm{w/o}}$, Bottom: DSP. As the scale increases, HIQL$^{\mathrm{w/o}}$ begins to show significant detours, whereas DSP maintains efficient planning.}
    \label{fig:large_comparison}
\end{figure}

\subsection{Giant Maze}
\begin{figure}[h]
    \centering
    \subfigure[Task 1]{
        \begin{minipage}[b]{0.182\linewidth}
            \centering
            \includegraphics[width=\linewidth]{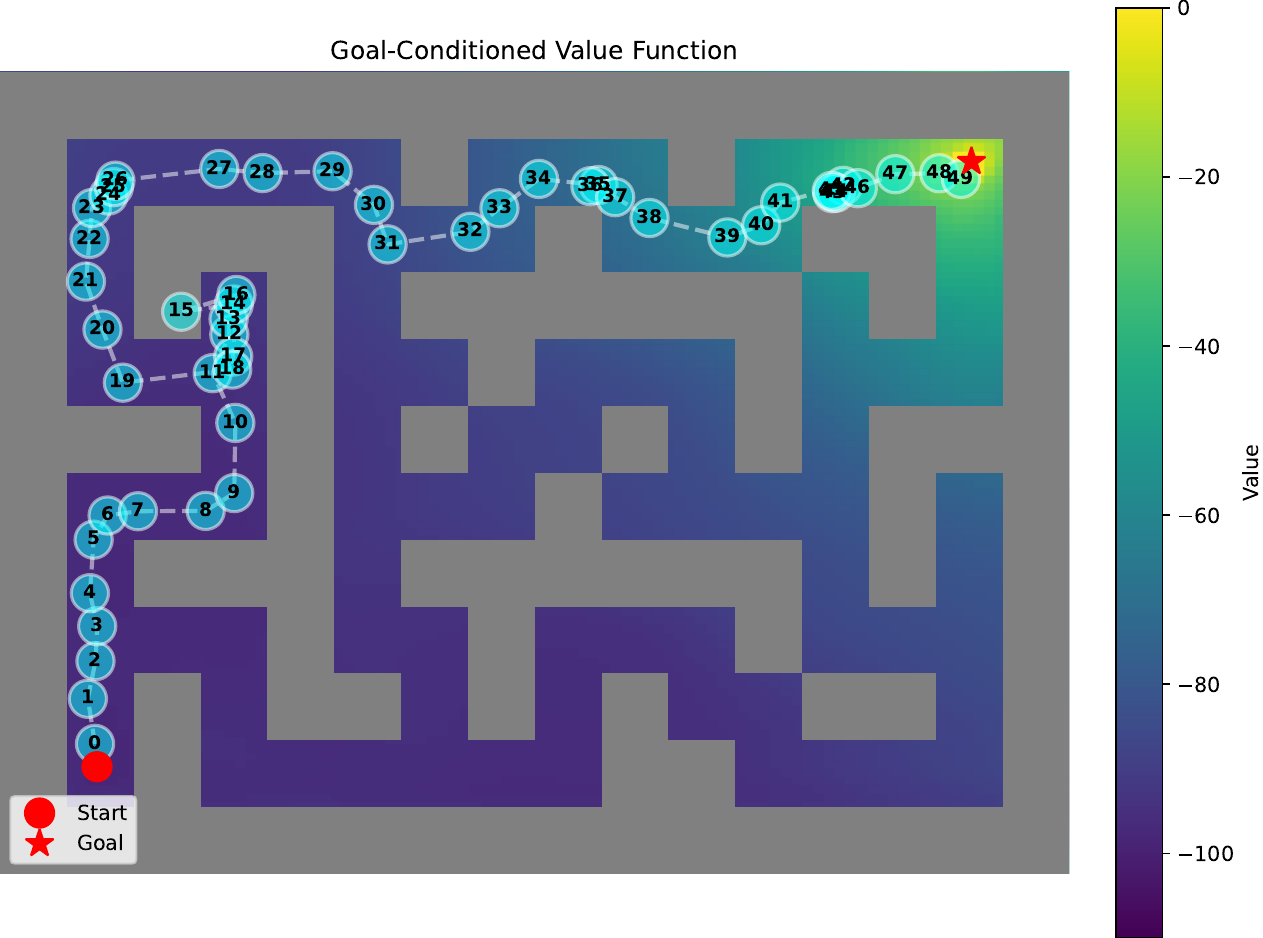} \\
            \includegraphics[width=\linewidth]{Figures/Visualization/giant/DSP_giant_value_function_task_1.pdf}
        \end{minipage}
    }
    \hfill
    \subfigure[Task 2]{
        \begin{minipage}[b]{0.182\linewidth}
            \centering
            \includegraphics[width=\linewidth]{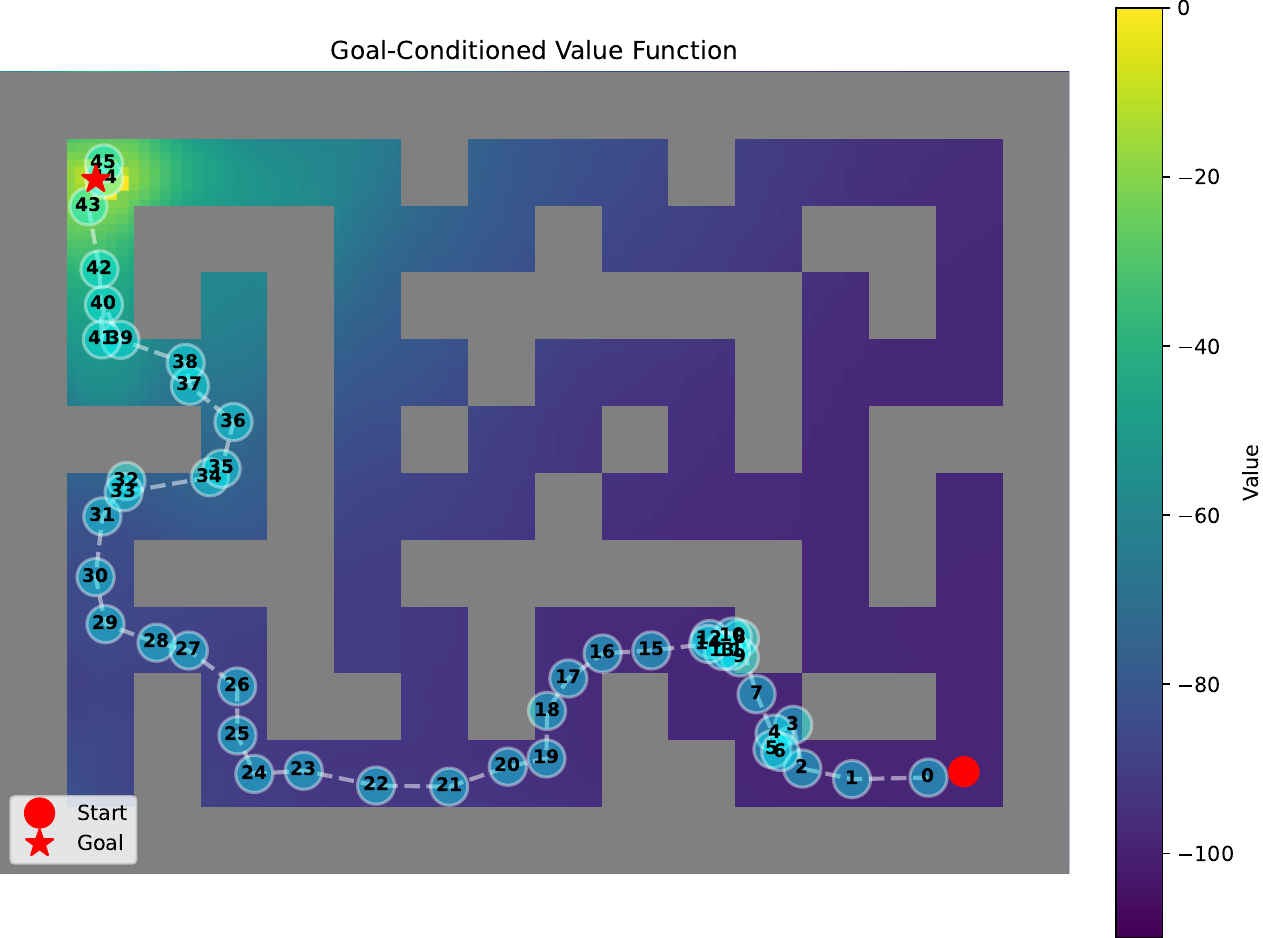} \\
            \includegraphics[width=\linewidth]{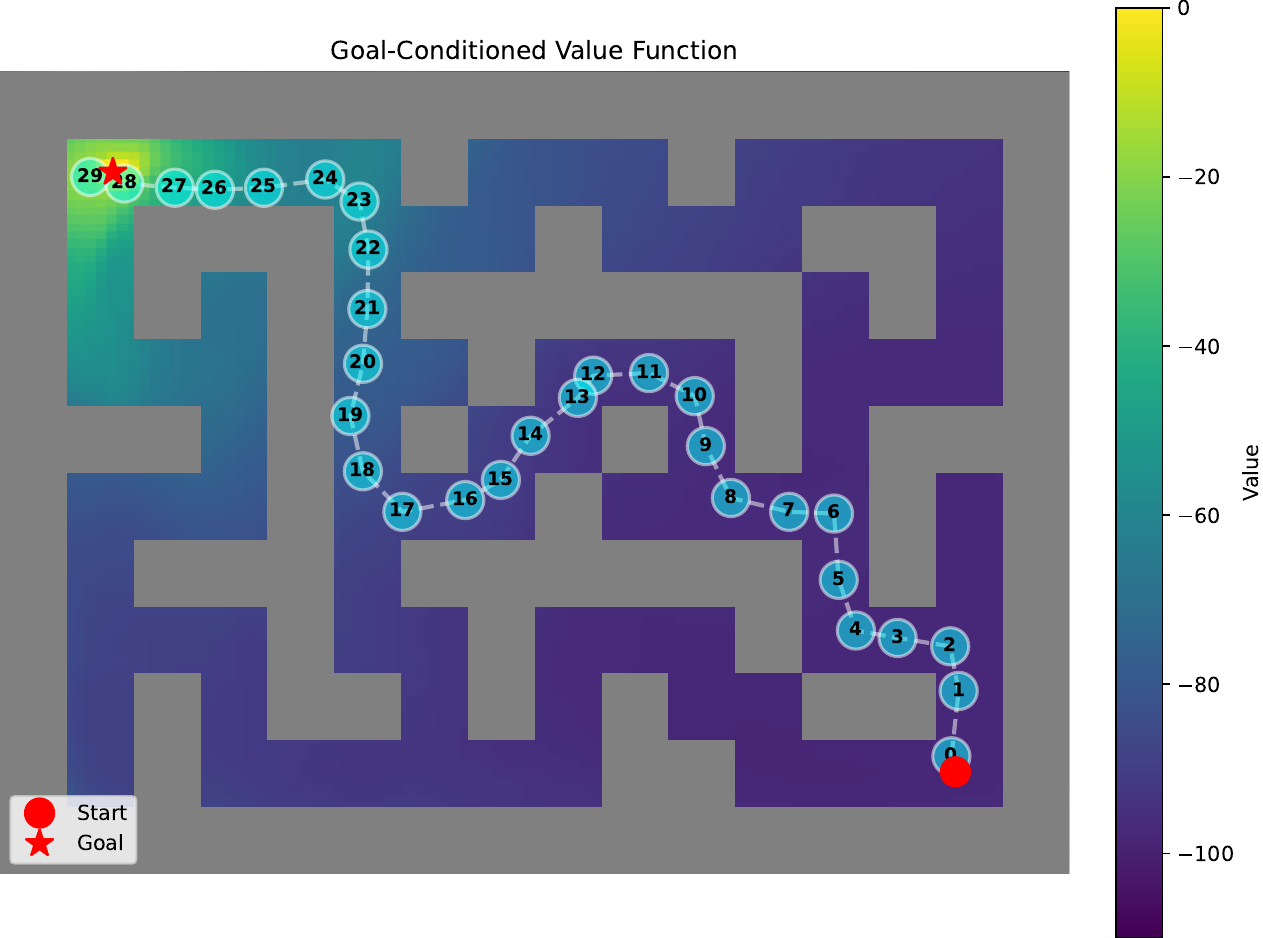}
        \end{minipage}
    }
    \hfill
    \subfigure[Task 3]{
        \begin{minipage}[b]{0.182\linewidth}
            \centering
            \includegraphics[width=\linewidth]{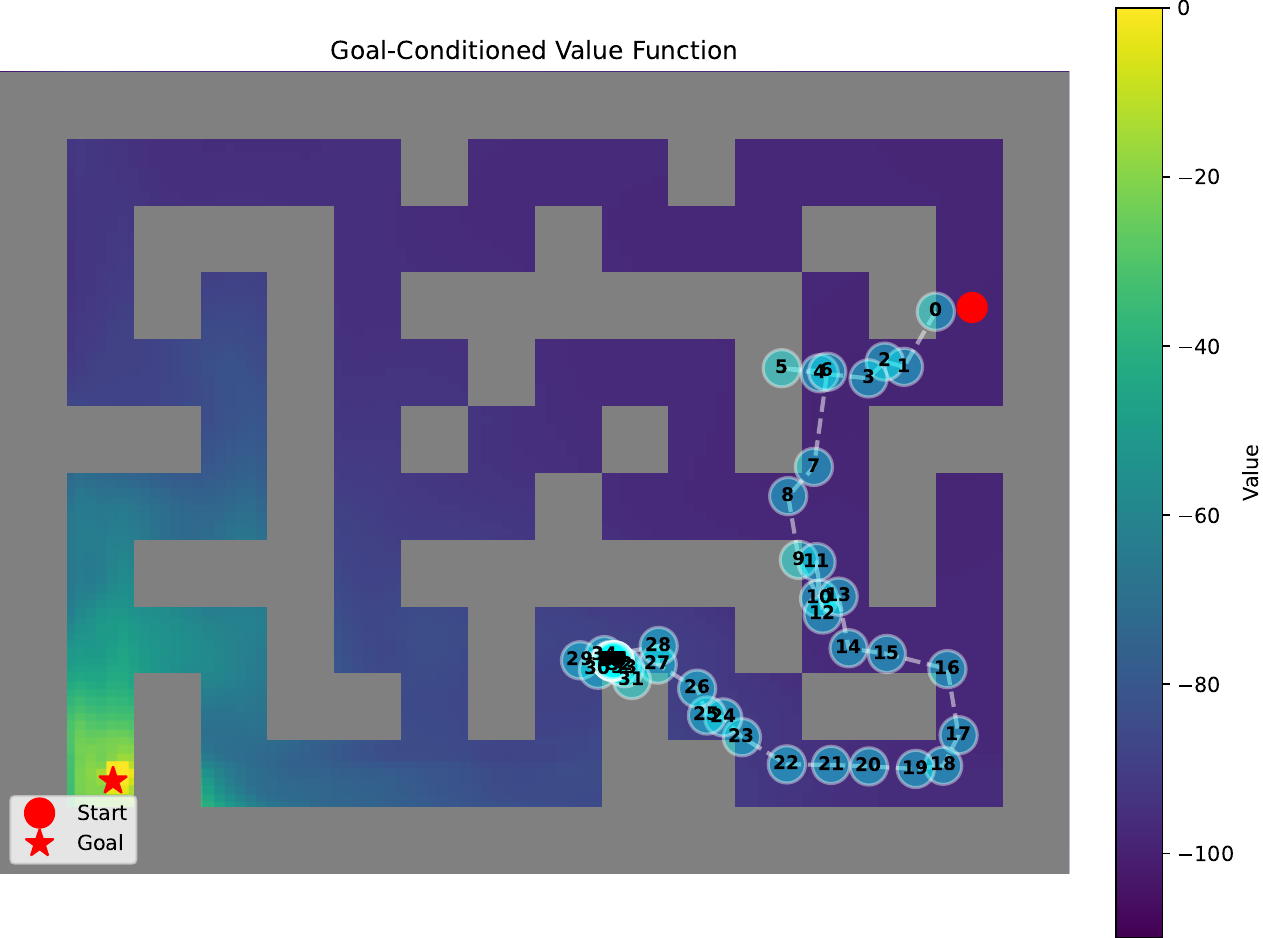} \\
            \includegraphics[width=\linewidth]{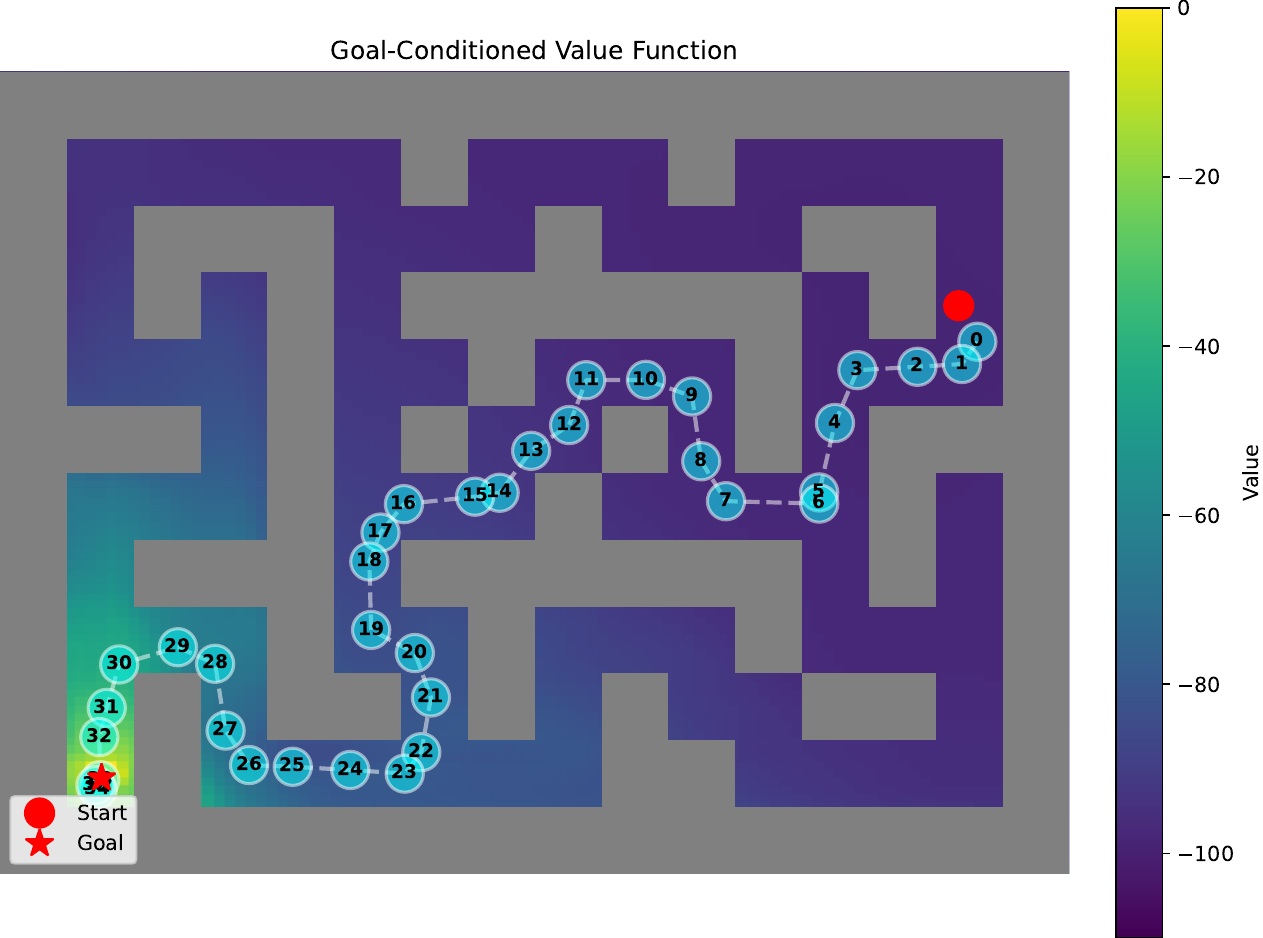}
        \end{minipage}
    }
    \hfill
    \subfigure[Task 4]{
        \begin{minipage}[b]{0.182\linewidth}
            \centering
            \includegraphics[width=\linewidth]{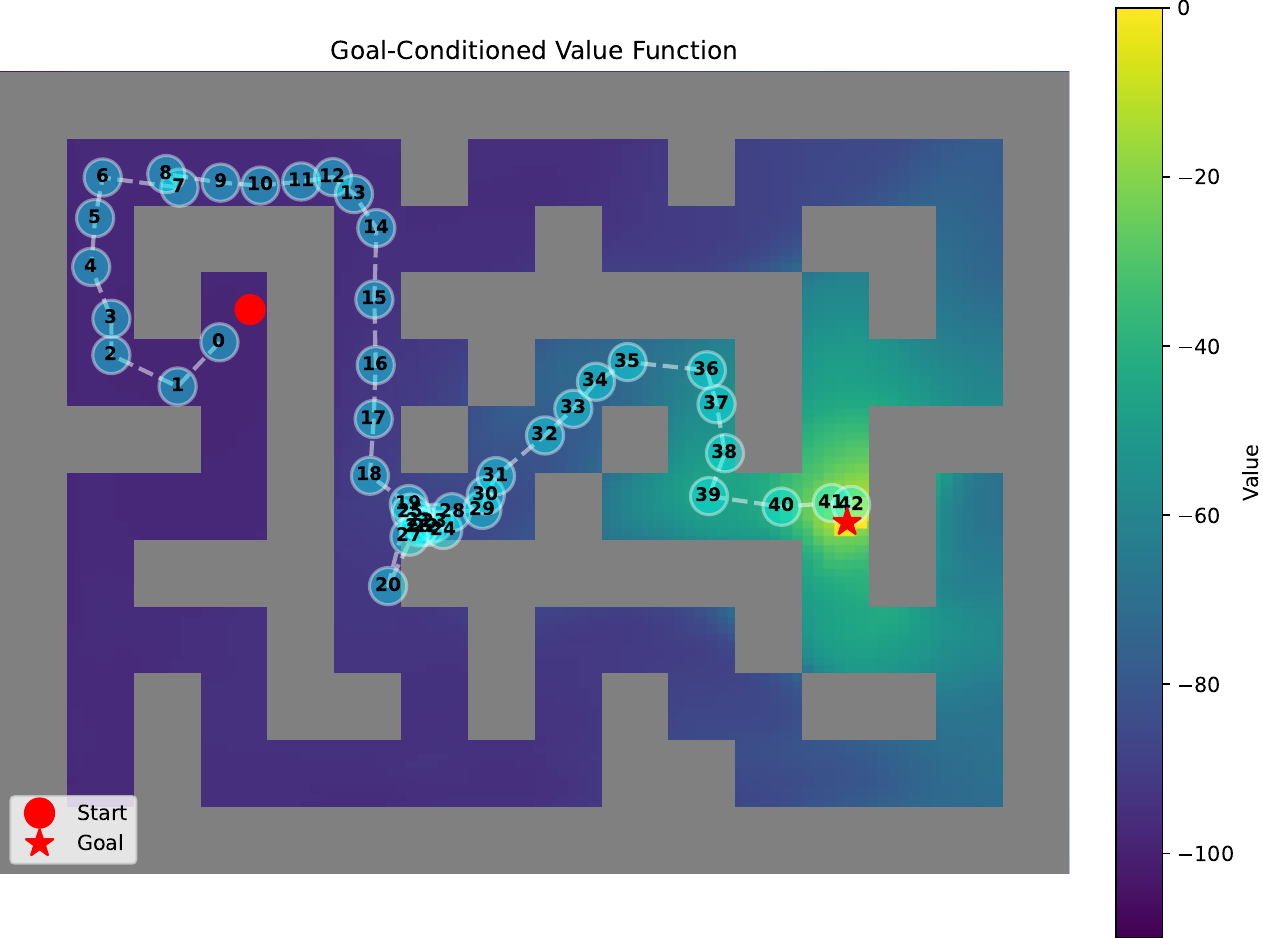} \\
            \includegraphics[width=\linewidth]{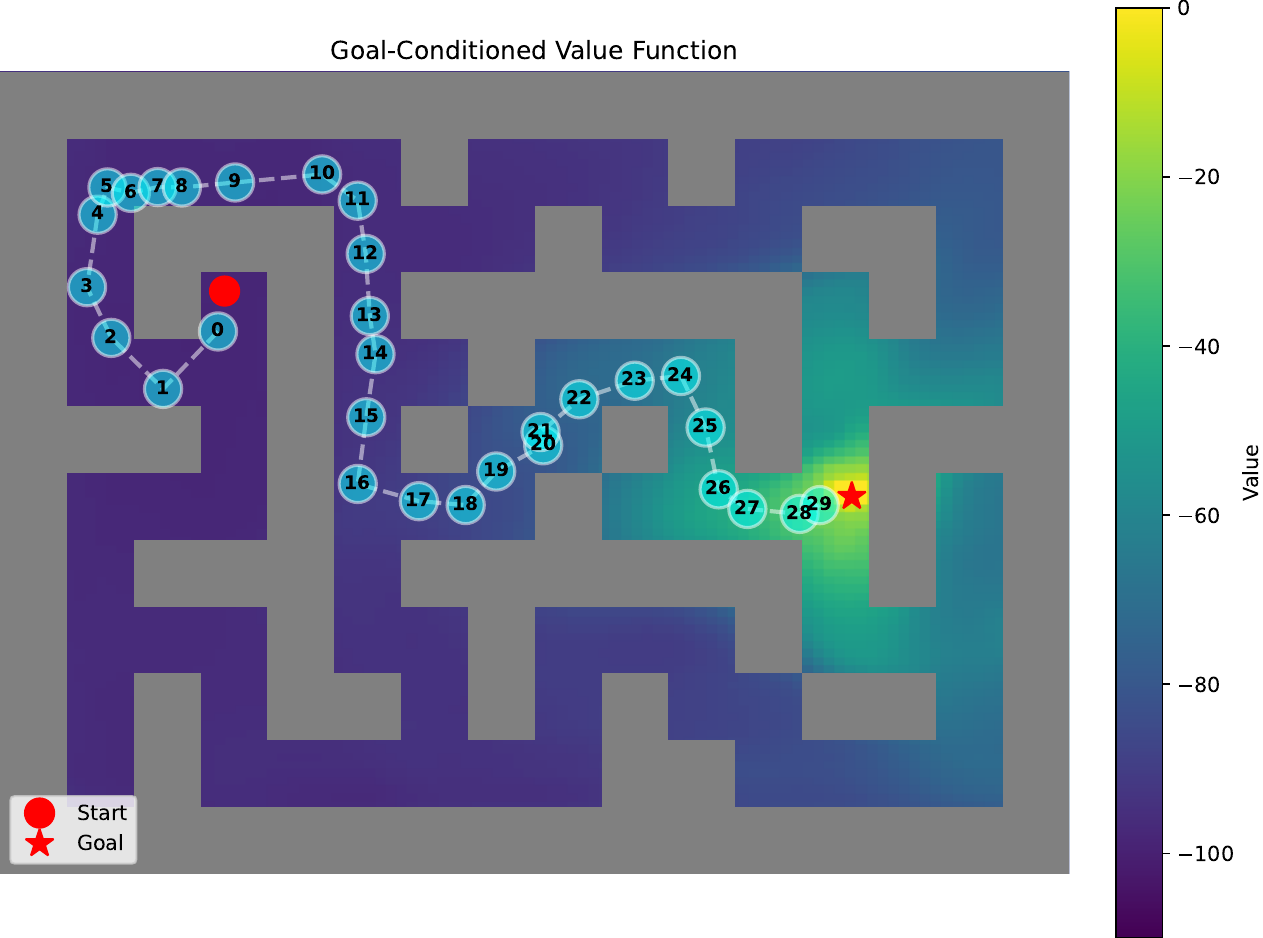}
        \end{minipage}
    }
    \hfill
    \subfigure[Task 5]{
        \begin{minipage}[b]{0.182\linewidth}
            \centering
            \includegraphics[width=\linewidth]{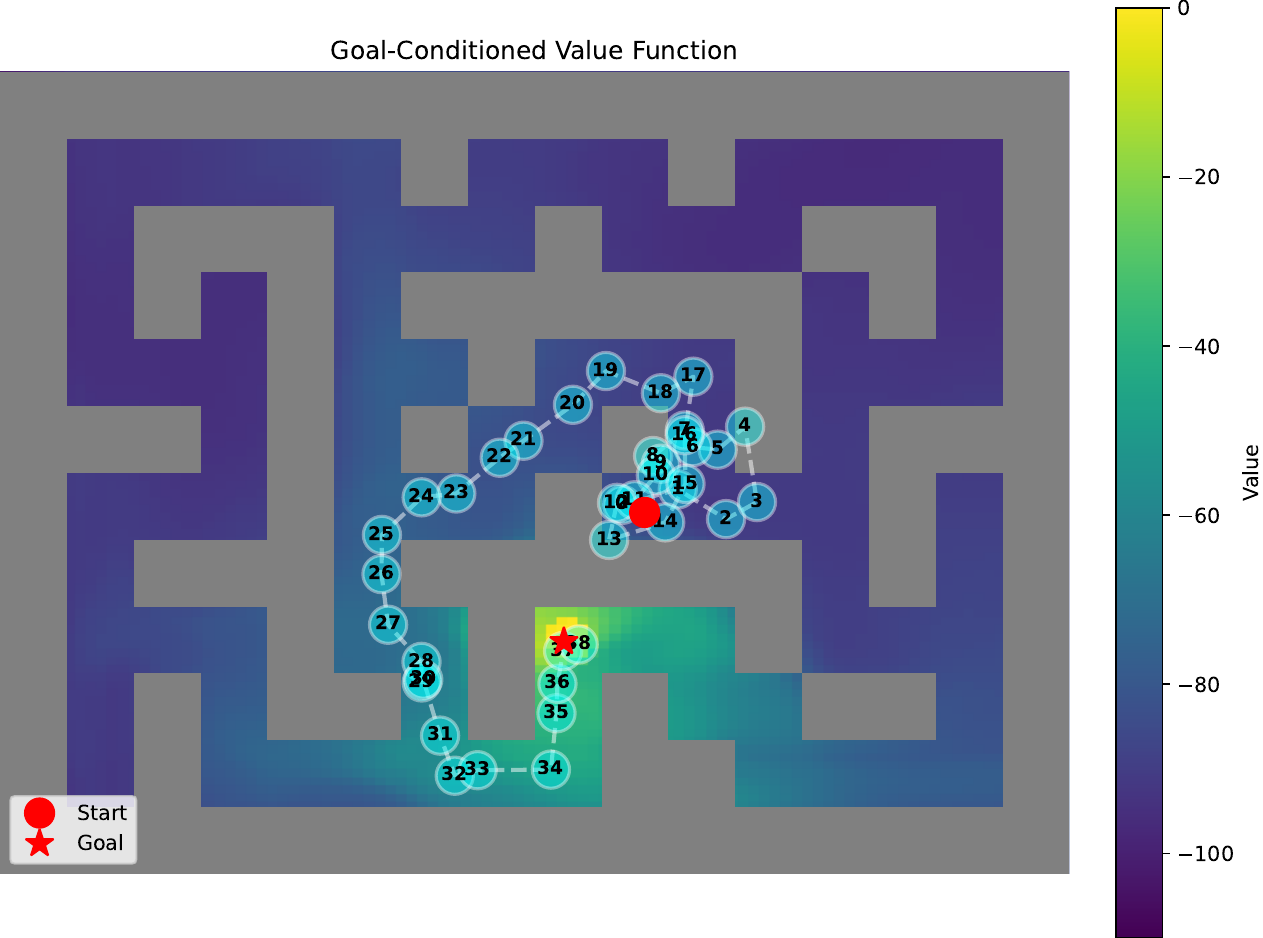} \\
            \includegraphics[width=\linewidth]{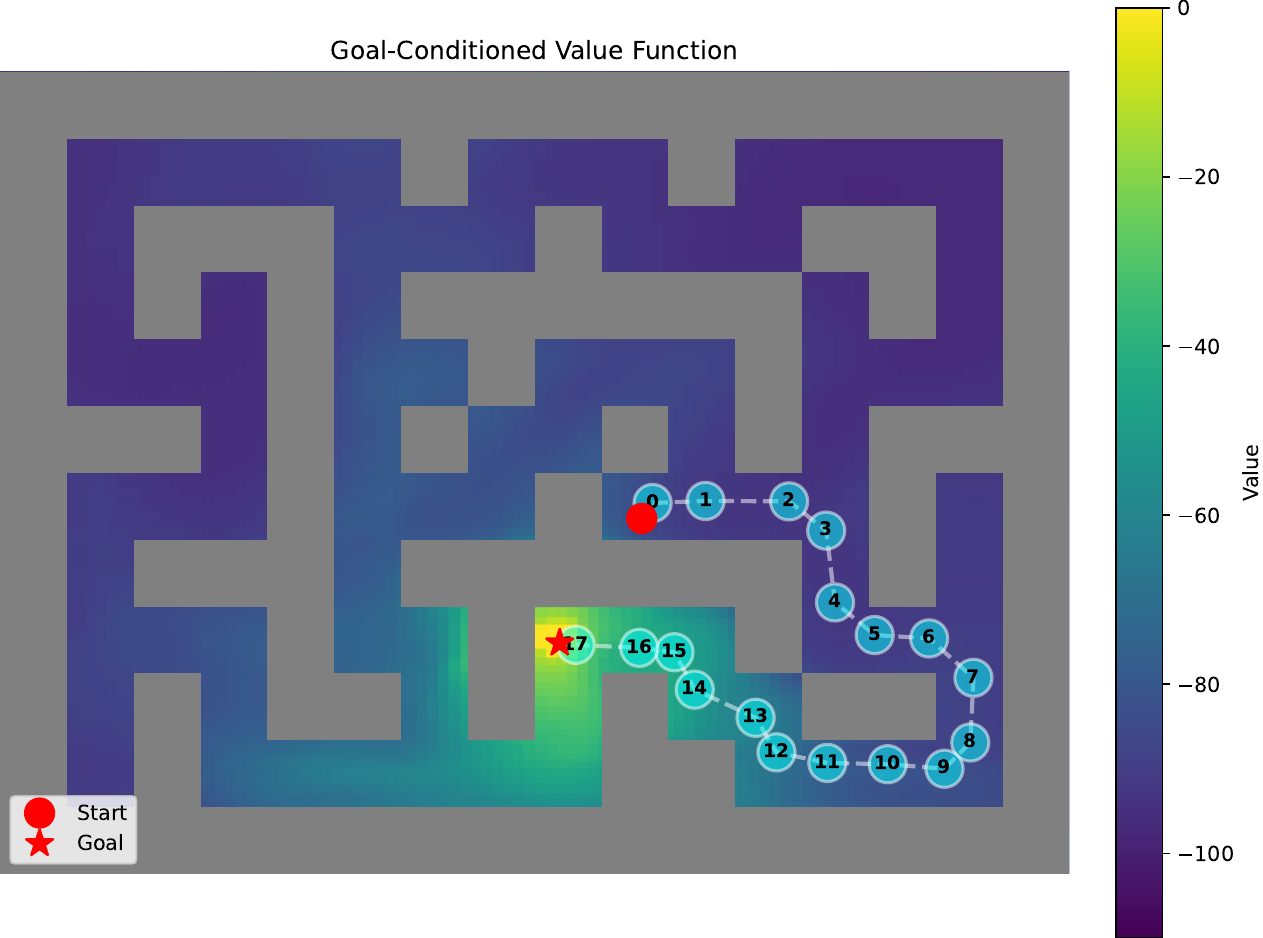}
        \end{minipage}
    }
    \caption{Visualization in the \texttt{giant} maze. Top: HIQL$^{\mathrm{w/o}}$, Bottom: DSP. In this extreme setting, HIQL$^{\mathrm{w/o}}$ often produces less reliable plans, while DSP more consistently reaches the target.}
    \label{fig:giant_comparison}
\end{figure}
Finally, we extend our visualization to the \texttt{giant}-sized maze, where the curse of horizon is most acute. Here, HIQL$^{\mathrm{w/o}}$ exhibits severe performance degradation. The high-level policy, struggling with the immense scale, often fails to generate valid subgoals, leading to task failure. Even in successful trials, the attenuation of the value signal results in subgoals that linger in local optima; although the agent may eventually escape, the traversal is far from time-optimal.

In sharp contrast, DSP shows more consistent navigation behavior. However, in this extreme setting, we observe a phenomenon of \textbf{subgoal clustering}, where consecutive high-level subgoals are repeatedly generated in a local region. Empirically, this behavior is not indicative of high-level planning failure, as the generated subgoals remain topologically valid and goal-directed. Instead, the clustering pattern demonstrates the planner's robustness to local execution stochasticity. Since the low-level policy is trained via AWR, it may exhibit minor deviations due to local value approximations. Crucially, unlike the larger global planning errors seen in baselines, these deviations are bounded within short horizons. As a result, the high-level planner successfully re-issues corrective subgoals in the same vicinity to compensate for these local execution errors, ensuring eventual task completion.

\section{Comparison Experiments on D4RL Antmaze}
\label{app:d4rl_antmaze}
To further position DSP with respect to conceptually related diffusion-based hierarchical planning methods, we additionally evaluate DSP on the D4RL AntMaze benchmark. Our main experiments are conducted on OGBench, which is specifically designed for offline goal-conditioned RL and provides broader coverage of long-horizon reasoning, stitching, and stochasticity. By contrast, D4RL AntMaze is more limited for evaluating offline GCRL since it mainly uses a single fixed goal. Nevertheless, D4RL AntMaze is the benchmark on which several prior diffusion-based hierarchical planning methods were originally reported, and therefore provides the most direct setting for comparison.

\begin{table}[h]
    \centering
    \caption{Comparison with diffusion-based hierarchical planning methods on D4RL AntMaze. We report mean $\pm$ standard deviation over 5 random seeds, with each seed evaluated on 100 episodes. Baseline numbers are taken from the corresponding original papers when reported under the same D4RL AntMaze setting. The best-performing entry in each row is highlighted in \textbf{bold}.}
    \label{tab:d4rl_antmaze}
    \resizebox{0.85\linewidth}{!}{
    \begin{tabular}{lccccc}
    \toprule
        \textbf{Datasets} & \textbf{HDMI} & \textbf{DTAMP} & \textbf{HD} & \textbf{SIHD} & \textbf{DSP} \\
        \midrule
        \texttt{antmaze-umaze-v2} & $86.1 \pm 2.4$ & $-$ & $-$ & $-$ & $\mathbf{98.6 \pm 2.1}$ \\
        \texttt{antmaze-umaze-diverse-v2} & $73.7 \pm 1.1$ & $-$ & $94.0 \pm 4.9$ & $96.5 \pm 2.8$ & $\mathbf{96.8 \pm 2.6}$ \\
        \texttt{antmaze-medium-play-v2} & $-$ & $\mathbf{89.3 \pm 3.9}$ & $-$ & $-$ & $87.0 \pm 3.3$ \\
        \texttt{antmaze-medium-diverse-v2} & $-$ & $76.7 \pm 4.5$ & $88.7 \pm 8.1$ & $92.2 \pm 5.0$ & $\mathbf{92.4 \pm 4.4}$ \\
        \texttt{antmaze-large-play-v2} & $-$ & $62.0 \pm 2.9$ & $-$ & $-$ & $\mathbf{88.8 \pm 4.1}$ \\
        \texttt{antmaze-large-diverse-v2} & $71.5 \pm 3.5$ & $53.3 \pm 9.7$ & $83.6 \pm 5.8$ & $89.4 \pm 4.2$ & $\mathbf{95.4 \pm 2.9}$ \\
        \bottomrule
    \end{tabular}}
\end{table}

Table~\ref{tab:d4rl_antmaze} compares DSP with representative diffusion-based hierarchical planning baselines, including HDMI~\citep{HDMI}, DTAMP~\citep{DTAMP}, HD~\citep{HD}, and SIHD~\citep{SIHD}. We report mean and standard deviation over 5 random seeds, with each seed evaluated on 100 episodes. Baseline numbers are taken from the corresponding original papers when reported on the same D4RL AntMaze setting. The best-performing entry in each row is highlighted in bold.

Overall, DSP is highly competitive and outperforms prior diffusion-based hierarchical planning methods on most reported tasks. The gains are particularly pronounced on larger and more challenging tasks such as \texttt{antmaze-large-play-v2} and \texttt{antmaze-large-diverse-v2}, where long-horizon decision errors are more likely to accumulate. On averages over overlapping tasks, DSP substantially outperforms DTAMP (90.9 vs.\ 70.3) and HDMI (96.9 vs.\ 77.1), and also improves over HD (94.9 vs.\ 88.8) and SIHD (94.9 vs.\ 92.7).

These results are consistent with the key design of DSP. Unlike prior diffusion-based hierarchical planning methods, DSP does not use diffusion as a trajectory- or sequence-level planner. Instead, it models high-level decision making as a goal-conditioned generative subgoal policy. This is consistent with the reduced high-level value-noise path analyzed in Proposition~\ref{prop:4.2}. At the same time, classifier-free guidance enables direct and controllable goal-directed subgoal selection at inference time, while Proposition~\ref{prop:5.1} explains the implicit advantage-weighted bias induced by this guidance.

\section{Dexterous Manipulation Evaluation}
\label{app:dexterous_manipulation}

We evaluate DSP and HIQL on two bimanual dexterous-manipulation tasks from Bi-DexHands~\citep{Bi-DexHands}: \texttt{ShadowHandOver} and \texttt{ShadowHandCatchUnderarm}. Using the official collection interface, we construct a fixed offline dataset of $10^6$ transitions per task with PPO. Both methods are trained on identical datasets under the same OGBench-style GCRL protocol. HIQL and DSP results are success rates (\%) reported as means $\pm$ standard deviations over five seeds. PPO is shown only as a data-collection reference, not as an offline or compute-matched baseline.

The tuples in \autoref{tab:bidexhands_results} report $(d_s,d_a,d_w,d_g)$. The 7D goal $g$ is derived directly from the task-specified target-object pose, while the original 7D waypoint $w$ is a future achieved object pose. To isolate the effect of waypoint dimensionality without changing the final task, we keep the 7D goal fixed and replace the task-space waypoint with a 55/67D task-relevant configuration-space waypoint containing bimanual configuration and object-pose information. Here, $d_w$ denotes the waypoint dimension before method-specific processing: HIQL retains its learned subgoal representation, whereas DSP directly generates in the selected waypoint space.

\begin{table}[ht]
    \centering
    \caption{
        Success rates (\%) on Bi-DexHands.
        PPO denotes the data collector.
        Dimensions are reported as $(d_s,d_a,d_w,d_g)$.
    }
    \label{tab:bidexhands_results}
    \small
    \setlength{\tabcolsep}{4.5pt}
    \begin{tabular}{@{}lcccc@{}}
        \toprule
        Datasets & $(d_s,d_a,d_w,d_g)$
        & PPO & HIQL & DSP \\
        \midrule
        \texttt{ShadowHandOver}
        & $(387,40,7,7)$
        & 46.5 & $56.6\pm4.2$ & $\mathbf{74.2}\pm2.9$ \\

        \texttt{ShadowHandCatchUnderarm}
        & $(411,52,7,7)$
        & 37.2 & $46.0\pm4.1$ & $\mathbf{54.6}\pm3.3$ \\

        \texttt{ShadowHandOver}
        & $(387,40,55,7)$
        & 46.5 & $54.2\pm5.9$ & $\mathbf{67.6}\pm2.4$ \\

        \texttt{ShadowHandCatchUnderarm}
        & $(411,52,67,7)$
        & 37.2 & $41.2\pm3.4$ & $\mathbf{46.4}\pm3.9$ \\
        \bottomrule
    \end{tabular}
\end{table}

With 7D task-space waypoints, DSP achieves higher mean success than HIQL by 17.6 and 8.6 percentage points on \texttt{ShadowHandOver} and \texttt{ShadowHandCatchUnderarm}, respectively. Moving to 55/67D configuration-space waypoints reduces DSP's mean success by 6.6/8.2 points, compared with 2.4/4.8 points for HIQL. Nevertheless, DSP retains higher mean success by 13.4/5.2 points in the two higher-dimensional waypoint settings. These results show that waypoint dimensionality is a genuine challenge for direct DSP generation, while providing initial task-specific evidence that its advantage is not solely attributable to the 7D waypoint representation.

The collected datasets contain executable local behaviors. Accordingly, these experiments compare how DSP and HIQL plan and execute bimanual behavior under identical offline data support, rather than testing whether either method can acquire skills absent from the dataset. The results provide initial evidence beyond OGBench's parallel-jaw manipulation setting, without constituting a comprehensive evaluation of dexterous manipulation.


\end{document}